\documentclass[final]{siamart171218}

\usepackage{braket,amsfonts, bbm}
\usepackage{dutchcal}
\usepackage{array}
\usepackage{pgfplots}
\pgfplotsset{compat=newest}
\usepackage{subcaption}
\usepackage{mathtools}
\usepackage{eqparbox}
\usepackage{enumerate}
\usepackage{etoolbox}
\usepackage{xcolor}
\usepackage{showlabels}
\usepackage{bbm}

\newsiamthm{claim}{Claim}
\newsiamremark{conjecture}{Conjecture}
\newsiamremark{remark}{Remark}
\newsiamremark{example}{Example}
\newsiamremark{hypothesis}{Hypothesis}
\newsiamremark{problem}{Problem}
\crefname{hypothesis}{Hypothesis}{Hypotheses}
\AtEndEnvironment{problem}{\null\hfill$\Diamond$}
\AtEndEnvironment{example}{\null\hfill$\Diamond$}
\AtEndEnvironment{remark}{\null\hfill$\Diamond$}
\newtheorem{assumption}{Assumption}

\usepackage{algorithmic}

\usepackage{graphicx,epstopdf}
\usepackage{float}

\usepackage{amsopn}

\usepackage{verbatim}

\usepackage{MnSymbol}

\usepackage{dsfont}

\newcommand{\mcl}{\mathcal}
\newcommand{\mbf}{\mathbf}
\newcommand{\mbb}{\mathbb}

\newcommand{\eps}{\epsilon}

\newcommand{\tH}{H'}
\newcommand{\tC}{C'}
\newcommand{\tPhi}{\Phi'}
\newcommand{\up}{\mbf u^+}
\newcommand{\um}{\mbf u^-}

\newcommand{\Ppp}{P^{\tiny{++}}}
\newcommand{\Pmm}{P^{\tiny{--}}}
\newcommand{\Ppm}{P^{\tiny{+-}}}
\newcommand{\Pmp}{P^{\tiny{-+}}}

\newcommand{\tchi}{\bar{\pmb \chi}}

\newcommand{\sgn}{\text{\rm{sgn}}}
\newcommand{\iidsim}{\stackrel{iid}{\sim}}

\newcommand{\msf}[1]{{\mathsf #1}}

\newcommand{\bu}{\mbf u}

\newcommand{\bv}{\mbf v}

\newcommand{\be}{\mbf e}

\newcommand{\ba}{\mbf a}
\newcommand{\bc}{\mbf c}

\newcommand{\tG}{\tilde{G}}
\newcommand{\tX}{\tilde{X}}
\newcommand{\tW}{\tilde{W}}
\newcommand{\tL}{\tilde{L}}
\newcommand{\tZ}{\tilde{Z}}
\newcommand{\ta}{\breve{a}}
\newcommand{\tba}{\breve{\ba}}
\newcommand{\tbeta}{\breve{\beta}}
\newcommand{\teta}{\breve{\eta}}
\newcommand{\tPsi}{\breve{\Psi}}
\newcommand{\tbb}{\breve{\mbf{b}}}
\newcommand{\tD}{\tilde{D}}

\newcommand{\asto}{\xrightarrow{\rm{a.s.}}}

\definecolor{darkred}{rgb}{.7,0,0}

\definecolor{darkgreen}{rgb}{.15,.55,0}

\definecolor{darkblue}{rgb}{0,0,0.7}

\DeclareMathOperator*{\argmin}{arg\,min}
\DeclareMathOperator*{\argmax}{arg\,max}

\newcommand{\as}[1]{{#1}}

\reversemarginpar
\usepackage[colorinlistoftodos,prependcaption,textsize=tiny]{todonotes}

\title{Consistency of  Semi-Supervised Learning\\ Algorithms on Graphs:
  Probit and One-Hot Methods
}

\author{Franca Hoffmann\and
Bamdad Hosseini\and Zhi Ren \and
Andrew M. Stuart 
\thanks{Computing and Mathematical Sciences, Caltech, Pasadena, CA
  (\email{fkoh@caltech.edu},
\email{bamdadh@caltech.edu}, \email{zren@caltech.edu}, \email{astuart@caltech.edu}).}} 

\begin{document}
\date{today}
\maketitle

\begin{abstract}
Graph-based semi-supervised learning is the problem of propagating 
labels from a small number of labelled data points to a larger set of
unlabelled data. This paper is concerned with the consistency
of optimization-based techniques for such problems,
in the limit where the labels have small noise and the
underlying unlabelled data is well clustered.
We study graph-based
probit for binary classification, and a natural generalization of
this method to multi-class classification using one-hot encoding.
The resulting objective function to be optimized comprises
the sum of a quadratic form defined through a rational
function of the graph Laplacian, involving only the unlabelled
data, and a fidelity term involving only the labelled data. 
The consistency analysis sheds light on the choice of the rational function
defining the optimization.
\end{abstract}




\section{Introduction}
Semi-supervised learning (SSL) is the problem of labelling all the points in a dataset,
by leveraging correlations and geometric information in the data points, together with
explicit knowledge of a subset of noisily observed labels. The primary goal of this article is to 
analyze  the probit and one-hot methods for transductive SSL. 
\as{We elaborate conditions under which 
these methods consistently recover the correct labels
attached to the unlabelled dataset. We do this in an idealized setting in which the unlabelled data is 
approximately clustered, and there is an unobserved latent variable
which determines labels and which is observed in a small noise regime. We prove consistency in a limit in which the data becomes
more clustered and the label noise goes to zero.} 
The formulation and analysis demonstrates how ideas from unsupervised learning
and, in particular spectral clustering, can be used as prior information;
this prior information is enhanced, or sets up a competition with, labelled 
data. In so doing, our analysis also elucidates the role of parameter
choices made when setting up the balance between labelled and unlabelled data. 
Furthermore, we exhibit useful properties of the probit and one-hot methods, 
including a representer theorem for the classifier, and a natural dimension
reduction which follows from this theorem and is impactful in practice.

\subsection{Background And Literature Review}

We start by giving informal statements of the problem to be solved,
and  a brief literature review. 
Consider a set of nodes $Z = \{ 1, \cdots, N \}$
and an associated
set of {\it feature vectors}
$X = \{ x_1, x_2, \cdots, x_N\}$. Each feature
vector $x_j$ is assumed to be a point in $\mbb R^d$. 
$X$ may thus be viewed as a function $X: Z \mapsto \mbb R^d$ or 
as an element of $\mbb R^{d \times N}.$ We refer to $X$ as {\it unlabelled data}.
Suppose there exists 
a function $l: Z \mapsto \{1, 2, \cdots, M\}$ that assigns one of $M$ distinct labels
to each point in $Z$. That is, for every point $j \in Z$
the value $l(j) = m$ indicates that $j$ belongs to class $m$ or is 
{\it labelled} as $m$.
Throughout
this article we assume that every point in $Z$ belongs to one class only.

Now let $Z' \subseteq Z$ be a subset of the nodes with  $|Z'| = J \le N$ and define
$y: Z' \mapsto \{1, 2, \cdots, M\}$ to be a {\it noisily observed  label} of each point
in $Z'$. We refer to $y$ as {\it labelled data}. With this setup we may define the SSL problem.

\begin{problem}[Semi-Supervised Learning]\label{SSL-general-form}
  Suppose $Z, Z', X$ and $y$ are known. 
  Find $l: Z \mapsto \{1, 2, \cdots, M\}$. 
\end{problem}

In order to solve this problem, which is highly ill-posed, it is necessary
to introduce some form of regularity on the labels, guided by the correlations
in $X$ for example, and to make assumptions
about the errors in the labels provided. One approach, which
we study here, is to assume that the labels on $Z$ are defined through a 
latent variable $u:Z \mapsto \mbb R^M$, whose regularity is defined through
the unlabelled data $X$, and a function $S: \mbb R^M \mapsto
 \{1, 2, \cdots, M\}$. Specifically we assume that there is a ground truth
function $u^\dagger: Z \mapsto \mbb R^M$ for which 
\begin{equation}
  \label{generalized-probit-error-model}
  y(j) := S\bigl(u^\dagger(j)  + \eta(j)\bigr), \qquad j \in Z',
\end{equation}
where $\eta(j) \iidsim \psi$ and
$\psi$ is the Lebesgue density of a 
zero-mean random variable on $\mbb R^M$. 
We may now introduce the following relaxation of the SSL problem.

\begin{problem}[Relaxed Semi-Supervised Learning]\label{relaxed-SSL-general-form}
Suppose $Z, Z', X$ and $y$ are known, together with the function $S$ and the
density $\psi.$ Find $u: Z \mapsto \mbb R^ M$ and define 
$l=S\circ u: Z \mapsto \{1, 2, \cdots, M\}$.
\end{problem}

In Problem \ref{probit-SSL-general-form} below we will define a class of
optimization functionals for $u$, giving an explicit instantiation of 
Problem \ref{relaxed-SSL-general-form}, and focus on the resulting optimization 
problems in our analysis. Before doing so we give a literature review explaining the 
context for this optimization approach.

The consistency of classification methods in the setting of
supervised learning is well-developed; see \cite{tewari2007consistency}
for a literature review and results applying to both binary
and multi-class classification, as well as the preceding work
in \cite{steinwart2001influence,steinwart2005consistency,wu2006analysis}
which establishes the problem in the framework of Vapnik \cite{vapnik1998statistical}.
The paper \cite{xu2009robustness} discusses the robustness
of such supervised classification methods, allowing for a small
fraction of adversarially labelled data points.  
There has been some recent analysis of logistic regression, and the reader
may access the literature on this subject via the recent papers
\cite{candes2020phase,sur2018modern}. All of this work on
supervised classification focuses
on the large data/large number of features setting, 
and often starts from assumptions that the unlabelled data is linearly separated.
None of it leverages the power of graph-based techniques to extract
geometric information in large unlabelled data sets.
To make the connection to graph-based techniques we need to discuss
unsupervised graph-based learning
\cite{belkin2002laplacian,von2007tutorial}. This is a subject
that has seen significant analysis in relation to consistency. The papers
\cite{spielmat1996spectral, spielman2007spectral} perform a careful analysis
of the spectral gaps of graph Laplacians resulting from clustered data,
studying recursive methods for multi-class clustering. The paper
\cite{ng2002spectral} introduced a way of thinking about, and analyzing,
multi-class unsupervised learning based on perturbing a perfectly clustered case;
we will leverage similar ideas in our work on SSL.
The paper \cite{von2008consistency} introduced the idea of studying the
consistency of spectral clustering in the limit of large i.i.d. data
sets in which the graph Laplacian converges to a limiting integral
operator; and the work \cite{trillos2018error, trillos2016variational} has
taken this further by working with localizing weight functions  designed
so that the limit of the graph Laplacian is a differential operator.

SSL is a methodology which combines the methods of
unsupervised learning and of supervised classification.
According to the definition in \cite{kostopoulos2018semi}
``SSL can be categorized into two somewhat different settings, namely inductive and 
transductive learning $\dots$ inductive SSL attempts to predict the labels on unseen 
future data, while transductive SSL attempts to predict the labels on unlabeled instances 
taken from the training set.'' In this paper our focus is on transductive SSL.
Initial attempts to solve the SSL problem employed combinatorial algorithms \cite{blum2001learning},
based on an explicit mathematical formulation stemming from Problem \ref{SSL-general-form}. 
Zhu and collaborators
introduced a relaxation similar to Problem \ref{relaxed-SSL-general-form}, leading to the 
influential papers \cite{zhu2003semi,zhu2003combining}. Their approach is most easily
described in the binary case in which they assume $S:\mbb R \mapsto \mbb R$ is the identity 
function and the labels are given in the form $\pm 1.$ From a modeling viewpoint this 
approach is unnatural because the categorical data is assumed to also lie in the real-valued 
space of the latent variable.
Bertozzi and Flenner \cite{{bertozzi2012diffuse}} introduced an interesting relaxation of 
this assumption, by means of a Ginzburg-Landau penalty term which favours real-values
close to $\pm 1$ but does not enforce the categorical values $\pm 1$ exactly. 
The probit approach to classification, described in the classic text on Gaussian process
regression \cite{rasmussen2006gaussian}, does not make the unnatural modeling assumption
underlying Zhu's work; instead it is based on taking $S$ to be the sign function. However the basic
form of probit in \cite{rasmussen2006gaussian} does not use unlabelled data
to extend labels outside the labelled data set, but instead does so through regular
Gaussian process regression: inductive SSL.

The extension of the probit method to
graph-based transductive SSL is described in \cite{bertozzi2018uncertainty}, where both Bayesian and optimization-based formulations are described;
in that paper, \eqref{generalized-probit-error-model} 
is also generalized to the level set form
\begin{equation}
  \label{generalized-probit-error-model2}
  y(j) := S\bigl(u^\dagger(j)\bigr)  + \eta(j), \qquad j \in Z',
\end{equation}
and a Bayesian formulation of the Ginzburg-Landau relaxation of  \cite{bertozzi2012diffuse}
is introduced.
The close relationship between level set and probit formulations is
discussed in \cite{stuart-zeronoiseSSL}.
The work of Belkin \cite{belkin2004regularization, belkin2002laplacian, belkin2006manifold}
demonstrates how both Gaussian process regression and graph-based
SSL can be used simultaneously; in the sense of the definition in \cite{kostopoulos2018semi},
transductive and inductive SSL are combined. All of the approaches which followed from the work of Zhu are 
readily generalized from the binary case to the multi-class setting, using the
idea of one-hot encoding, explained in detail in
subsection~\ref{sec:31}, in which each label is identified with 
a standard unit basis vector in ${\mbb R}^M$.

A large number of approaches to SSL have been developed in the literature
and a detailed discussion of all of them is outside the scope of this article.
We refer the reader to the  review articles  \cite{zhu2005semi}
and \cite{kostopoulos2018semi} for, respectively, the state-of-the-art in 2005
and a more recent appraisal of the field that categorizes various inductive 
and transductive approaches to SSL and semi-supervised regression.
The idea of regularization by graph Laplacians for SSL was developed in different
contexts such as manifold regularization
\cite{belkin2006manifold}, Tikhonov regularization \cite{belkin2004regularization}
and local learning regularization \cite{wu2007transductive}. 
However, while graph regularization methods are widely applied in practice
the rigorous analysis of their properties,
and in particular asymptotic consistency, 
is not well-developed within the context of SSL.
Indeed, to the best of our knowledge the consistency analysis of the
probit and one-hot methods has not been tackled  before.
SSL may be viewed as a method for boosting, refining or questioning
unsupervised graph-based learning, through labelling information; our
analysis sheds light on this process.

There has been other analysis of SSL methods, not concerning consistency.
In \cite{stuart-zeronoiseSSL} the authors studied the large data and zero
noise limits of the probit method. They derive a continuum
inverse problem using the methodology of \cite{trillos2018error, trillos2016variational}
that characterizes SSL when the number of
vertices of the graph and the number of observed labels is fixed, or goes to infinity
in a manner insuring a fixed fraction of labels. 
The authors also study the zero noise limit of
probit and level-set methods for SSL and show that both problems
approach the same limit as the noise variance goes to zero. 
In forthcoming papers \cite{HHOS1,HHOS2} we will build on this body of work
to study consistency of graph-based SSL in the
limit of large unlabelled data sets.

\subsection{Problem Setup And Preliminaries}\label{sec:problem-statement}
Our focus in this paper is on the analysis of algorithms built from
the introduction of  real-(vector)-valued latent functions, leading to precise
mathematical formulations of Problem \ref{relaxed-SSL-general-form}.
To make actionable algorithms we need to specify precisely how
the unlabelled data $X$ and the labelled data $y$ are used. The approach
we study here is to define the desired latent variable $u$ 
as the minimizer of a function comprised of two terms, one of which
enforces correlations and geometric information in the unlabelled data $X$, 
and the other which enforces consistency with
the label data $y$, on the assumption that they are related to $u$
as in \eqref{generalized-probit-error-model}.
To this end we view $X$ as a point cloud in $\mbb R^d$ and
associate a weight matrix $W =(w_{ij})$ to tuples $(x_i, x_j)$ in $X\times X$.
The weights $w_{ij}$, which are assumed to be non-negative, 
are chosen to measure affinities between $x_i$ and $x_j$.
Since similarity between data points is a symmetric relationship, we assume $w_{ij} = w_{ji}$ so that $W$ is a symmetric
matrix and define a proximity graph $G = \{ X, W\}$
with vertices $X$ and edge weights $W$. From $W$ we will construct a covariance
operator $C$ on spaces of functions $H=\{u:Z \mapsto \mbb R^ M\}$, using a graph Laplacian implied by $W$. 
We also define a misfit function $\Phi(\cdot\,; \cdot\,): H \times \{1, \cdots ,M\}^{J} \mapsto \mbb R$ which encodes
the assumption \eqref{generalized-probit-error-model} about the relationship between the labels and the
latent function. With these objects we then   
formulate the SSL problem as a regularized optimization problem.

\begin{problem}[Relaxed Semi-Supervised Learning As Optimization]\label{probit-SSL-general-form}
Suppose $Z, Z', X$ and $y$ are known, together with the function $S$, the covariance
operator $C$ and the misfit $\Phi$. Find the function $u^\ast$ defined by
  \begin{equation}\label{generic-optimization-SSL-minimizer}
    u^\ast = \argmin_{u \in H} \frac{1}{2} \langle u, C^{-1} u \rangle_{H} +
    \Phi(u; y).
  \end{equation}
\end{problem}
  
This optimization problem may be viewed as the MAP estimator associated
to the Bayesian inverse problem of finding the distribution of $u|y$
when the prior on $u$ is a Gaussian random measure on $H$ with covariance $C$
and $\Phi(u;y)$ is the negative log-likelihood of $y$ conditioned on $u$, i.e.
  \begin{equation}\label{likelihood-definition}
    \mathbb{P}(y |u) \propto \exp\bigl( - \Phi(u;y)\bigr), \qquad \text{assuming} \qquad
    {y}(j) = S\bigl(u(j) + \eta(j)\bigr).
  \end{equation}
We refer to $\Phi$ as the {\it likelihood potential}.

\subsection{Main Contributions}
The key question at the heart of this article is to identify conditions
under which the minimizer $u^\ast$ of Problem~\ref{probit-SSL-general-form}
correctly identifies the labels. To this end, we define the following
notion of consistency.

\begin{definition}[SSL asymptotic consistency]
  We say that Problem~\ref{probit-SSL-general-form} is asymptotically
  consistent if, for all $j\in Z$,
  \begin{equation*}
    S(u^\ast(j)) \asto S(u^\dagger(j)), \qquad\text{as} \qquad
    {\rm{std}} (\eta(j)) \downarrow 0,
  \end{equation*}
where $u^\dagger$ is the latent variable underlying the labelled
data \eqref{generalized-probit-error-model}.
\end{definition} 

In the above and throughout the rest of the
article $\asto$ denotes almost sure (a.s.) convergence
with respect to a common probability space on which the
measurement noise $\eta(j)$ are defined (see subsections~\ref{sec:binary-SSL-consistency}
and \ref{sec:one-hot-SSL-consistency} for a formal discussion of this mode of convergence).
We primarily focus on the probit and one-hot methods for SSL, corresponding to specific
choices of the function $S.$
As mentioned earlier probit is an optimization approach for binary classification that formulates
Problem~\ref{probit-SSL-general-form} with $M=2$. The one-hot method is a generalization of
 probit for multi-class classification when $M\ge 2$. We outline
these methods in detail in sections~\ref{sec:binary-classification}
and \ref{sec:multi-class-classification}.  
We show that probit and one-hot methods are asymptotically consistent
in the case where the graph $G$ is \as{nearly-disconnected} in the following sense.

\begin{definition}[\as{Nearly-disconnected} graph]\label{def:nearly-seperable-graph}
  A weighted graph $G=\{X, W\}$ is \as{nearly-disconnected} with $K$ clusters
  if there exist connected components
  $\tG_k= \{ \tX_k, \tW_k\}$ for $k \in \{1,\cdots, K\}$ so that the edges within
  each $\tG_k$ are $\mcl O(1)$,  but the edges between elements in different $\tG_k$ 
are $\mcl O(\eps)$ for a
  small parameter $\eps >0$. In other words, up to a reordering of the index set $Z$,
  the matrix $W$ is nearly block diagonal.
\end{definition}

Working in such a setting is a natural way of representing nearly clustered data,
and was exploited in the paper \cite{ng2002spectral} concerning unsupervised 
learning. The number of clusters $K$ is an inherent
geometric property of the unlabelled data $X$; determining a suitable choice of $K$ in practice can be challenging and depends on the scale one is interested in.
In the following informal statement of  our main  result we
assume that each component $\tG_k$ is associated with at least one pre-assigned label. The result shows that if $G$ is \as{nearly-disconnected} and the ground 
truth function $u^\dagger$ assigns the same label to all points within each 
component $\tG_k$ then the probit and one-hot methods are asymptotically 
consistent for an appropriate choice of matrix $C$ so long as
  at least one label is observed in each component $\tG_k$. Below, $L$ denotes the \emph{graph Laplacian}, a discrete diffusion operator acting on functions defined on the graph $G$, see section~\ref{sec:covariance-operator-definition} for a precise definition. 

\begin{theorem}[Consistency of probit and one-hot]\label{informal-asymptotic-consistency}
Suppose $G$ is a \as{nearly-disconnected} graph and let $L$ be a graph Laplacian
on $G$. Define the matrix $C = \tau^{2\alpha}(L + \tau^2I)^{-\alpha}$ with parameters
$\tau^2, \alpha >0$.  Assume  $S(u^\dagger)$ is constant on the components $\tG_k$ 
and at least one label is observed in each component $\tG_k$.
Then the probit and one-hot formulations are asymptotically consistent
for any sequence $(\eps, \tau, {\rm{std}} (\eta)) \downarrow 0$ 
along which $\eps = o(\tau^{2}).$
\end{theorem}

\begin{remark}
\as{Conceptually the parameter $\epsilon$ should be thought of as an 
inherent measure of how clustered the unlabelled data is; in this
paper we consider a specific set-up in which $\epsilon$ is defined 
as a measure of the size of edge weights between clusters. 
We also connect the labelled and unlabelled data 
via a model involving an unobserved latent variable, perturbed by 
noise $\eta$. Our consistency results are proven in the setting in
which $\epsilon$ and ${\rm{std}} (\eta))$ both tend to zero. This
is a strong assumption which, whilst allowing a precise theory,
may be difficult to apply directly in practice.
We believe that similar consistency results will hold under
different modeling assumptions which characterize clustering and
label noise in more general ways.  
Furthermore our consistency results demonstrate the importance of
choosing the hyperparameter $\tau$ in a data-dependent fashion.
Small $\epsilon$ induces a spectral gap in $L$ and for this
to translate into a spectral gap in $C$ we require $\tau$ to be
small too. However we also require $\eps = o(\tau^{2})$ so that the
number of eigenvalues of $C$ which are at, or close to, $1$ is 
the same as the number of clusters in the data. 
We also give theory and numerical evidence
showing that when $\eps = \Theta(\tau^{2})$ consistency may be lost.
Our theoretical results are asymptotic in nature and therefore cannot
apply directly to any one given data set. However our analysis
provides insights
into both algorithmic parameters choices, and algorithmic performance, 
in practical non-asymptotic set-ups. Indeed the papers 
\cite{bertozzi2012diffuse,bertozzi2018uncertainty} demonstrate
the use of optimization methodologies of the type introduced
here in practical non-asymptotic set-ups for real data problems, 
and the papers 
\cite{bertozzi2018uncertainty,chen2018robust,qiao2019uncertainty}
demonstrate analogous set-ups for related Bayesian approaches. 
An important conclusion of the theory and numerical experiments is that
careful choice of parameter $\tau$ is crucial for effective SSL.
The take-home message here is that the use of
hierarchical Bayesian methods, which tune $\tau$ automatically to the data,
can be beneficial; as demonstrated in practical experiments in \cite{chen2018robust}.}
\end{remark}

Formal statement and proof of the preceding main theorem
is given in Theorem~\ref{multiple-observation-consistency} 
(together with Corollaries~\ref{probit-asymptotic-consistency-single-observation} and 
\ref{probit-asymptotic-consistency-multiple-observations})
for the probit method and in Theorem~\ref{one-hot-consistency-multiple-observation}
(together with Corollary~\ref{one-hot-asymptotic-consistency-multiple-observations})
for the one-hot method.

As a secondary result, accompanying the preceding theorem,
we identify a natural dimension reduction
for probit and one-hot optimization problems. More precisely, we show that
finding $u^\ast \in H=\{u: Z \mapsto \mbb R^M\}$ is equivalent to
a similar optimization problem for a function $b^\ast \in \tH=\{b: Z' \mapsto \mbb R^M\}$.
Thus we can reduce the size of the optimization problems from
$N \times M$ to $J \times M$. This result, which is a discrete representer theorem,
has significant practical consequences when $J \ll N$.

\begin{theorem}[Dimension reduction for probit and one-hot]\label{thm:main2}
  The problem of finding $u^\ast$ is equivalent to an optimization problem of the form
  \begin{equation*}
    b^\ast  := \argmin_{b \in  \tH} \frac{1}{2} \langle b, (\tC)^{-1} b \rangle_{\tH}
    +  \tPhi( b; y),
  \end{equation*}
  where $\tC$ is a submatrix of $C$ after restriction of rows and columns to $Z'$,
and $\tPhi$ is defined from $\Phi$.
\end{theorem}

\as{A formal statement and proof of this theorem is presented in Corollary~\ref{p:binary-representer-theorem}
  for probit and Proposition~\ref{one-hot-dimension-reduction} for the one-hot method.
  These results also provide identities that relate $b^\ast$ to $u^\ast$ and vice versa. More precisely,
   pointwise values of the functions $b^\ast$ and 
   $u^\ast$ coincide on the labelled set $Z'$. Conversely, $u^\ast$ can be viewed as a smooth extension
 of $b^\ast$ from the labelled set $Z'$ to the entire index set $Z$.}

Finally we perform numerical experiments to illustrate the
behavior of probit and one-hot methods beyond the theoretical setting.   
In particular we demonstrate that when  $\eps = \Theta(\tau^2)$ these methods 
are not always consistent. An interesting observation we make is a sharp 
phase transition in the accuracy of both methods. More precisely, we observe a curve in the $(\eps/\tau^2,\alpha)$-plane  across which the probit and one-hot methods transition rapidly
from being consistent into inconsistent solutions based on majority label
  propagation, i.e.,  labelling all points in the dataset according to the label that is observed most often (see Figures~\ref{num-exp-probit-consistency} and
   \ref{num-exp-one-hot-accuracy-multiple-observations-in-clusters}). 
Intuitively this happens because, for larger values of $\eps/\tau^2$, it is
  cheaper to minimize the quadratic regularization term in the optimization problem of Theorem~\ref{thm:main2} than  to minimize  the misfit term $\Phi$.


\subsection{Outline}
Section~\ref{sec:binary-classification} is devoted to analysis of the probit method
where $M=2$.  The problem is formulated as inference for a latent
real-valued function on the nodes of a graph, with the sign determining the assignation
of a binary label. An optimization approach is
employed in which the graph Laplacian constructed from the unlabelled data
is used for regularization, and a
generic zero-mean log-concave label measurement noise is assumed; this results in a
convex data misfit term.  
We study the properties of this optimization problem, showing that the related
optimization functional is convex. We prove a
representer theorem and then study asymptotic consistency of the method
in Corollary \ref{probit-asymptotic-consistency-single-observation},
Theorem~\ref{multiple-observation-consistency} and 
Corollary~\ref{probit-asymptotic-consistency-multiple-observations}, the
precise statements of Theorem~\ref{informal-asymptotic-consistency}
in the probit case.

Section~\ref{sec:multi-class-classification} has the same structure as section~\ref{sec:binary-classification} but focuses on the multi-class setting (i.e., $M \ge 2$) and
employs the one-hot method to link a real-vector-valued latent variable to the labels. 
The key results here are Corollary \ref{one-hot-asymptotic-consistency-single-observation},
Theorem~\ref{one-hot-consistency-multiple-observation} and Corollary~\ref{one-hot-asymptotic-consistency-multiple-observations}, the precise versions of
 Theorem~\ref{informal-asymptotic-consistency} in case of the one-hot method.
  Section~\ref{sec:NE} contains numerical experiments confirming the
  key theoretical results from the two preceding sections, and illustrates the
  behavior of probit and one-hot methods beyond the theoretical setting.
In section \ref{sec:conc} we summarize and discuss future work.

\subsection{Notation}

Throughout we use $Z$ to denote the nodes of a graph carrying a pre-assigned unlabelled data    point at each node, and $Z'$ the subset of nodes which also carry a label.
\as{We use $G_0 = \{ X, W_0\}$ to denote a disconnected graph with $K$ disconnected subgraphs (clusters) $\tG_k = \{ \tX_k, \tW_k\}$ for $k=1, \dots, K$. The $\tX_k$ are
  a subset of the points in $X$ with indices $\tZ_k \subset Z$ while $\tW_k$ are submatrices of $W_0$. We also use $\tZ_k'$ to denote the
  subset of labelled points within $\tZ_k$. Subsequently we denote the graph Laplacian matrices of the subgraphs $\tG_k$ by $\tL_k$.
  We also introduce a nearly-disconnected graph $G_\eps = \{ X, W_\eps\}$ with the weight matrix $W_\eps$ that is considered
  to be a perturbation of $W_0$ and use $L_\eps$ to denote the graph Laplacian on this nearly-disconnected graph. These concepts are
introduced and discussed in subsection~\ref{sec:covariance-perturbation} and used extensively in the rest of the article.}
  
We use
$u$ to denote real-(vector)-valued functions on $Z$ which are acted upon by a
nonlinear classifier to assign labels.  
We use $|\cdot|$ to denote the cardinality of a set; $\langle \cdot\,,\cdot\,\rangle, \|\cdot\|$
denote the Euclidean inner-product and norm unless stated otherwise.
We employ the standard $\Theta$, $\mcl O$ and $o$ notations as in \cite{cormen2009introduction}:
given positive functions $f(s), g(s)$, we  write
\begin{itemize}
\item  $f(s) = \Theta(g(s))$ if there 
exist constants $c_1, c_2, s_0>0$ so that $$0 \le c_1 g(s) \le f(s) \le c_2 g(s)\qquad \forall\, s \in(0, s_0]\,,$$
\item  $f(s) = \mcl O(g(s))$ if there exists $c, s_0> 0$ so that 
$$0 \le f(s) \le c g(s)\qquad \forall\,s \in(0, s_0]\,,$$
\item \as{$f(s) = o(g(s))$} if for any 
constant $c >0$ there exists $s_0(c) >0$ so that $$0 \le f(s) < c g(s)\qquad \forall\,s \in (0, s_0(c)]\,.$$
\end{itemize}

\section{Binary Classification: The Probit Method}\label{sec:binary-classification}

In subsection \ref{sec:21} we set up the probit methodology, noting that the
binary classification problem ($M=2$) can be formulated using a latent variable function
which is $\mbb R^{M-1}-$valued rather than  $\mbb R^{M}-$valued. In
subsection \ref{ssec:22} we study the likelihood contribution to the optimization
problem, resulting from the labelled data,
and in subsection \ref{sec:covariance-operator-definition} 
the quadratic regularization resulting from the unlabelled data.
In subsection \ref{sec:properties-binary-proibit-minimizer} we
study the probit minimization problem, formulating the results via a discrete representer
theorem, and in subsection \ref{sec:covariance-perturbation} we study the
properties of the representers via the properties of the eigenstructure of
the covariance, exploiting the \as{nearly-disconnected} graph structure.
subsection \ref{sec:binary-SSL-consistency} concludes the analysis of the probit method,
studying consistency in some detail.

\subsection{Set-Up}
\label{sec:21}

We start with the case of binary classification where the nodes $Z$ belong to only two
classes. For simplicity we assume that $l(j) \in \{ -1, +1 \}$ for all $j \in Z$
rather than taking $l(j) \in \{1, 2\}$.
This assumption is at odds with our notation in subsection~\ref{sec:problem-statement}
but allows for a simpler formulation of Problem~\ref{probit-SSL-general-form}.
 Since the classes
are identified with the integers $+1$ and $-1$ a natural choice for
the classifier function $S$ is the sign function:
\begin{equation}
  \label{S-sign-func}
  S: \mbb R \mapsto \{-1, +1\}, \qquad
  S(t) = \sgn(t) := \left\{
    \begin{aligned}
      &+1 ,\quad \text{ if } t  \ge 0, \\
      & -1, \quad \text{ if } t < 0.
    \end{aligned}\right.
\end{equation}
With the above choice for $S$ we can take the latent variable $u$
to be a real valued function on $Z$, i.e., $u: Z \mapsto \mbb{R}$.
We can then naturally  identify the function $u$ with a vector $\mbf u \in \mbb{R}^N$
where $\mbf u = ( u_1, u_2, \cdots, u_N )^T$ and $u_j = u(j)$ for $j \in Z$.
This allows to view Problems~\ref{relaxed-SSL-general-form} and \ref{probit-SSL-general-form} as
the inverse problem of
finding a vector $\mbf u^\ast$ in $\mbb R^N$.
In the remainder of this section we will utilize this vector notation
for convenience.

\subsection{The Probit Likelihood}
\label{ssec:22}
Let us begin by 
deriving the likelihood potential $\Phi(u;y)$ for the probit method.
Let $S$ be as in \eqref{S-sign-func}
and recall \eqref{generalized-probit-error-model}, then
\begin{equation*}
  {y}(j) = \sgn(u_j + \eta_{j}), \qquad \eta_{j} \iidsim \psi, \quad j \in Z,
\end{equation*}
wherein we have identified the noise $\eta$ with a vector 
$\pmb \eta =(\eta_1, \dots, \eta_N)^T \in \mbb R^N$.
Suppose  $\psi$ is a symmetric  probability density function on $\mbb R$ and denote the
the cumulative distribution function (CDF) of $\psi$ by
$\Psi$. Then,
\begin{equation*}
  \mbb P (y(j) = +1 | u_j) = \mbb P  ( -u_j \le \eta_j ) = \mbb P ( -u_j y(j) \le \eta_j)
   = \Psi(u_j y(j)).
 \end{equation*}
 For more details on this calculation, see similar arguments for the multi-class case in section~\ref{ssec:32}.
 Similarly,
\begin{equation*}
  \mbb P (y(j) = -1 | u_j) = \mbb P  ( -u_j > \eta_j ) = \mbb P ( u_j  y(j) > \eta_j)
   = \Psi(u_j  y(j)).
 \end{equation*}
 From \eqref{likelihood-definition} it follows that the {\it probit likelihood} potential
 $\Phi(u;y)$ has the form
 \begin{equation}
   \label{binary-probit-likelihood}
   \Phi(u;y) = -\sum_{j\in Z'} \log \Psi(u_jy(j)).
 \end{equation}

 \subsection{Quadratic Regularization Via Graph Laplacians (Binary Case)}
 \label{sec:covariance-operator-definition}
 Let us now formulate a quadratic regularization term for the probit method.
Recall our encoding of the nodes $Z$ and their similarities
via a weighted graph $G = \{ X, W\}$ with vertices at $x_j$ and edge weights $w_{ij}= w_{ji}$
for $i,j \in Z$. We denote by $d_i$ the degree of each node $i \in Z$ as
\begin{equation*}
  d_i := \sum_{j\in Z} w_{ij},
\end{equation*}
and further define the diagonal matrix $D := {\rm diag} (d_i) \in \mbb R^{N \times N}$. Finally, given constants
$p, q \in \mbb R$ we define
the graph Laplacian operator on $G$
\begin{equation}
  \label{graph-Laplacian}
  L := D^{-p} ( D - W) D^{-q} \in \mbb R^{N\times N}.
\end{equation}
Different choices of $p$ and $q$ result in different normalizations of the
graph Laplacian, see \cite{Belkin2006ConvergenceOL, chung1997spectral, CoifmanLafon2006, stuart-zeronoiseSSL, trillos2018error, trillos2016variational, shimalik2000, slepcev2019analysis, von2007tutorial, von2008consistency} and the references therein \as{as well as \cite[Sec.~5]{HHOS1} where a detailed discussion
  around various weightings of graph Laplacians and their connection to a family of elliptic
operators is laid out.}  For example, $p = q=0$ leads to the usual {\it unnormalized} graph Laplacian,
when $ p = q = 1/2$ we obtain the {\it symmetric normalized} graph Laplacian, 
and
$p = 1$ and $q = 0$ gives the {\it random walk} graph Laplacian.
Different normalizations of the graph Laplacian have been used for spectral clustering in the literature, but a thorough understanding of the advantages and disadvantages of certain parameter choices is still lacking, see \cite{von2007tutorial}.
Throughout we enforce $p=q$ in order to make $L$ symmetric with respect to the
Euclidean inner-product, making no other assumptions regarding the value of $p,q$; however our results can be generalized to $p \ne q$
by using appropriate $D-$weighted inner-products.
For $p=q$, we can then write for any vector $\mbf x\in \mbb{R}^N$,
\begin{equation}\label{xLx} 
    \langle \mbf x,L \mbf x \rangle 
    =\frac12 \sum_{i,j=1}^N w_{ij} \left|\frac{\mbf x_i}{d_i^p} - \frac{\mbf x_j}{d_j^p} \right|^2\,.
\end{equation}
Given a graph Laplacian $L$ and parameters $\alpha, \tau^2 > 0$ we define a family
of covariance operators
\begin{equation}
  \label{covariance-matrix}
  C_{\tau} = \tau^{2\alpha}(L + \tau^2 I)^{-\alpha} \in \mbb R^{N\times N},
\end{equation}
where $I \in \mbb R^{N \times N}$ denotes the identity matrix. We then use this covariance
matrix to define the quadratic regularization term in Problem~\ref{probit-SSL-general-form}.
To this end note that in the binary case we may identify $H=\mbb R^N$; we make this
identification in what follows in this section, 
and $\langle \cdot\,,\cdot\, \rangle$ then denotes
the standard Euclidean inner-product.

\begin{remark}
  We use the term covariance operator to refer to the matrix $C_{\tau}$ following the
  connection between optimization problems of the form  \eqref{generic-optimization-SSL-minimizer}
  and MAP estimators within the Bayesian formulation of probit given in \cite{bertozzi2018uncertainty}.
  In the Bayesian perspective  $C_{\tau}$ is the covariance operator of a Gaussian prior
  measure on $\bu$,  and $\bu^\ast$ coincides with the 
  MAP estimator of $\bu^\dagger$.
\as{We note that the covariance $C_{\tau}$ may be viewed as a form of
discrete Mat{\'e}rn covariance, in the framework of \cite{Rue}.
The scaling of $C_{\tau}$ that we adopt ensures that the spectrum of $C_{\tau}$ lies
in $[0,1]$ and hence controls the total variance of samples $u$ from the
prior: $\mathbb{E}^{N(0,C_{\tau})} \|u\|^2 \approx K$, where $K$ is the number of clusters in the disconnected or
nearly-disconnected graph setting. \as{Further study of the connections between $C_\tau$ and Mat{\'e}rn
  kernels is outside the scope of this article and is postponed  to our companion papers \cite{HHOS2,HHOS1}
  where the continuum limits of graph Laplacian and covariance matrices such as $C_{\tau}$
  are studied when the elements of $X$ are drawn i.i.d. 
at random from a probability measure, and $N \to \infty,$
building on \cite{trillos2016variational}.} }
\end{remark}

\subsection{Properties Of The Probit Minimizer}\label{sec:properties-binary-proibit-minimizer}
With the likelihood $\Phi$ and covariance matrix $C_\tau$ identified we
can now  discuss properties of the {\it probit functional}
\begin{equation}
  \label{probit-functional}
  {\msf J}(\mbf u) := \frac{1}{2} \langle \mbf u, C_\tau^{-1} \mbf u \rangle + \Phi(\mbf u; y),
  \qquad \mbf u \in \mbb R^N.
\end{equation}
\begin{remark}
In the following we will study the problem of minimizing ${\msf J}$. \as{We 
highlight the fact that related optimization problems for objective
functions of the form 
\begin{equation}
  \label{probit-functional_2}
  {\msf J}(\mbf u) := \frac{1}{2} \langle \mbf u, L \mbf u \rangle + \Upsilon(\mbf u; y),
  \qquad \mbf u \in E
\end{equation}
have been defined and studied in 
\cite{bertozzi2012diffuse,bertozzi2018uncertainty, zhu2003semi}, although
asymptotic consistency has not been investigated there.
In order to give these methods a Bayesian interpretation we need to
define a covariance $C=L^{-1}$, noting that $L$ is invertible on the
set $E=\{\mbf u \in  \mbb R^N: \langle \mbf u,  D_0^p \mbf 1 \rangle=0\}$, 
where  $\mbf 1 \in \mbb R^N$ denotes the vector of ones. 
$L$ is invertible on $E$ because $D_0^p \mbf 1$ spans
the null-space of $L$ when the graph is pathwise connected.
Introduction of $C_{\tau}$ with $\tau>0$ not only circumvents the need to
work on $E$ but also allows for consistent prior modeling of the situation
in which multiple clusters have the same prior label. Furthermore
the parameter $\alpha$ is needed when the large data limit $N \to
\infty$ is considered; see \cite{HHOS1}.}
\end{remark}

Our first task is to prove existence and uniqueness of the minimizers of $\msf J$
by proving it is strictly convex.
The following proposition follows directly from
  \cite[Thm.~1]{bagnoli2005log} and
states that the CDF of a log-concave probability distribution function (PDF) is also log-concave.

\begin{proposition}[Convexity of the likelihood potential $\Phi$]\label{convexity-of-Phi}
  Let $\psi$ be a continuously differentiable, symmetric and strictly log-concave PDF with
  full support on $\mbb R$.
  Then $\Psi$  is also strictly log-concave and so $\Phi (\cdot; y):
  \mbb R^N \mapsto \mbb R $ is strictly convex.
\end{proposition}

Convexity of the quadratic regularization term in \eqref{probit-functional}
follows directly from Lemma~\ref{C-is-strict-positive-definite}
that establishes that the matrix $C_\tau$ is strictly positive-definite whenever $\tau^2, \alpha >0$. 
With the convexity of both terms in the definition of  $\msf J$ established
we can now characterize its minimizer.

 \begin{proposition}[Representer theorem for the probit functional]\label{binary-representer-theorem}
  Let $G = \{ X, W\}$ be a weighted graph  and
  let $\psi$ be a PDF that is continuously differentiable, symmetric and  strictly log-concave  with full support on $\mbb R$. Suppose the likelihood potential $\Phi$ is
  given by \eqref{binary-probit-likelihood} and the
  matrix $C_\tau$ is given by \eqref{covariance-matrix} with parameters $\tau^2, \alpha >0$.
  Then the following hold.
  \begin{enumerate}[(i)]
  \item The probit functional $\msf J$ has a unique minimizer
    $\mbf u^\ast \in \mbb R^N$.

  \item The minimizer $\mbf u^\ast$ satisfies the Euler-Lagrange (EL) equations
    \begin{equation}\label{binary-EL}
      C_\tau^{-1} \mbf u^\ast = \sum_{j \in Z'} F_j( u^\ast_j) \mbf e_j,
    \end{equation}
    \as{where  $\mbf e_j$ is the $j$-th standard
    coordinate vector in $\mbb R^N$ and
    \begin{equation}\label{def-F-j}
    F_j(s):= \frac{y(j) \psi(s y(j))}{\Psi(s y(j))}.
  \end{equation}
}

  \item The minimizer $\mbf u^\ast$ has a sparse representation 
    \begin{equation}\label{binary-minimizer-expansion}
      \mbf u^\ast = \sum_{j \in Z'} \ta_j \mbf c_j,
    \end{equation}
    where $ C_\tau \mbf e_j =: \mbf c_j = (c_{1j},\cdots, c_{Nj})^T$ are a subset of the column space of $C_\tau=(c_{ij})_{i,j\in Z}$ and $\ta_j \in \mbb R.$

    \item The vector $\mbf u^\ast$ defined in \eqref{binary-minimizer-expansion}
    solves \eqref{binary-EL} if and only if the coefficients  $\ta_j$ satisfy the
    non-linear system of equations
    \begin{equation*}
     \ta_j =  F_j\left( \sum_{k \in Z'}  \ta_k c_{jk} \right), \qquad \forall j \in Z'.
    \end{equation*}
\end{enumerate}

\end{proposition}

\begin{proof}
 (i)  Since $\Psi$ is the CDF of a random variable on $\mbb R$ with
  full support then $\Psi(s) \in (0,1)$. Thus, $-\log \Psi \ge 0$ and so
  $\Phi$ is bounded from below. Furthermore, $\Phi$ is  convex following
  Proposition~\ref{convexity-of-Phi}.
  On the other hand, the matrix $C_\tau^{-1}$ is strictly positive definite
  following Lemma~\ref{C-is-strict-positive-definite} and so
  the quadratic term $\frac{1}{2} \langle \mbf w, C_\tau^{-1} \mbf w \rangle$ is strictly
  convex and positive. Thus, since the functional $\msf J$ is bounded from below and is the
  sum of strictly convex functions then $\msf J$ is strictly
  convex and has a unique minimizer. 

  (ii) Since $\psi$ is $C^1(\mbb R)$ the CDF $\Psi$ is $C^2(\mbb R)$ and 
  $\psi/\Psi$ is $C^1(\mbb R)$ 
  and locally bounded since $\psi$ has full support. Then $\msf J: \mbb R^N \mapsto \mbb R$
  is differentiable and the minimizer $\mbf u^\ast$ satisfies
  the first order optimality condition $\nabla \msf J(\mbf u^\ast) = 0$.
  The statement now follows by directly computing the gradient of $\msf J(\bu)$
  with respect to $\bu$.

  (iii--iv) Multiply \eqref{binary-EL} by $C_\tau$ to get
  \begin{equation*}
    \mbf u^\ast = \sum_{j\in Z'} F_j(u^\ast_j) C_\tau \mbf e_j = \sum_{j\in Z'} \ta_j \mbf c_j,
  \end{equation*}
  where we set $\ta_j = F_j(u^\ast_j)$ for $j \in Z'$. Now substitute the expansion of $\mbf u^\ast$
  into the definition of $\ta_j$ to get
  \begin{equation}\label{binary-aj-nonlinear-system}
    \ta_j = F_j\left( \left(\sum_{k \in Z'} \ta_k \mbf c_k \right)_j \right),
  \end{equation} This establishes the ``only if'' statement in (iv).
  In order to establish the converse, suppose the $\ta_j$ satisfy 
  \eqref{binary-aj-nonlinear-system}. Multiply this equation by $\mbf c_k$ and
  sum over $j \in Z'$ to get
  \begin{equation*}
    \sum_{j \in Z'} \ta_j \mbf c_j = \sum_{j \in Z'}
    F_j\left( \sum_{k \in Z'}  \ta_j c_{jk} \right) \mbf c_j.
  \end{equation*}
  now define $\mbf u^\ast = \sum_{j \in Z'}  \ta_j \mbf c_j$ to get  
  \begin{equation*}
    \mbf u^\ast = \sum_{j \in Z'} F_j( u^\ast_j) \mbf c_j.
  \end{equation*}
  The claim follows by multiplying
  this equation by $C_\tau^{-1}$. 
\end{proof}

\begin{remark}[Connection to kernel regression]
We note that Proposition~\ref{binary-representer-theorem} is closely related to the 
representer theorem in Gaussian process and kernel regression \cite[Sec.~6.2]{rasmussen2006gaussian}. Similar result to ours can be 
found in \cite[Thm.~1]{smola1998kernel} and  \cite{scholkopf2002learning}.
\end{remark}

Part (iv) of Proposition~\ref{binary-representer-theorem} suggests  that
the problem of minimizing $\msf J$ is analogous to a low-dimensional
optimization problem. To this end we now define
a one-to-one reordering
\begin{equation}\label{reordering-index-map}
  \pi: Z' \mapsto \{ 1, 2, \cdots, J\}, \qquad  \pi^{-1}: \{1, 2, \cdots, J\} \mapsto
  Z',
\end{equation}
that allows us to 
associate the coefficients $\{ \ta_j \}_{j \in Z'}$
with a vector $\mbf a = (a_1, \cdots a_J)^T\in \mbb R^J$ via 
\begin{equation*}
   a_{\pi(j)} = \ta_j, \qquad j \in Z',
\end{equation*}
and define submatrix $\tC \in \mbb R^{J\times J}$ by
the identity
\begin{equation}\label{sub-matrix-C-tilde}
  (\tC_\tau)_{\pi(i), \pi(j)} = c_{ij}', \qquad i,j \in Z'.
\end{equation}
That is, $\tC_\tau$ is the matrix $C_\tau$ with the rows and columns
of the indices in  $Z \setminus Z'$  removed.
Finally, we define $\mbf b := \tC_\tau \mbf a$. We then have the following
natural dimension reduction for the  probit optimization problem.

\begin{corollary}[Probit dimension reduction]\label{probit-dimension-reduction} \label{p:binary-representer-theorem}Suppose the conditions of Proposition~\ref{binary-representer-theorem} are
  satisfied. Then the following hold.  
  \begin{enumerate}[(i)]
  \item The problem of finding the minimizer
    $\mbf u^\ast \in \mbb R^N$ of the functional $\msf J$ is
    equivalent to the problem of finding the vector
    $\mbf b^\ast \in \mbb R^J$ that solves
  \begin{equation}\label{dimension-reduced-binary-EL}
    (\tC_\tau)^{-1} \mbf b^\ast = F'(\mbf b^\ast),
  \end{equation}
  where the map $F': \mbb R^J \mapsto \mbb R^J$ is defined as
  \begin{equation*}
    F'(\bv) = (f_1(v_1), \cdots, f_J(v_J))^T, \qquad 
    f_k(v_k):= F_{\pi^{-1}(k)}( v_k)
  \end{equation*}
\as{and $F_j$ are defined in \eqref{def-F-j}.}
  \item Moreover, the vector $\mbf b^\ast$ solves the optimization problem
  \begin{equation*}
   \mbf b^\ast = \argmin_{\mbf v \in \mbb R^J} \: {\msf J'}(\mbf v),
 \end{equation*}
 where
 \begin{equation*}
{\msf J'}(\mbf b) := \frac{1}{2} \langle \mbf b, (\tC_\tau)^{-1} \mbf b \rangle +
\tPhi(\mbf b;y),
\end{equation*}
and
$$\tPhi(\mbf b;y)=-\sum_{j=1}^J \log \Psi\Big( b_{j} y(\pi^{-1}(j)) \Big).$$
\item The two solutions $\mbf b^\ast \in \mbb R^J$ and
  $\mbf u^\ast \in \mbb R^N$ satisfy the relationship
  \begin{equation}\label{u-ast-from-b-ast}
     \mbf u^\ast = \sum_{j \in Z'} \left((\tC_\tau)^{-1} \mbf b^\ast\right)_{\pi(j)} \mbf c_j,
  \end{equation}
  and
  \begin{equation*}
    b^\ast_k = u^\ast_{\pi^{-1}(k)}, \qquad k = \{1, 2, \cdots, J\}.
  \end{equation*}
\end{enumerate}
\end{corollary}

\begin{proof}
This result follows from Proposition~\ref{binary-representer-theorem} and direct computations.
\end{proof}

\begin{remark}[Variable Elimination And  Gaussian Process Regression]
There is a simple explanation for the finite dimensional representer theorem 
which underlies Proposition~\ref{binary-representer-theorem} and 
Proposition~\ref{p:binary-representer-theorem}. If we re-order the variables
in $\mbf u$ into components $\up$ in $Z'$ and $\um$ in $Z\setminus Z'$, and re-order
the components of the precision matrix $P=C_\tau^{-1}$ 
then setting the gradient of $\msf J$ to $0$ in this 
re-ordered set of variables gives equations of the form
\begin{equation*}
\left(
\begin{array}{cc}
\Ppp & \Ppm \\
\Pmp & \Pmm
\end{array}
\right) \left(
\begin{array}{c}
\up\\
\um
\end{array}
\right)=\left(
\begin{array}{c}
{\mbf g}(\up)\\
0
\end{array}
\right).
\end{equation*}
This follows from the fact that $\Phi(\mbf u)$ does not depend on $\um;$ the term 
${\mbf g}(\up)$ results from the gradient of $\Phi(\mbf u)$ with respect to $\up.$
From this re-ordering of the equations several things are apparent: (i) the bottom
row provides a linear mapping from $\up$ to $\um$ since $\Pmm$ is invertible
whenever $C_\tau$ is; (ii) using this linear mapping it is possible
to obtain a closed nonlinear equation for $\up$ only, from the top row,
and the linear part of this equation has a  Schur complement form; 
(iii) the unknown $\um$ is recovered by solving
a linear equation; (iv) the nonlinear  equation for 
$\up$ may be viewed as the equation for a critical point of a functional of $\up$ only.  
These four points are encapsulated in the previous two theorems, where they are
rendered in a form familiar from Gaussian process regression and 
representer theorems \cite{rasmussen2006gaussian}.
Ideas analogous to those described in this remark underlie all representer theorems,
but are not so transparent in the infinite-dimensional setting. 
We present the results in the abstract form of
Proposition~\ref{binary-representer-theorem} and   
Proposition~\ref{p:binary-representer-theorem} to highlight the formal analogies 
with our companion papers \cite{HHOS2,HHOS1} which study the limiting continuum optimization 
problems that arise in the $N \to \infty$ limit.
\end{remark}

The expansion \eqref{binary-minimizer-expansion} indicates that the minimizer
$\mbf{ u}^\ast \in \text {span } \{ \mbf c_j \}_{j \in Z'}$; we refer to the
$\mbf c_j$ as representers. In other words,
the minimizer $\mbf u^\ast$ belongs to a subspace of the column space of the
covariance matrix $C_\tau$. Recall that by definition $C_\tau =\tau^{2\alpha} (L + \tau^2 I)^{-\alpha}$
and so we can compute the vectors $\mbf c_j$ by solving the 
linear equations,
\begin{equation}\label{c-j-identity}
  ( L + \tau^2I)^\alpha \mbf c_j = \tau^{2\alpha}\mbf e_j, \qquad j \in Z',
\end{equation}
that cost $J$ linear solves involving an $N \times N$ matrix.
With the $\{\mbf c_j\}_{j \in Z'}$ at hand we can extract the matrix $\tC_\tau$
and solve the nonlinear system \eqref{dimension-reduced-binary-EL} for $\mbf b^\ast$
and in turn compute the solution $\mbf u^\ast$ by \eqref{u-ast-from-b-ast};
note that $F_j$ is defined in 
Proposition \ref{binary-representer-theorem}(ii).  
Then, whenever $J \ll N$ solving  the dimension reduced problem
\eqref{dimension-reduced-binary-EL} is typically much faster than 
solving the
full nonlinear system \eqref{binary-EL}. We present evidence of  this
improved efficiency in subsection~\ref{sec:num-exp-one-hot}
in the context of the one-hot method for multi-class classification.

We now proceed to exploit the geometry in the problem dictated by
the \as{nearly-disconnected} graph structure that forms the basic assumption
underlining our consistency analysis.
It is clear that the geometry of $\mbf u^\ast$ is dictated by the
geometry of the vectors $\mbf c_j$. It is then natural for us to
try to identify the geometry of the $\mbf c_j$.  By Lemma~\ref{c-j-eigen-expansion} 
we have the expansion 
  \begin{equation}
    \label{c-j-eigenvector-expansion}
    \mbf c_j =  \sum_{k=1}^N \frac{1}{\lambda_k} (\pmb \phi_k)_j \pmb \phi_k,
  \end{equation}
  where $\{ \lambda_k, \pmb \phi_k\}$ are the eigenpairs of $C_\tau^{-1}$.
Therefore, 
by analyzing the spectrum of $C_\tau$ we can identify the geometry of the vectors $\mbf c_j$
which together with the vector $\mbf b^\ast$
allow us to identify the minimizer $\mbf u^\ast$ and eventually prove
consistency of the probit minimizer.
Spectral analysis of $C_\tau$  is outlined in Appendix~\ref{appendix:perturbation-theory}, and
in the next subsection
we present the main propositions and assumptions that are used
in the remainder of the article.

\subsection{Perturbation Theory For Covariance Operators}
\label{sec:covariance-perturbation}
Consider
a disconnected graph $G_0 = \{ X, W_0\}$ consisting of $K < N$ connected  components $\tG_k$, i.e.,  the
subgraphs $\tG_k$ are connected  but there exist no edges between
pairs of components $\tG_i$, $\tG_k$ with $i \ne k.$
  Without loss of generality assume the nodes in $Z$ are ordered so that
  $Z = \{ \tZ_1, \tZ_2, \cdots, \tZ_K\}$ and the  $\tZ_k$ collect the 
  nodes in the $k$-th subgraph $\tG_k$. We refer to the $\tZ_k$ as {\it clusters}.
  Thus, the weight matrix $W_0 = (w^{(0)}_{ij})$ satisfies
  \begin{equation}\label{separate-clusters-weight-matrix}
    \left\{\begin{aligned}
      &w^{(0)}_{ij} \ge 0 \qquad \text{ if } i \neq j \text{ and } i, j \in \tZ_k \text{ for some } k,\\
      &w^{(0)}_{ij} = 0 \qquad \text{ if } i = j \text{ or }  i \in \tZ_k, j \in \tZ_\ell, \text{ for } k \neq \ell.
  \end{aligned}\right.
\end{equation}
We will show that  when $\tau$ is small the geometry of $\mbf c_j$ is dominated by
indicator functions of the clusters $\tZ_k$. First, let us collect some assumptions
on the graph $G_0$.

\begin{assumption}\label{assumptions-on-G-0}
The  graph $G_0= \{ X, W_0 \}$ satisfies the following conditions with $K < N$:
  \begin{enumerate}[(a)]
  \item The weight matrix $W_0$ satisfies \eqref{separate-clusters-weight-matrix}
    and has a block diagonal form $W_0 = \text{diag}(\tilde{W}_1, \cdots,
    \tilde{W}_K)$ where $\tW_k$ are the weight matrices 
of the subgraphs $\tG_k$. 
 
\item Let $\tL_k$ be the
  graph Laplacian matrices of the subgraphs $\tG_k$, i.e.,
  \begin{equation*}
  \tL_k:= \tD_k^{-p}(\tD_k - \tW_k) \tD_k^{-p}
\end{equation*}
with $\tD_k$ denoting the degree matrix of $\tW_k$.
  There exists a uniform constant $\theta >0$ so that for $j=1,\cdots, K$ the
 submatrices $\tL_j$ have a uniform spectral gap, i.e.,
 \begin{equation}\label{uniform-spectral-gap-sub-matrices}
     \langle \mbf x, \tL_j \mbf x \rangle \ge \theta \langle \mbf x, \mbf x \rangle,
   \end{equation}
   for all vectors $\mbf x \in \mbb R^{N_k}$ and $\mbf x \bot \tD_k^{p} \mbf 1_k$
   where $\mbf 1_k\in \mbb R^{N_k}$ are vectors of ones.
  \end{enumerate}
\end{assumption}

Note that the preceding assumption means that the clusters $\tG_k$
are pathwise connected.
Further, condition~\eqref{uniform-spectral-gap-sub-matrices} excludes the possibility of outliers, that is, nodes of zero degree. This means that the inner product as expressed in \eqref{xLx} is well defined.
In the following and throughout the remainder of the article we introduce the following notation:
We define the graph Laplacian in terms of the weight matrix $W_0$ 
and the associated degree matrix $D_0 := {\rm diag}(d_i^{(0)})$:
\begin{equation}
  \label{graph-Laplacian_x}
L_0 := D_0^{-p} ( D_0 - W_0) D_0^{-p} \in \mbb R^{N\times N}.
\end{equation}
Let $Z = \{ \tZ_1, \tZ_2, \cdots, \tZ_K\}$ and define the weighted
indicator functions 
  \begin{equation}\label{chi-k-definition}
(\pmb \chi_k)_j := \left\{
    \begin{aligned}
      & \left(d^{(0)}_j\right)^p, \qquad &&\text{ if } j \in \tZ_k,\\
      &0, \qquad &&\text{ otherwise}, 
  \end{aligned}\right.
  \qquad \text{and} \qquad \bar{\pmb \chi}_k := \frac{1}{\| \pmb \chi_k\|} \pmb \chi_k.
\end{equation}
Similarly, the weighted indicator function on  $Z$ is denoted by
\begin{equation}\label{chi-definition}
    \pmb \chi := D_0^p \mbf 1, \qquad \text{and} \qquad 
    \bar{\pmb \chi} : = \frac{1}{\| \pmb \chi \|} \pmb \chi.
\end{equation}
The next proposition, whose proof is 
given in appendix~\ref{sec:proof-of-perfectly-sep-spectral-theory},
identifies the geometry of  the  covariance matrix constructed from $L_0$. The key take away
is that, for small $\tau$, the covariance matrix is nearly block diagonal which  implies  
 negligible correlation  outside clusters. 
%
\begin{proposition}\label{perfectly-separated-c-j}
Let $G_0 = \{X, W_0\}$ satisfy Assumption~\ref{assumptions-on-G-0} and let
$L_0$ be a graph Laplacian of form \eqref{graph-Laplacian_x} on $G_0$ and
define the covariance matrix $C_{\tau,0}$ on $G_0$ for $\tau^2, \alpha >0$
\begin{equation}\label{C-tau-zero-definition}
  C_{\tau,0} := \tau^{2\alpha}(L_0 + \tau^2 I)^{-\alpha}.
\end{equation}
  Then as $\tau \downarrow 0$,
  \begin{equation*}
    \left\| \mbf c_{j,0} - \left( \bar{\pmb \chi}_k  \right)_j \bar{\pmb \chi}_k \right\|^2 
 \le  \Xi \tau^{4\alpha}  \qquad \forall j \in \tZ_k,
\end{equation*}
where $\mbf c_{j,0} = ( c_{1j}^{(0)}, \cdots, c^{(0)}_{Nj} )^T$ is the $j$-th column of $C_{\tau,0}$, and $\Xi>0$ is a uniform constant.
\end{proposition}

Thus, when $\tau^{2\alpha}$ is small the vectors $\mbf c_{j,0}$ have a similar geometry to 
the set functions $\bar{\pmb \chi}_k$. We now show that this result remains true when
the graph $G_0$ is perturbed. 

Consider a perturbation of the matrix $W_0$
by modifying some of the entries $w^{(0)}_{ij}$ and possibly making the graph connected.
More precisely, let $G_\epsilon = \{ X, W_\epsilon\}$ where
\begin{equation}\label{W-eps-perturbation-expansion}
  W_\eps = W_0 + \sum_{h=1}^\infty \eps^h W^{(h)}.
\end{equation}
We need to collect some assumptions on the perturbed matrix $W_\eps$
to restrict the type of perturbations that are allowed.

\begin{assumption}\label{assumptions-on-G-eps}
  The graph $G_\eps = \{X, W_\eps\}$  satisfies the following assumptions:
  \begin{enumerate}[(a)]
  \item The weight matrix $W_\eps$ satisfies expansion \eqref{W-eps-perturbation-expansion}
    with $W_0$ satisfying \eqref{separate-clusters-weight-matrix}.
  \item The sequence of matrix norms satisfies  $ \{\| W^{(h)} \|_2\}_{h \in {\mbb Z}} \in \ell^\infty$, 
and for each $h \in {\mbb Z}$, $ W^{(h)} = (w^{(h)}_{ij})$ is self-adjoint and satisfies
  \begin{equation}
    \label{W-k-condition}
    \left\{
    \begin{aligned}
      & w^{(h)}_{ij} \ge 0,\quad \text{if} \quad w_{ij}^{(0)} = 0 \quad
      \text{for} \quad i,j \in Z, i \neq j,
  \end{aligned}
  \right.
\end{equation} 
  \end{enumerate}
\end{assumption}
Note that $w^{(h)}_{ij}$ may be negative for indices $i,j$ such that $w^{(0)}_{ij}>0$.

Associated to the weight matrix $W_\eps$ is a graph Laplacian and covariance matrix
\begin{equation}\label{L-eps-perturbation-expansion}
  L_\eps := D_\eps^{-p} (D_\eps - W_\eps) D_\eps^{-p}, \qquad
  C_{\tau, \eps} := \tau^{2\alpha} (L_\eps + \tau^2 I )^{-\alpha},
\end{equation}
where $\tau^2, \alpha >0$ and $p\in \mbb{R}$. We then have the following result stating that 
if $\eps=o(\tau^{2})$, then the geometry of the column space of $C_{\tau, \eps}$
remains close to the set functions $\pmb \chi_k$, whilst if $\eps=\Theta(\tau^{2})$ then
prior correlation between clusters is introduced. 
The proof is given in appendix~\ref{sec:proof-of-weakly-sep-spectral-theory}.
\begin{proposition}\label{geometry-of-covaraince-functions}
  Suppose $G_0$ satisfies Assumption~\ref{assumptions-on-G-0} and
  $G_\epsilon$ satisfies Assumption~\ref{assumptions-on-G-eps}.
  For $\tau^2, \alpha >0$ define the covariance matrix
  $C_{\tau,\epsilon}$ as in \eqref{L-eps-perturbation-expansion}
  and denote the $j$-th column of
 $C_{\tau, \epsilon}$ by $\mbf c_{j,\epsilon} = (c_{1j}^{(\eps)}, \cdots, c_{Nj}^{(\eps)})^T$.
Then
  \begin{enumerate}[(a)]
  \item If $\epsilon = o(\tau^2)$, then there exists a constant $\Xi >0$ independent of
    $\epsilon$ and $\tau$ so that 
  \begin{equation*}
    \left\|  \mbf c_{j,\eps} - \left( \bar{ \pmb \chi}_k \right)_j \bar{ \pmb \chi}_k \right\|^2
    \le \Xi \left(
      \eps^{2}/\tau^{4} + \tau^{4\alpha} + \eps^2  \right), \qquad \forall j \in \tZ_k.
  \end{equation*}.

\item If $\eps/\tau^2 = \beta >0$ is constant,
  then there exist constants $\Xi, \Xi'>0$,
  independent of  $\epsilon$ and $\tau$ so that
 \begin{equation*}
  \left\|  \mbf c_{j,\eps} - 
   \left[ (1 - \tbeta) (\bar{\pmb \chi})_j \bar{\pmb \chi}  + \tbeta (\bar{\pmb \chi}_k)_j \bar{\pmb \chi}_k
   \right]\right\|^2 
   \le
   \Xi \left(\eps^{2}  + \tau^{4\alpha}  \right), \qquad \forall j \in \tZ_k,
 \end{equation*}
 and where $\tbeta = (1+\Xi' \beta)^{-\alpha}$.
 \end{enumerate}
\end{proposition}

\subsection{Consistency Of The Probit Method}\label{sec:binary-SSL-consistency}

Throughout this section
we consider a graph $G_0 = \{ X, W_0\}$ consisting of $K$ components,
along with perturbed graphs  $G_\eps = \{ X, W_\eps\}$ as
introduced in the previous subsection. As before, we use $Z' \subset Z$
to denote the set of points where labels are observed and
assume the usual ordering $Z = \tZ_1 \cup \tZ_2 \cup\cdots\cup \tZ_K$ where $\tZ_k$ denotes
the $k$-th cluster in $Z$. Recall the  
probit assumption on the labelled data, namely that  
\begin{equation}
  \label{probit-observation-model}
  y(j) = \sgn(  u^\dagger_j + \eta_j), \qquad j \in Z',
\end{equation}
where $\mbf u^\dagger = ( u_1^\dagger, \cdots, u_N^\dagger)^T$ is the vector isomorphic to the ground truth function $u^\dagger$.
The additive noises $\eta_j$ are assumed to be a rescaling of a
  sequence of i.i.d. samples from a reference density
  $\psi$. That is, for all $j \in Z'$,
  \begin{equation}
    \label{noise-scaling-iid-sequence-probit}
    \eta_j = \gamma \teta_j, \qquad \teta_j \iidsim \psi,
  \end{equation}
where $\psi$ is the PDF of a centered random variable with unit standard
deviation,  and thus
$\gamma >0$ is the standard deviation of the $\eta_j$.
Thus the $\eta_j$ are i.i.d and have distribution
\begin{equation}
  \label{iid-noise-density}
  \psi_\gamma(t) = \frac{1}{\gamma} \psi\left( \frac{t}{\gamma} \right).
\end{equation}

We recall a useful result stating that log-concave random variables have exponential tails
\cite[Thm.~4.3.7]{bogachev-malliavin}. 

\begin{lemma}\label{convex-measures-have-exp-tails} 
  Let $\psi(t)$ be a log-concave PDF on $\mbb{R}$. Then there is $\omega_c>0$
such that, for all $\omega \in [0,\omega_c)$, 
$\int_{\mbb{R}} \exp( \omega |t|) \psi(t) dt < + \infty$.  
\end{lemma}
With this lemma  we can estimate the probability of the event where
the observed labels $y(j)$ have the same value 
as ${\rm sgn}(u^\dagger_j)$, i.e., the event where the data is exact.
\begin{lemma}\label{data-is-exact-with-high-probability}
Let $\psi(t)$ be a log-concave PDF on $\mbb{R}$. Then there exist constants $\omega_1, \omega_2 > 0$ depending only on $\psi$, so that 
$$
\mbb{P} \left( y(j) = {\rm{sgn}}(u^\dagger_j), \text{ for all }j \in Z' \right)
\ge  \prod_{j\in Z'}  \left[ 1-  \omega_2  \exp \left( - \frac{\omega_1}{\gamma} |u^\dagger_j|   \right) \right]. 
$$
That is, when $\gamma>0$ is small the data $y$ is exact with high probability.
\end{lemma}
\begin{proof}
  By Lemma~\ref{convex-measures-have-exp-tails} there exists a sufficiently small $\omega_1 >0$
and constant $\omega_2>0$
so that for $\gamma >0$,
  \begin{equation*}
\omega_2 =  \int_{\mbb{R}}  \exp\left( \frac{\omega_1}{\gamma} |t|   \right)  \psi_\gamma(t) dt   < + \infty.
\end{equation*} 
Let $\eta_j \sim \psi_\gamma$ 
then by Markov's inequality for $\theta > 0$ 
\begin{equation*}
  \mbb{P}( \eta_j \ge \theta) \le \omega_2  \exp\left( -  \frac{\omega_1}{\gamma} \theta \right).
\end{equation*}
But $y(j) \neq {\sgn}(u^\dagger(j))$ whenever
\begin{align*}
  \eta_j   < -| u^\dagger_j|  \quad &{\rm if } \quad u^\dagger_j \ge 0, \\ 
  \eta_j   \ge | u^\dagger_j|  \quad  &{\rm if } \quad u^\dagger_j < 0.
\end{align*}
The result now follows from the symmetry of the $\psi_\gamma$ and independence of the $\eta_j$.
\end{proof}

  \begin{lemma}\label{as-exactness-of-data-probit}
    Suppose \eqref{probit-observation-model} and \eqref{noise-scaling-iid-sequence-probit} hold,
    $\psi$  is log-concave and  $|u^\dagger_j| > \theta >0$ for all $j \in Z'$. Then
    for any sequence $\gamma \downarrow 0$, 
    $y(j) \asto \sgn(u^\dagger_j)$ a.s.  with respect to $ \prod_{j \in Z'}\psi(t_j)$ the law of the
    i.i.d. sequence  $\{ \teta_j\}_{j\in z'}$.
\end{lemma}

\begin{proof}
   Since $\psi$ is log-concave it has exponential
  tails by Lemma~\ref{convex-measures-have-exp-tails} and
  $|\teta_j| < \infty$  a.s.
  \footnote{In fact the proof reveals
that all we need is that the $\teta_j$ are a.s. finite, for which
log-concavity suffices.}
  Recall that value of $\eta_j = \gamma \teta_j$. Then for any fixed
    $\teta_j \in (-\infty, \infty)$,  if $\gamma < \theta/ |\teta_j|$ then
  $y(j) =  \sgn(u^\dagger_j)$. Since $\teta_j$ are a.s. finite the result follows. 
\end{proof}

Now
consider a probit likelihood potential of the form
\begin{equation}
  \label{binary-probit-likelihood-F-form}
  \Phi_\gamma(\mbf u; y) :=  \sum_{j \in Z'} - \log \Psi_\gamma(u_j y(j)), 
\end{equation}
where
\begin{equation}
  \label{binary-probit-Fj-definition}
  \Psi_\gamma(s) = \int_{-\infty}^s \psi_\gamma(t) dt, \qquad s \in \mbb R.
\end{equation}
For $\epsilon, \tau^2, \gamma >0$ we study the consistency of
minimizers  of the
functionals 
\begin{equation}
  \label{perturbed-binary-probit-functional}
   {\msf J}_{\tau, \epsilon, \gamma}(\mbf u)  := \frac{1}{2} \langle \mbf u, C_{\tau, \epsilon}^{-1} \mbf u \rangle
  + \Phi_\gamma(\mbf u; y),
  \qquad \mbf u \in \mbb R^N.
\end{equation}
This functional is of the same form as \eqref{probit-functional}, and so the results in section~\ref{sec:properties-binary-proibit-minimizer} apply.


\subsubsection{Probit Consistency With A Single Observed Label}
We start with the simple case of a single observed label.
Without loss of generality assume $Z' = \{ 1\}$, that is the observed label
is the first point in the first cluster $\tZ_1$ and $u^\dagger_1 >0$.

\begin{proposition}\label{single-observation-consistency}
  Consider the single observation setting above and suppose Assumptions~
  \ref{assumptions-on-G-0} and \ref{assumptions-on-G-eps}
  are satisfied by $G_0$ and $G_\eps$. Let $\bu^\ast$ denote the
  minimizer of ${\msf J}_{\tau, \eps, \gamma}$ and let $\gamma > 0$.
  \begin{enumerate}[(a)]
    \item If $\eps = o(\tau^2)$ as $\tau \to 0$ then $\exists \tau_0 >0$ 
so that  $\forall (\tau,\gamma) \in (0,\tau_0) \times \mbb (0,\infty)$ and 
    $\forall j \in \tZ_1$
    \begin{equation}\label{singe-observations-prob-lower-bound}
      \mbb P \Big( \sgn\left(  u^\ast_j \right) = +1 \Big) \ge 1- \omega_2 \exp \left( -\frac{\omega_1}{\gamma}
        |u^\dagger_1| \right),
    \end{equation}
    where $\omega_1, \omega_2 >0$ are uniform constants
      depending only on $\psi$.
 \item If $\eps = \Theta(\tau^2)$ then the above statement holds for all $j \in Z$.  
  \end{enumerate}
\end{proposition}

\begin{proof} (a) By Corollary~\ref{probit-dimension-reduction} we
  have that $ u^\ast_1 = b$ where $b$ solves
  \begin{equation}\label{single-observation-dimension-reduction}
  b = (C_{\tau, \eps})_{11} F_{1,\gamma}( b).
\end{equation}
where we recall 
$F_{1,\gamma}(s) = y(1) \psi_\gamma(sy(1))/\Psi_\gamma(s y(1))$.
Furthermore, 
  by
  Proposition~\ref{geometry-of-covaraince-functions}(a) and
  equivalence of $\ell_\infty$ and $\ell_2$ norms we infer that there exists $\tau_0>0$
  so that $\forall \tau \in (0, \tau_0)$ we have
  \begin{equation}\label{single-observation-c-1-approximation}
    (\bc_{j,\eps})_1 = \left\{
      \begin{split}
        &  \left(\bar{\pmb \chi}_1 \right)_j \left(\bar{\pmb \chi}_1 \right)_1  + \mcl{O} \left( \eps/\tau^{2} +  \tau^{2\alpha} + \eps\right),
        \qquad &&j \in \tZ_1,\\
        & \mcl{O} \left( \eps/\tau^{2} + \tau^{2\alpha} + \eps\right), \qquad  &&j \not\in \tZ_1.
      \end{split}\right.
  \end{equation}
  Thus,  we can rewrite
  \eqref{single-observation-dimension-reduction} up to leading order
  in the form
  \begin{equation*}
      b = |(\tchi_1)_1|^2  F_{1,\gamma}( b).
    \end{equation*}
    Now consider the event where $y(1) = +1$, i.e., the measurement is
    exact. Then $b>0$ since $F_{1,\gamma} >0$ and 
    $(\tchi_1)_1>0$. 
    Finally,
    by Corollary~\ref{probit-dimension-reduction}(iii) we can write
    \begin{equation*}
      \bu^\ast = F_{1,\gamma}(b) \bc_{1,\eps} = F_{1,\gamma}(b)
     (\tchi_1)_1 \bar{\pmb \chi}_1
      + \mcl O \left(\eps/\tau^{2} +  \tau^{2\alpha} + \eps \right). 
    \end{equation*}
    Thus, when $ \eps/ \tau^2, \tau$ and $\eps$ are sufficiently small and the data is exact, $\bu^\ast$ is positive on $\tZ_1$. The claim now
    follows by Lemma~\ref{data-is-exact-with-high-probability}. The statement in (b) follows
    by an identical argument except that in this case
     \begin{equation*}
       \bc_{1,\eps} =  
   \left[ (1 - \tbeta) (\bar{\pmb \chi})_1 \bar{\pmb \chi} 
   + \tbeta (\bar{\pmb \chi}_1)_1 \bar{\pmb \chi}_1
   \right]
       + \mcl{O}\left(\tau^{2\alpha} + \eps \right)
     \end{equation*}
     with $\tbeta \in (0, 1)$.
     Thus in the event that $y(1) = +1$ the minimizer $\bu^\ast$ is positive on all of $Z$.
  {}
\end{proof}

The following corollary 
shows that the relationship between $\tau$ (which is user-specified)
and $\epsilon$ (which is a property of the unlabelled data) is crucial in determining
how the probit algorithm assigns labels to the entire data set in the small
noise  limit; case (a) is neutral about $Z \setminus \tZ_1$ whilst case (b) leads to
$Z \setminus \tZ_1$ being labelled the same as $\tZ_1.$ 
The reason for the difference is that the limit process in (a) corresponds, asymptotically,
to a setting where a priori different clusters have no correlation whilst under
the limit process (b) there is positive correlation; thus, under (b), the one given label
is propagated to the entire set of nodes. When more clusters are labelled then similar effects are
present under the limit process (b), but are harder to express analytically because they
set-up a competition between potentially conflicting prior information and observed label information.
This is investigated numerically in section \ref{sec:NE}.

\begin{corollary}\label{probit-asymptotic-consistency-single-observation}
Consider the single observation setting as above.  Suppose Proposition~\ref{single-observation-consistency} is satisfied and
    $|u^\dagger_j|  >0$ for all $j \in Z'$. Then
   the following holds with a.s. convergence in the sense of Lemma~\ref{as-exactness-of-data-probit}:
  \begin{enumerate}[(a)]
  \item For any sequence  $\gamma, \tau, \eps \downarrow 0$ along 
    which $\eps = o(\tau^{2})$ it holds that
\begin{equation*}
       \sgn\left( u^\ast_j \right) \asto \sgn(u^\dagger_1),
      \qquad \forall j \in \tZ_1.
    \end{equation*}
  \item For any sequence  $\gamma, \tau, \eps \downarrow 0$ along
    which $\eps = \Theta(\tau^2)$ the above statement holds for all $j\in Z$.
\end{enumerate}
\end{corollary}

\begin{proof}
  In the the proof of Proposition~\ref{single-observation-consistency}(a)
  we showed that $\sgn(u^\ast_j ) = +1$ on $\tZ_1$ so long as the data $y(1) = +1$ independent of
  $\gamma >0$. In light of this, (a)  follows directly from Lemma~\ref{as-exactness-of-data-probit}
  implying that the data $y(j) \to \sgn(u^\dagger_j)$  a.s. as $\gamma \downarrow 0$. Statement (b) follows
  in the same way but using the proof of Proposition~\ref{single-observation-consistency}(b).
\end{proof}

\subsubsection{Probit Consistency with Multiple Observed Labels}
Let us now consider the setting where multiple labels are observed, i.e. $|Z'| = J \ge 2$.
 We need to make  an additional assumption on the ground truth function $\bu^\dagger$.

\begin{assumption}\label{assumptions-on-u-dagger-binary}
  Let $\tZ_k$ be a cluster within which a label has been observed, i.e., $\tZ_k \cap Z' \neq \emptyset$.
  Then $\sgn(\bu^\dagger)$ does not change within $\tZ_k$. 
\end{assumption}

It is helpful in the following to define $Z''$ to be the index set of
nodes within all clusters $\tZ_k$ where a 
label has been observed, i.e., 

\begin{equation}\label{Z-pp}
Z'' := \cup_{\{k : \tZ_k \cap Z' \neq \emptyset\}} \tZ_k.
\end{equation}

\begin{theorem}\label{multiple-observation-consistency}
  Consider the multiple observation setting above
   and suppose Assumptions~
   \ref{assumptions-on-G-0}, \ref{assumptions-on-G-eps} and
   \ref{assumptions-on-u-dagger-binary}
   are satisfied by $G_0$, $G_\eps$ and $\bu^\dagger$. Let
   $\bu^\ast$ be the minimizer of $\msf J_{\tau, \eps, \gamma}$.
   If $\eps = o(\tau^2)$ as $\tau \to 0$ then $\exists \tau_0 >0$
so that  $\forall (\tau,\gamma) \in (0,\tau_0) \times (0, \infty)$ and
    $\forall j \in Z''$
    \begin{equation*}
      \mbb P \Big( \sgn\left( u^\ast_j\right) = \sgn(u^\dagger_j ) \Big) \ge
      \prod_{i \in Z'} \left[ 1- \omega_2 \exp \left( -\frac{\omega_1}{\gamma}
          |u^\dagger_i| \right) \right], \qquad \forall j \in Z'',
    \end{equation*}
    where $\omega_1, \omega_2 >0$ are uniform constants depending only
      on $\psi$. 
\end{theorem}

\begin{proof}
  Our proof follows a similar approach to the single observation case.
  The main difference is that now the dimension
reduced system \eqref{dimension-reduced-binary-EL}   takes the form
\begin{equation}\label{b-ast-eps-system}
  (\tC_{\tau,\eps})^{-1} \mbf b^\ast = F'_\gamma(\mbf b^\ast).
\end{equation}
 Where $\tC_{\tau, \eps}$
is now the submatrix of $C_{\tau, \eps}$ with the rows and columns of the
indices $Z \setminus Z'$ removed.
It follows from Proposition~\ref{geometry-of-covaraince-functions} that
for small $\tau$ the matrix $\tC_{\tau, \eps} = (c'_{ij})$
approaches a block diagonal matrix and so 
\begin{equation}\label{consistency-proof-tilde-C-block-diagonal-form}
  c'_{ij} = \left\{
    \begin{aligned}
    &  \left( \bar{\pmb \chi}_k \right)_{\pi^{-1}(j)}
    \left( \bar{\pmb \chi}_k\right)_{\pi^{-1}(i)} + \mcl{O}\left(\eps/\tau^{2} + \tau^{2\alpha} + \eps \right),
    &&\text{if} \quad \pi^{-1}(i), \pi^{-1}(j) \in \tZ_k,\\
    &\mcl{O}\left(\eps/\tau^{2} + \tau^{2\alpha} + \eps \right), && \text{if} \quad \pi^{-1}(i) \in \tZ_k, \pi^{-1}(j) \in \tZ_\ell, k \neq \ell.
    \end{aligned}
    \right.
\end{equation}
Without loss of generality assume that observations are made in the clusters
$\tZ_1, \cdots, \tZ_{K'}$. Note that $K' \le K$ since we do not need to assume
observations are made in every cluster.
Let $\tZ'_1, \cdots, \tZ'_{K'}$ denote the indices of the labelled nodes in the corresponding
clusters and define $J_k : =|\tZ_k'|$.
Then \eqref{b-ast-eps-system} approximately decouples between the clusters
and up to leading order we can write
\begin{equation*}
    \left( \bar{\pmb \chi}_k \right)_j^{-1} b^\ast_{\pi(j)} =  \sum_{i \in \tZ'_k} 
    \left(\bar{\pmb \chi}_k\right)_i {F}_{i,\gamma}( b^\ast_{\pi(i)}), \qquad \text{for } j \in  \tZ'_k \text{ and } k \in \{1, \cdots, K'\}.
\end{equation*}
Observe that   the right hand side is
independent of $j$ and so, writing $b^\ast_\ell = u^\ast_{\pi^{-1}(\ell)}$ for $\ell\in \{1, 2, \cdots, J\}$ as in Corollary~\ref{probit-dimension-reduction}(iii), it follows that $D_0^{-p}\mbf u^\ast$ is a constant vector on the index sets $\tZ_k'$. 
Thus, we can  further simplify this equation to get
\begin{equation*}
      \left( \bar{\pmb \chi}_k \right)_j^{-1} u^\ast_{j} =  \sum_{i \in \tZ'_k} 
  \left(\bar{\pmb \chi}_k\right)_i F_{i,\gamma}(u^\ast_j),
  \qquad \text{for }j \in \tZ'_k \text{ and } k \in \{1, \cdots, K'\},
\end{equation*}
which we only need to solve  once on every cluster. Finally, observe that  $\sgn(\bu^\dagger)$
does not change on $\tZ_k$ following Assumption~\ref{assumptions-on-u-dagger-binary}
and so in the event that $y$ is exact we have
\begin{equation*}
  u^\ast_{j} = 
   \left(\bar{\pmb \chi}_k \right)_j
  \left( \sum_{i \in \tZ'_k} \left(\tchi_k\right)_i\right)
  \frac{ \sgn(u^\dagger_j) \psi_\gamma\left( \sgn(u^\dagger_j) u^\ast_{j} \right) }
  { \Psi_\gamma \left(  \sgn(u^\dagger_j) u^\ast_{j} \right)}  \,,
  \qquad \text{for }j \in \tZ'_k \text{ and } k \in \{1, \cdots, K'\}.
\end{equation*}
To this end, $\sgn(\mbf u^\ast)$ agrees with $\sgn(\bu^\dagger)$ on the observation nodes.
Once again using Corollary~\ref{probit-dimension-reduction}(iii)
we see that
\begin{equation*}
  \bu^\ast = \sum_{j \in \tZ'}
    \frac{ \sgn(u^\dagger_j) \psi_\gamma\left( \sgn(u^\dagger_j) b^\ast_{\pi(j)} \right) }
    { \Psi_\gamma \left(  \sgn(u^\dagger_j) b^\ast_{\pi(j)} \right)} \bc_{j,\eps}
    = \sum_{k=1}^{K'} \ta_k 
    \bar{\pmb \chi}_k,
  \end{equation*}
  where the last identity is once more up to leading order following
  Proposition~\ref{geometry-of-covaraince-functions} and for  coefficients
  $$
  \ta_k := \sum_{j\in Z'}  \frac{ \sgn(u^\dagger_j) \psi_\gamma\left( \sgn(u^\dagger_j) b^\ast_{\pi(j)} \right) }
    { \Psi_\gamma \left(  \sgn(u^\dagger_j) b^\ast_{\pi(j)} \right)} 
    \left(\bar{\pmb \chi}_k\right)_j,
  $$
  such that $\sgn (\ta_k)$ agrees with $\sgn( \bu^\dagger)$ on $\tZ_k$
  for $k =1, \cdots, K'$. Finally, the claim  follows by applying
  Lemma~\ref{data-is-exact-with-high-probability} to compute the probability of the
  event where the data is exact. 
\end{proof}

Similarly to Corollary~\ref{probit-asymptotic-consistency-single-observation}
the next corollary follows from the proof of Theorem~\ref{multiple-observation-consistency}
and Lemma~\ref{as-exactness-of-data-probit}.

\begin{corollary}\label{probit-asymptotic-consistency-multiple-observations}
  Suppose Theorem~\ref{multiple-observation-consistency} is satisfied and 
  $|u^\dagger_j|  > 0$ for $j \in Z'$.
    Then for any sequence $\gamma, \tau, \eps \downarrow 0$
    along which $\eps = o(\tau^2)$ it holds that, 
    \begin{equation*}
       \sgn\left(u^\ast_j \right) \asto \sgn( u^\dagger_j) \qquad \forall j \in Z'',
\end{equation*}
 with a.s. convergence  in the sense of Lemma~\ref{as-exactness-of-data-probit}.
\end{corollary}

\section{Multi-Class Classification: The One-Hot Method}\label{sec:multi-class-classification}
In the multi-class setting $M>2$  we can no longer use the $\sgn( \cdot )$ function to reduce the 
dimension of the latent variable to $\mbb R^{M-1}$ as we did for binary $M=2$ classification in 
section~\ref{sec:binary-classification}.  Instead we use one-hot encoding and work directly
with latent variables taking value in $\mbb R^M$. We set up the one-hot methodology 
in subsection~\ref{sec:31} assuming that $M \ge 2$, though it would be 
un-necessary to use it for $M=2$ and Gaussian label noise when it reduces to probit. 
In subsection \ref{ssec:32} we study the form of the one-hot likelihood that appears in
the optimization problem, resulting from the labelled data.
In subsection~\ref{sec:one-hot-quadratic-penalty} 
we introduce a quadratic regularization term  for the one-hot method
that uses the covariance matrix $C_\tau$ of \eqref{covariance-matrix} 
and is analogous to the quadratic penalty used in the probit method.
In subsection~\ref{sec:properties-multiclass-onehot-minimizer} 
we
study the one-hot minimization problem, formulating a discrete representer
theorem for the one-hot method. 
Subsection~\ref{sec:one-hot-SSL-consistency} concludes the analysis of the one-hot method,
studying consistency in some detail by putting together the results of
previous subsection with the spectral theory introduced in section~\ref{sec:covariance-perturbation}.

\subsection{Set-Up}\label{sec:31}
We now turn our attention to the multi-class classification problem, i.e.,
 where the label function $l: Z \mapsto \{1, \cdots, M\}$ assigns one of
$M\ge 1$ classes to each point in $X$. In this case the
sign function from section~\ref{sec:binary-classification} is no longer 
an appropriate classifying function and we need a different method.
We shall utilize the 
 {\it one-hot mapping} 
\begin{equation}
  \label{one-hot-S}
  S(\mbf v) = \argmax_{k} v_k, \qquad
  \mbf v = (v_1, \cdots, v_M) \in \mbb R^M.
\end{equation}
In the case of two maximal elements $v_{k_1}=v_{k_2}$, we take the smallest index. As with probit, the case of a near-tie is prone to misclassification by perturbation. For the purpose of consistency analysis, we later make assumptions that ensure a tie for the maximal element cannot occur.

The latent variable $u: Z \mapsto \mbb R^M$ 
is  isomorphic to a matrix $U = (u_{mj}) \in \mbb R^{M \times N}$.
We use $\bu_{j}$ to denote the $j$-th column of $U$ as a  vector in  $\mbb R^M$.
With this notation at hand  we consider the following  model for observed labels $y$:
\begin{equation}
  \label{one-hot-label-model}
  y(j) = S(\bu_{j} + {\pmb \eta}_j), \qquad j \in Z',
\end{equation}
where
\begin{equation*}
  {\pmb \eta}_j = ( \eta_{1j}, \cdots, \eta_{Mj})^T \in \mbb R^M, \quad \text{and} \quad
  \eta_{mj} \iidsim \psi.
\end{equation*}
Here $\psi$ is a  probability density function on $\mbb R$ as before.

\begin{remark}
  Note that the assumption that $\eta_{mj}$ are i.i.d. is not needed in general and
  one can consider correlations in the observation noise both between different classes
  and also amongst different points in the dataset. However, for simplicity  we
  only consider i.i.d. noise and leave the correlated noise setting for future study.
\end{remark}

\subsection{The One-Hot Likelihood}\label{ssec:32}
We begin by identifying the likelihood potential $\Phi$ for the model  \eqref{one-hot-label-model}.
For $j \in Z'$ and $m, \ell \in \{1, \cdots, M\}$ we have 
\begin{equation*}
  \begin{split}
    \mbb P[ y(j) = m | U] & = \mbb P[ u_{mj} + \eta_{mj} \ge u_{\ell j} +
    \eta_{\ell j}, \quad
    \forall \ell \in \{1, \cdots, M\}] \\
    & = \mbb P[ \eta_{\ell j} \le \eta_{mj} +  u_{mj} - u_{\ell j}, \quad \forall \ell
    \in \{1, \cdots, M\}]\\
    & = \mbb E \Big[ \mbb P[ \eta_{\ell j} \le \eta_{mj} +  u_{mj} - u_{\ell j}, \quad \forall \ell
    \in \{1, \cdots, M\}] \Big| \eta_{mj}\Big] \\
    & = \mbb E \Big[ \prod_{\ell \neq m} \mbb P[ \eta_{\ell j} \le \eta_{mj} +  u_{mj} - u_{\ell j}]
    \Big| \eta_{mj} \Big]\\
    & = \int_{\mbb R} \psi(t) \prod_{\ell \neq m} \Psi(t + u_{mj} - u_{\ell j})  dt
    =: \tPsi(\bu_j;m).
\end{split}
\end{equation*}
where $\Psi$ is the CDF of $\psi$ as in the binary case.
To this end, we define the likelihood potential $\Phi(U; y)$ as

\begin{equation}\label{one-hot-likelihood}
  \Phi(U; y)
  = - \sum_{j\in Z'} \log \tPsi(  \bu_j; y(j))
  = - \sum_{j\in Z'} \log \left( \int_{\mbb R} \psi(t) \prod_{\ell \neq y(j)} \Psi( t + u_{y(j) j} - u_{\ell j})
    dt
    \right),
\end{equation}
which is in a similar form to \eqref{binary-probit-likelihood}.

\subsection{Quadratic Regularization Via Graph Laplacians (Multi-Class Case)}
\label{sec:one-hot-quadratic-penalty}
Recall, the  matrix $C_\tau$ defined in \eqref{covariance-matrix} based on the graph
Laplacian $L$.  In a similar way to the probit method we  define a quadratic regularization term
for
matrices $U\in \mbb R^{M \times N}$ of the form
\begin{equation}\label{multi-class-quad-term}
  \langle C_\tau^{-1},  U^T U \rangle_F
  = \sum_{j,\ell =1}^N (C_{\tau}^{-1})_{j\ell} (U^T U)_{j\ell} 
  = \sum_{m=1}^M \sum_{j,\ell=1}^N (C_{\tau}^{-1})_{j\ell} u_{m \ell} u_{mj},
 \end{equation}
where $\langle \cdot, \cdot \rangle_F$ is the Frobenius inner product.
In the following we will us this quadratic term to regularize Problem~\ref{probit-SSL-general-form}
in the multi-class setting.

\begin{remark}
  If we think of $C_\tau^{-1}$ as a smoothing operator then  
 the above choice for the regularization term promotes smoothness of the rows of $U$
 while the columns of $U$ can be discontinuous. This means that each component of the
 function $u: Z \mapsto \mbb R^M$ isomorphic to $U$ is smooth amongst the vertices of 
 $G$ while the components themselves are allowed to be discontinuous at each node. 
\end{remark}

\subsection{Properties Of The One-Hot Minimizer}
\label{sec:properties-multiclass-onehot-minimizer}
Putting together the one-hot likelihood in \eqref{one-hot-likelihood} and
the quadratic regularization term \eqref{multi-class-quad-term} we
define  the {\it one-hot functional} 
\begin{equation}
  \label{one-hot-functional}
  \mcl J ( U) : = \frac{1}{2} \langle C_\tau^{-1}, U^TU \rangle_F
  + \Phi(U; y), \qquad U \in \mbb R^{M \times N}.
\end{equation}
We will see shortly that the one-hot functional has very similar properties
to the probit
functional $\msf J$ in binary classification. In particular, the regularization
term \eqref{multi-class-quad-term} is strictly convex and provides stability and geometric information
via the operator $C_\tau$ and
the one-hot likelihood $\Phi$ is also convex  and makes sure the minimizer of $\mcl J$
is a good predictor of observed labels. We start by showing the convexity of the
likelihood potential. 

\begin{proposition}[Convexity of the one-hot likelihood]\label{one-hot-likelihood-is-convex}
  Let $\psi$ be a log-concave PDF on $\mbb R$.
  Then the function
  $$\tPsi(\bv; m) = \int_{\mbb R} \psi(t) \prod_{\ell \neq m} \Psi( t + v_{m} - v_{\ell})  dt, \qquad
  \bv \in \mbb R^M,
  $$
  is log-concave on $\mbb R^M$ for all $m \in \{1, \cdots, M\}$, and where $\Psi$ is the CDF of $ \psi$.
\end{proposition}

\begin{proof}
 By \cite[Thm.~1]{bagnoli2005log} the functions $\Psi$ are
  log-concave whenever $\psi$ is log-concave. Furthermore, since
  log-concavity is preserved under affine transformations 
  and finite products  \cite[Sec~3.1]{saumard2014log}
  we conclude that
  $f(t, \bv) = \psi(t) \prod_{\ell \neq m} \Psi(t + v_m - v_\ell)$ is log-concave
  on $\mbb R^{M+1}$. The result now follows from the fact that
  the marginals of a log-concave function are also log-concave
  \cite[Thm.~3]{Prekopa80logarithmicconcave} and 
 $\tPsi(\bv; m)$ is precisely the marginal of $f(t, \bv)$
  over the $t$ variable.
\end{proof}

 Putting this result together with the
fact that the quadratic term in \eqref{one-hot-functional} is strictly convex
whenever $C_\tau^{-1}$ is strictly positive definite 
(which is true when $\tau^2 >0$)
gives the following result.

\begin{proposition}\label{one-hot-functional-is-strictly-convex}
  Let $\psi$ be a continuous and log-concave PDF on $\mbb R$ and
  let $C_\tau^{-1}$ be a strictly positive-definite matrix on $\mbb R^N$. Then
  the functional $\mcl J$ defined in \eqref{one-hot-functional} with
  $\Phi$ given by  \eqref{one-hot-likelihood} is
  strictly convex. 
\end{proposition}

We are now set  to prove an analog of Proposition~\ref{binary-representer-theorem}
for the one-hot functional.

\begin{proposition}[Representer theorem for one-hot functional]\label{one-hot-representer-theorem}
  Let $G = \{ X, W\}$ be a weighted graph and let $\psi$ be a
  log-concave PDF. Suppose $\Phi$ is given by
   \eqref{one-hot-likelihood} and the matrix $C_\tau$
  is given by \eqref{covariance-matrix} with parameters $\tau^2, \alpha >0$. Then,
  \begin{enumerate}[(i)]
  \item The one-hot functional  $\mcl J$ has a unique minimizer
    $U^\ast \in  \mbb R^{M \times N}$.
  \item The minimizer $U^\ast$ satisfies the EL equations
    \begin{equation}
      \label{one-hot-EL}
        C_\tau^{-1} U^{\ast T}  = \sum_{j \in Z'} \be_j (\mbf f_j(\mbf u^\ast_j))^T,
    \end{equation}
    where $\be_j$ is the $j$-th standard coordinate vector in $\mbb R^N$,
    the vector $\bu^\ast_j$ denotes the $j$-th column of $U^\ast$ and the
    functions $\mbf f_j : \mbb R^M \mapsto \mbb R^M$ are defined as
    \begin{equation}\label{bf-definition}
     \mbf f_j(\bv) = (f_{1j}(\bv), \cdots, f_{Mj}(\bv))^T, \qquad
      f_{mj} (\bv) = \frac{1}{\tPsi(\bv; y(j))} \frac{\partial \tPsi
      (\bv; y(j))}{
        \partial v_m},
    \end{equation}
    for vectors $\bv = (v_1, \cdots, v_M)^T$
    and 
    \begin{equation*}
      \frac{\partial \tPsi
      ( \bv; m)}{\partial v_{i}} = 
      \left\{
        \begin{split}
          -& \int_{\mbb R} \psi(t) \psi(t+ v_{m} - v_{i}) \prod_{\ell \neq i, m}
          \Psi(t+ v_{m} - v_{\ell})
           dt \quad &&\text{if}
          \quad i \neq m, \\
          \sum_{k \neq m} & \int_{\mbb R} \psi(t)
          \psi(t+v_{m} - v_{k}) \prod_{\ell \neq k, m} \Psi(t + v_{m} - v_{\ell})  dt
         \quad &&\text{if}
          \quad i = m.
        \end{split}
        \right.
      \end{equation*}
  \item The minimizer $U^\ast$ can be represented using the expansion
    \begin{equation*}
      U^\ast = \sum_{j \in Z'}  \tba_j \bc_j^T,
    \end{equation*}
   where $\tba_j \in \mbb R^{M}$ and
    $\bc_j = C_\tau \be_j \in \mbb R^N$.

  \item The matrix $U^\ast$ solves \eqref{one-hot-EL} if and only if the vectors
    $\tba_j = (\ta_{1j}, \cdots, \ta_{Mj})^T \in \mbb R^M$ solve the nonlinear
    system of equations
    \begin{equation}
      \label{dimension-reduced-system-one-hot}
      \breve \ba_j = \mbf f_j\left( \sum_{k \in Z'} c_{jk} \breve \ba_k\right), \qquad \forall j \in Z',
    \end{equation}
    where $c_{ij}$ denote the entries of $C_\tau$.
  \end{enumerate}
\end{proposition}

\begin{proof}
  The method of proof is very similar to Proposition~\ref{binary-representer-theorem}.
  (i) Follows directly from Proposition~\ref{one-hot-functional-is-strictly-convex}.
  (ii) Observe that $\mcl J(U)$ is continuously differentiable, and so \eqref{one-hot-EL} follows by directly computing the
  first order optimality conditions $\nabla \mcl J (U^\ast) = 0$.
  Proof of (iii) and (iv) is very similar to Proposition~\ref{binary-representer-theorem}(iii) and (iv) and is essentially the result of solving the EL
  equations \eqref{one-hot-EL} directly. 
\end{proof}

\begin{proposition}[One-hot dimension reduction]\label{one-hot-dimension-reduction}
  Suppose the conditions of Proposition~\ref{one-hot-representer-theorem} hold. Then 
  \begin{enumerate}[(i)]
  \item The problem of finding the matrix $U^\ast \in \mbb R^{M \times N}$ the minimizer
    of the one-hot functional $\mcl J$, is equivalent to the problem of finding the
    Matrix $B^\ast \in \mbb R^{M \times J}$ that solves
    \begin{equation}
      \label{one-hot-dimension-reduced-EL}
       (\tC_\tau)^{-1} B^{\ast T} = \mcl F'( B^\ast), 
    \end{equation}
    where $\tC_\tau$ is as in \eqref{sub-matrix-C-tilde} and,
    for matrices $B \in \mbb R^{M \times J}$ the map $\mcl F': \mbb R^{M \times J}
    \mapsto \mbb R^{J \times M}$ is defined by
    \begin{equation*}
      (\mcl F'( B))_{im} := f_{m \pi^{-1}(i)}( \mbf b_i) \quad
      \text{for} \quad (i,m) \in \{1, \cdots, J\} \times
  \{1, \cdots, M\},
\end{equation*}
where the reordering map
$\pi$ is as in \eqref{reordering-index-map} and $\mbf b_i$ denotes the
$i$-th column of  $B$.

  \item Moreover, the matrix $B^\ast$ solves the optimization problem
    \begin{equation*}
      B^\ast = \argmin_{B \in \mbb R^{J \times M}} {\mcl J'}( B ),
    \end{equation*}
    where
    \begin{equation*}
      {\mcl J'}(B) := \frac{1}{2} \langle (\tC_\tau)^{-1}, B^T B \rangle_F
      - \sum_{i=1}^J  \log \tPsi \Big( \mbf b_i; y \left(\pi^{-1}(i) \right) \Big)
    \end{equation*}

  \item  The matrices $B^\ast$ and $U^\ast$ satisfy the relationship
    \begin{equation*}
      U^{\ast} = \sum_{j \in Z'}  \tbb^{\ast}_j \bc_j^T ,
    \end{equation*}
     where $\tbb^\ast_j$ denotes the $\pi(j)$-th column of
      $ B^{\ast}(\tC_{\tau})^{-T}$
    and
    \begin{equation*}
      \mbf b^\ast_k = \bu^\ast_{\pi^{-1}(k)}, \qquad k =\{1,\cdots, J\}.
    \end{equation*}
  \end{enumerate}
\end{proposition}

\begin{proof}
  (i)
  Let $A = (a_{mi}) \in \mbb R^{M \times J}$ be the matrix with
  entries $a_{mi} = \tba_{m \pi^{-1}(i)}$. That is, the columns of $A$ are the $\tba_j$
  vectors. 
Then we can rewrite \eqref{dimension-reduced-system-one-hot} as
\begin{equation}\label{almost-dimension-reduced-system-one-hot}
  A = -\left(\mcl F'( \tC_{\tau} A^T)\right)^T\,,  
\end{equation}
Let
\begin{equation}\label{B-ast-and-A-identity}
B^{\ast T} = \tC_\tau A^T \in \mbb R^{J \times M}
\end{equation}
then we can rewrite \eqref{almost-dimension-reduced-system-one-hot} as
\begin{equation*}
  (\tC_{\tau})^{-1}  B^{\ast T} = \mcl F'( B^\ast).
\end{equation*}
(ii) Denote by $b^\ast_{mi}$ the entries of $B^\ast$.
Then we can directly verify that
\begin{equation*}
  (\mcl F'(B^\ast))_{im} = -  \frac{\partial}{\partial b^\ast_{mi}}\sum_{i=1}^J
  \log \tPsi\Big( \mbf b^\ast_i; y(\pi^{-1}(i)) \Big),
\end{equation*}
from which we  infer that the matrix $B^\ast$ indeed solves the following optimization problem
\begin{equation*}
  B^\ast = \argmin_{B \in \mbb R^{J \times M}} \frac{1}{2} \langle (\tC_\tau)^{-1},
  B^T B \rangle_F - \sum_{i =1}^J \log \tPsi \Big({\mbf b}_{i}; y(\pi^{-1}(j)) \Big).
\end{equation*}
(iii) Following  \eqref{B-ast-and-A-identity}  $A = B^\ast (\tC_{\tau})^{-T}$.
Let $\ba_i$ denote the columns of $A$.
Then by Proposition~\ref{one-hot-representer-theorem}(iii),
\begin{equation*}
  U^\ast = \sum_{j \in Z'} \tba_{j} \mbf c_j^T =
  \sum_{i =1}^J \tba_{\pi^{-1}(i)} \mbf c_{\pi^{-1}(i)}^T
  = \sum_{i =1}^J \ba_i \mbf c_{\pi^{-1}(i)}^T
  = \sum_{j\in Z'} \tbb_j^\ast \mbf c_{j}^T.
\end{equation*}
On the other hand, using $ B^\ast =  A (\tC_\tau)^T$ and Proposition~\ref{one-hot-representer-theorem}(iii) we can write
\begin{equation*}
  b^\ast_{mk} = \sum_{i = 1}^J a_{m \pi^{-1}(i)} c_{\pi^{-1}(i), \pi^{-1}(k)} = u_{m \pi^{-1}(k)}, 
\end{equation*}
which gives the desired identity connecting $\mbf b^\ast_k$ and $\bu^\ast_{\pi^{-1}(k)}$.
{}
\end{proof}
\subsection{Consistency Of The One-Hot Method}
\label{sec:one-hot-SSL-consistency}
In analogy with  binary classification we now discuss consistency of
multi-class classification using the one-hot method. Our results here make
use of the perturbation theory developed in subsection~\ref{sec:covariance-perturbation}.
As in subsection~\ref{sec:binary-SSL-consistency}
we consider a graph $G_0 = \{ X, W_0\}$ consisting of $K$ connected
clusters $\tZ_1, \cdots, \tZ_K$ and let $G_\eps = \{ X, W_\eps\}$ be a
family of graphs parameterized by $\eps >0$ that are perturbations of $G_0$. We denote by $C_{\tau,\eps}$ the covariance matrix corresponding to the graph $G_\eps$ as defined in \eqref{L-eps-perturbation-expansion}. 
Further, we assume the data $y$ is generated by a ground truth function $U^\dagger \in \mbb R^{M \times N}$
that is,
\begin{equation}\label{one-hot-noise-model-rescaled-std}
  y(j) = S(  \bu^\dagger_j + {\pmb \eta}_j), \qquad j \in Z',
\end{equation}
where $S$ is defined in \eqref{one-hot-S} 
and we define the noise $\eta_{mj}$ through a reference random variable
analogous to \eqref{noise-scaling-iid-sequence-probit}:
\begin{equation}\label{noise-scaling-iid-sequence-one-hot}
  \eta_{mj} = \gamma \teta_{mj}, \qquad \teta_{mj} \iidsim \psi,
\end{equation}
where $\psi$ is a mean zero PDF with unit standard deviation.
In the same spirit as subsection~\ref{sec:binary-SSL-consistency}, our first task is to estimate the probability of the event where the observed labels $y(j)$
are exact, i.e., $y(j)$ coincides with the index of the maximal element in the $j$-th column of $U^\dagger$.

\begin{lemma}\label{data-is-exact-with-high-probability-multiclass}
  Suppose $\psi$ is log-concave then 
there exist constants $\omega_1, \omega_2 > 0$ so that 
\begin{equation*}
  \begin{split}
  \mbb P \big( y(j) &= S(\bu_j^\dagger),\: \forall j \in Z'\big)\\
 & \ge \int_{\mbb R} \psi_\gamma(t)
  \prod_{j \in Z'} \bigg[  1 - \omega_2
    \exp \left( - \frac{\omega_1}{\gamma} \left(t +
        \min_{k\neq y(j)}\left\{ u^\dagger_{y(j)j} - u^\dagger_{kj} \right\}\right)  \right) \bigg]^{M-1}  dt.
  \end{split}
\end{equation*}
That is,
if  $\gamma$ is  small the data $y$ is exact  with high probability.
\end{lemma}

\begin{proof}
  Observe that $y(j) = S(\bu^\dagger_j)$ whenever $u^\dagger_{y(j) j} + \eta_{ y(j) j} \ge
  u^\dagger_{mj} + \eta_{mj}$ for all $m \neq y(j)$. Therefore,
  \begin{equation*}
    \begin{split}
      \mbb P \left( y(j) = S(\bu_j^\dagger)\: \forall j \in Z'\right)
      &= \mbb P ( \{ \eta_{mj} \le \eta_{y(j) j} + u^\dagger_{y(j) j} - u^\dagger_{mj} : j \in Z',
      m \neq y(j)\} )\\
      &= \mbb E \: \mbb P \left( \{ \eta_{mj} \le \eta_{y(j)j} + u^\dagger_{y(j)j} - u^\dagger_{mj},
        : j \in Z',  m \neq y(j)\}
       | \eta_{y(j)j} \right).
    \end{split}
  \end{equation*}
  Now by Lemma~\ref{convex-measures-have-exp-tails} and Markov's inequality
  (see proof of Lemma~\ref{data-is-exact-with-high-probability}) along with
  independence of the $\eta_{mj}$ we conclude
 there exist constants $\omega_1, \omega_2 >0$ so that
  for fixed $j \in Z'$,
  \begin{equation*}
    \begin{split}
    \mbb P
    &\big( \{ \eta_{mj} \le \eta_{y(j)j}  + u^\dagger_{y(j)j} - u^\dagger_{mj}, :  m \neq y(j)\}
    | \eta_{y(j)j} \big)\\
     & \ge \mbb P
     \big( \{ \eta_{mj} \le  \eta_{y(j)j}  + \min_{k\neq y(j) }\{ u^\dagger_{y(j)j} - u^\dagger_{kj}\} :
     m \neq y(j)\}
    | \eta_{y(j)j} \big)\\
    &\ge \prod_{m \neq y(j)} \bigg[  1 - \omega_2
    \exp \left( - \frac{\omega_1}{\gamma} \left(\eta_{y(j)j} +
        \min_{k\neq y(j)}\{ u^\dagger_{y(j)j} - u^\dagger_{kj} \}\right)  \right) \bigg] \\
    & = \bigg[  1 - \omega_2
    \exp \left( - \frac{\omega_1}{\gamma} \left(\eta_{y(j)j} +
        \min_{k\neq y(j)}\{ u^\dagger_{y(j)j} - u^\dagger_{kj} \}\right)  \right) \bigg]^{M-1}.
  \end{split}
\end{equation*}
Therefore,
\begin{equation*}
  \begin{split}
  &\mbb P \left( \{ \eta_{mj} \le \eta_{y(j)j} + u^\dagger_{y(j)j} - u^\dagger_{mj}, : j \in Z',  m \neq y(j)\}
    | \eta_{y(j)j} \right)\\
  &\ge \prod_{j \in Z'} \bigg[  1 - \omega_2
    \exp \left( - \frac{\omega_1}{\gamma} \left(\eta_{y(j)j} +
        \min_{k\neq y(j)}\left\{ u^\dagger_{y(j)j} - u^\dagger_{kj} \right\}\right)  \right) \bigg]^{M-1}.
   \end{split}
\end{equation*}
Integrating the above bound over $\eta_{y(j)j}$ gives the desired result.
\end{proof}

The following lemma is the analog of Lemma~\ref{as-exactness-of-data-probit}
for the one-hot  model \eqref{one-hot-noise-model-rescaled-std}.
The method of proof is identical to that lemma
and is therefore omitted. 

  \begin{lemma}\label{as-exactness-of-one-hot-data}
Suppose \eqref{one-hot-noise-model-rescaled-std} and \eqref{noise-scaling-iid-sequence-one-hot}
    hold, $\psi$ is log-concave and 
    \begin{equation}\label{one-hot-as-convergence-assumption}
   \as{ \min_{m,k\in \{1, \cdots, M\},  m \neq k} \{ |u^\dagger_{mj} - u^\dagger_{kj}| \} > \theta >0
    \qquad \forall j \in Z'. }
  \end{equation}
 Then for any sequence $\gamma\downarrow 0$, $y(j) \asto S(\bu^\dagger_j)$ with respect to $\prod_{(m,j)\in \{1, \dots, M\}\times Z'} \psi(t_{mj})$
   the law of the i.i.d. sequence  $\{ \teta_{mj}\}_{(m,j) \in \{1, \cdots, M\} \times Z'}$.
\end{lemma}

With the above lemmata at hand we are now in a position to study consistency of minimizers
of the one-hot functional
\begin{equation}
  \label{one-hot-J-eps}
  \mcl J_{\tau, \eps, \gamma} ( U) : = \frac{1}{2} \langle C_{\tau, \eps}^{-1}, U^T U \rangle_F
  + \Phi_\gamma(U; y), \qquad U \in \mbb R^{M \times N},
\end{equation}
where
\begin{equation*}
  \Phi_{\gamma}(U;y) := - \sum_{j \in Z'} \log \tPsi_\gamma( \bu_j; y(j)),
  \qquad   \tPsi_\gamma( \bv; m) := \int_{\mbb{R}} \psi_\gamma(t)
  \prod_{\ell \neq m} \Psi_\gamma( t + v_m - v_\ell) dt
\end{equation*}
and $\Psi_\gamma$ is the CDF of $\psi_\gamma(\cdot):=\frac{1}{\gamma}\psi(\frac{\cdot}{\gamma})$ as before.
\subsubsection{One-Hot Consistency With A Single Observed Label}
Once again we start with the case of a single observed label with $Z' = \{1\}$
belonging to the first cluster $\tZ_1$ and without loss of generality we
assume $u^\dagger_{11} > u^\dagger_{m1}$ for all $m  \neq 1$ so that the correct
label of the observed node is $1$. 

\begin{proposition}\label{one-hot-consistency-single-observation}
   Consider the single observation setting above and suppose Assumptions~
  \ref{assumptions-on-G-0} and \ref{assumptions-on-G-eps}
  are satisfied by $G_0$ and $G_\eps$. Let $U^\ast$ be the minimizer of $\mcl J_{\tau, \eps, \gamma}$ and
  let $\gamma >0$.
  \begin{enumerate}[(a)]
  \item If $\eps = o(\tau^2)$ as $\tau \to 0$  then $\exists \tau_0 >0$ so that
    $\forall (\tau, \gamma) \in (0,\tau_0) \times (0,\infty)$ and $\forall j \in \tZ_1$
    \begin{equation}\label{one-hot-single-observation-prob-lower-bound}
      \mbb P \Big( S(\bu_j^\ast) = 1 \Big) \ge  
        \int_{\mbb R} \psi_\gamma(t)
   \bigg[  1 - \omega_2
    \exp \left( - \frac{\omega_1}{\gamma} \left(t +
        \min_{m\neq 1}\left\{ u^\dagger_{11} - u^\dagger_{m1} \right\}\right)  \right) \bigg]^{M-1}  dt,
  \end{equation}
  where $\omega_1, \omega_2 >0$ are uniform constants depending only on $\psi$.
 \item If $\eps = \Theta(\tau^2)$ then the above statement holds for all $ j \in Z$.
  \end{enumerate}
\end{proposition}

\begin{proof} (a)
  The method of proof is very similar to that of
  Proposition~\ref{single-observation-consistency}.
 The dimension reduced system \eqref{one-hot-dimension-reduced-EL} takes the
simpler form
\begin{equation}
  \label{one-hot-dimension-reduced-EL-single-observation}
  \mbf b^\ast_1 = (C_{\tau, \eps})_{11} \mbf f_1( \mbf b^\ast_1).
\end{equation}
  As before, if $\eps = o(\tau^2)$
  by Proposition~\ref{geometry-of-covaraince-functions}(a)
  there exists $\tau_0 >0$ so that $\forall \tau \in (0, \tau_0)$
  \eqref{single-observation-c-1-approximation} holds and so, up to leading
  order \eqref{one-hot-dimension-reduced-EL-single-observation} is equivalent to
  \begin{equation*}
     \mbf b^\ast_1 =  |\left(\tchi_1\right)_1|^2 \mbf f_1( \mbf b^\ast_1).
   \end{equation*}
   Now consider the event where $y(1) = S(\bu^\dagger_1) = 1$, i.e., the
   data is exact. Then, we immediately see from \eqref{bf-definition} and
   the fact that $\psi_\gamma$ and $\Psi_\gamma$ are positive that,
   the only entry of 
   $\mbf f_1( \mbf b^\ast_1)$ that is not negative is the first
   entry and so $b^\ast_{11} >0$ while $b^\ast_{m1} <0$ for all $m \neq 1$.
   Finally, by Proposition~\ref{one-hot-dimension-reduction}(iii) we
   can write
   \begin{equation*}
     U^{\ast  T}= \bc_{1, \eps} \cdot \mbf f_1( \mbf b^\ast_1)^T
     = \left(\tchi_1\right)_1\tchi_1  \cdot
     \mbf f_1( \mbf b^\ast_1)^T + \mcl O \left( \eps/\tau^{2} +  \tau^{2\alpha} + \eps\right).
   \end{equation*}
   It is then straightforward to see that when $\tau$ is sufficiently small then
   for all $j \in \tZ_1$ we have $S(\bu^\ast_j) = y(1) = S(\bu^\dagger_1) = 1$
   and the claim follows by bounding the probability of the event where $y(1) = 1$
    using Lemma~\ref{data-is-exact-with-high-probability-multiclass}.
   Part (b) follows by a very similar argument to proof of
   Proposition~\ref{single-observation-consistency}(b). 
 \end{proof}

 The following corollary is the analogue of Corollary~\ref{probit-asymptotic-consistency-single-observation} for the one-hot method. The proof follows from the proof of
 Proposition~\ref{one-hot-consistency-single-observation} and
 Lemma~\ref{as-exactness-of-one-hot-data}.

\begin{corollary}\label{one-hot-asymptotic-consistency-single-observation}
  Suppose Proposition~\ref{one-hot-consistency-single-observation} and condition
    \eqref{one-hot-as-convergence-assumption} are satisfied.
    Then the following holds with a.s. convergence in the sense of Lemma~\ref{as-exactness-of-one-hot-data}:
  \begin{enumerate}[(a)]
  \item For any sequence $\gamma, \tau, \eps \downarrow 0$ along which
    $\eps = o(\tau^2)$ it holds that
      \begin{equation*}
       S\left( \bu^\ast_j \right) \asto S\left(\bu^\dagger_1\right),
      \qquad \forall j \in \tZ_1.
    \end{equation*}
  \item For any sequence $\gamma, \tau, \eps \downarrow 0$ along which
      $\eps = \Theta(\tau^{2})$ the above statement holds true for all $j\in Z$. 
\end{enumerate}
\end{corollary}



\subsubsection{One-Hot Consistency With  Multiple Observed Labels}
We finally consider the general setting where multiple labels are observed
and $|Z'| = J \ge 2$. We  need the analog of Assumption~\ref{assumptions-on-u-dagger-binary}
in multi-class classification:

\begin{assumption}\label{assumptions-on-u-dagger-multiclass}
  Let $\tZ_k$ be a cluster within which a label has been observed, i.e., $\tZ_k \cap Z' \neq \emptyset$.
  Then $S(\bu_j^\dagger)$ is constant for all  $j \in \tZ_k$. 
\end{assumption}

\begin{proposition}\label{one-hot-consistency-multiple-observation}
  Consider the multiple observation setting above
   and suppose Assumptions~
   \ref{assumptions-on-G-0}, \ref{assumptions-on-G-eps} and
   \ref{assumptions-on-u-dagger-multiclass}
   are satisfied by $G_0$, $G_\eps$ and $U^\dagger$. Let $U^\ast$ be the minimizer of
   $\mcl J_{\tau,\eps,\gamma}$  and  $Z''$ be as in \eqref{Z-pp} the set of nodes in clusters for which labels have been observed.
   If $\eps = o(\tau^2)$ as $\tau \to 0$ then $\exists \tau_0 >0$ so that
   $\forall (\tau, \gamma) \in (0, \tau_0) \times (0,\infty)$ and $\forall j \in Z''$
    \begin{equation*}
      \mbb P \Big( S\left( \bu^\ast_j \right) = S\left( \bu^\dagger_j \right)  \Big) \ge
      \int_{\mbb R} \psi_\gamma(t)
  \prod_{j \in Z'} \bigg[  1 - \omega_2
    \exp \left( - \frac{\omega_1}{\gamma} \left(t +
        \min_{m\neq y(j)}\left\{ u^\dagger_{y(j)j} - u^\dagger_{mj} \right\}\right)  \right) \bigg]^{M-1}  dt,
  \end{equation*}
  where $\omega_1, \omega_2 >0$ are uniform constants depending only on $\psi$.
\end{proposition}

\begin{proof}
  The proof follows similar steps to the binary result in
  Proposition~\ref{multiple-observation-consistency} and we use the same
  notation as in the proof of that result.
Here the dimension reduced system \eqref{one-hot-dimension-reduced-EL} takes the form
\begin{equation}
\label{multiclass-multiple-observation-dimension-reduced-EL}
(\tC_{\tau, \eps})^{-1} B^{\ast T} = \mcl F'(B^\ast).
\end{equation}
  Then, Proposition~\ref{geometry-of-greens-functions}
  implies that $\tC_{\tau, \eps}$ is nearly block diagonal
  and \eqref{consistency-proof-tilde-C-block-diagonal-form} holds.
  Then, up to leading order \eqref{multiclass-multiple-observation-dimension-reduced-EL}
  takes the form
  \begin{equation*}
   \left(\tchi_k\right)_j^{-1} \mbf b^\ast_{\pi(j)} = 
    \sum_{i \in \tZ_k'} \left(\tchi_k\right)_i
    \mbf f_i(\mbf b^\ast_{\pi(i)}), \qquad \text{ for } j \in \tZ'_k \text{ and }
    k \in \{ 1, \cdots, K\}\,,
  \end{equation*}
  where we used the same notation as in \eqref{bf-definition} and made use of the approximation \eqref{consistency-proof-tilde-C-block-diagonal-form} for the elements of $\tC_{\tau, \eps}$.
  Once again the right hand side of the above expression is independent of $j$
  and so the $\left(\tchi_k\right)_j^{-1}\mbf b^\ast_{\pi(j)}$ vectors are constant for all $j \in \tZ'_k$.
  We then have
  \begin{equation*}
    \mbf b^\ast_{\pi(j)} = \left(\tchi_k\right)_j \sum_{i \in \tZ_k'} \left(\tchi_k\right)_i
    \mbf f_i\left( \frac{\left(\tchi_k\right)_i}{\left(\tchi_k\right)_j}\mbf b^\ast_{\pi(j)}\right), \qquad \text{ for } j \in \tZ'_k \text{ and }
    k \in \{ 1, \cdots, K\}.
  \end{equation*}
  Now consider the event where the data $y$ is exact which happens with
  high probability following Lemma~\ref{data-is-exact-with-high-probability-multiclass}.
  Since $U^\dagger$ satisfies Assumption~\ref{assumptions-on-u-dagger-multiclass}
  then $y(j)$ is constant for all $j \in \tZ'_k$.
  Using the definition of $\mbf f_j$ given in \eqref{bf-definition}, we directly
  verify that the only positive coordinate of $\mbf f_i\left( \frac{\left(\tchi_k\right)_i}{\left(\tchi_k\right)_j}\mbf b^\ast_{\pi(j)}\right)$
  is the $y(j)$-th coordinate while all other coordinates are negative. Since $\left(\tchi_k\right)_i$ and $\left(\tchi_k\right)_j$ are both strictly positive in the above, we conclude 
  $S( \mbf b^\ast_{\pi(j)}) = y(j) = S(\bu^\dagger_j) $. The desired result now follows
  by Proposition~\ref{one-hot-dimension-reduction}(iii).
\end{proof}

Similar to Corollary~\ref{one-hot-asymptotic-consistency-single-observation}
  we obtain the following result by applying Theorem~\ref{one-hot-consistency-multiple-observation}
  and Lemma~\ref{as-exactness-of-one-hot-data}.

\begin{corollary}\label{one-hot-asymptotic-consistency-multiple-observations}
Suppose Theorem~\ref{one-hot-consistency-multiple-observation}  and 
  condition \eqref{one-hot-as-convergence-assumption} are satisfied.
  Then for any sequence $\gamma, \tau, \eps \downarrow 0$ along which
  $\epsilon = o(\tau^2)$ it holds that
  \begin{equation*}
  S\left(\bu^\ast_j \right) \asto S\left(\bu^\dagger_j\right), \qquad \forall j \in Z'', 
\end{equation*}
with a.s. convergence in the sense of Lemma~\ref{as-exactness-of-one-hot-data}.
\end{corollary}


\section{Numerical Experiments}
\label{sec:NE}
In this section we turn our attention to numerical experiments that are designed to confirm
our theoretical findings, and to expand upon the behavior of the
probit and one-hot methods beyond our theory. In subsection~\ref{sec:num-exp-spectral-perturbation} we study the
spectrum of graph Laplacians on a graph consisting of three clusters that are weakly connected reaffirming the
analysis of appendix~\ref{appendix:perturbation-theory}. Subsection~\ref{sec:probit-num-exp} is dedicated to the
consistency of the probit method. Here we demonstrate a curve in the $(\eps/\tau^2, \alpha)$ plane
across which probit transitions from being consistent to propagating the majority label. Finally, we repeat similar
experiments for the one-hot method in subsection~\ref{sec:num-exp-one-hot} demonstrating a similar phase
transition curve in the $(\eps/\tau^2,\alpha)$ plane and showing that the curve is sensitive to the number of
observed labels in different clusters. 

\subsection{Spectrum Of Covariance Operators}
\label{sec:num-exp-spectral-perturbation}
We begin with a numerical demonstration of the perturbation theory
of appendix~\ref{appendix:perturbation-theory} and in particular
the result of Proposition~\ref{low-lying-spectrum-of-C}. At the same time
we will introduce a synthetic experiment which is used throughout this section for illustration.

We consider a random dataset using  points drawn from
a mixture of three Gaussian distributions centered at $(1,0,0)$, $(0, 1, 0)$ and
$(0,0,1)$ with variance $0.1$. We draw $N = 150$ points in total and
the dataset has three clusters with $50$ points in each one. In order 
to construct a weighted graph on this dataset we choose the kernel function
\begin{equation*}
  \kappa(t) = \mbf{1}_{\{t \le 0.25\}}(t),
\end{equation*}
and set $w^{(0)}_{ij} = \kappa(|x_i - x_j|)$ where $| \cdot |$ denotes the Euclidean
norm in $\mbb R^3$. Figure~\ref{num-exp-random-clusters}(a) shows the
data points as well as the connected components of the resulting weighted graph in this example.
With the weight matrix $W_0$ at hand we define the graph Laplacian operator $L_0 = D_0 - W_0$ and
let $C_{\tau, 0}^{-1} = \tau^{-2\alpha}( L_0 + \tau^{2} I)^{\alpha}$. In other words, we fix $p=0$ and so the functions $\{\tchi_k\}_{k=1}^K$ simplify to the indicator functions on the clusters $\tZ_1,...,\tZ_K$. We further define a perturbation of
this matrix by replacing $W_0$ with $W_\eps$ with entries $w^{(\eps)}_{ij} = \kappa_{\eps}( |x_i - x_j|)$ where
the perturbed kernel has the form
\begin{equation}\label{num-exp-perturbed-kernel}
  \kappa_{\eps}(t) = \kappa(t) + \eps  \exp \left( -\frac{t^2}{(0.25 + \eps)^2} \right).
\end{equation}
Clearly, the resulting weighted graph for any positive value of $\epsilon$ is fully connected but
the weight of edges connecting the three clusters are at most of order $\epsilon$.
We consider the perturbed covariance matrix $C_{\tau,\eps}^{-1} = \tau^{-2\alpha}( L_\eps + \tau^2I)^\alpha$
and study its low-lying spectrum. In the notation of appendix~\ref{appendix:perturbation-theory}
we use $\lambda_{2,\eps}$ and $\lambda_{3,\eps}$ to denote the first two non-trivial
eigenvalues of $C_{\tau, \eps}$. 
Figure~\ref{num-exp-random-clusters}(b) shows that as $\eps/\tau^2$ becomes small
$\lambda_{j,\eps} -1 $ vanishes linearly in $\eps/\tau^2$ for $j=2, 3$ which is
in perfect agreement with Proposition~\ref{low-lying-spectrum-of-C}(ii). We also observe that $\lambda_{4,\eps}$ blows up with $\tau^{-2}$ as predicted in Proposition~\ref{spectral-gap-existence}.
On the other hand, Figure~\ref{num-exp-random-clusters}(c) shows the
$\ell^2$ distance between the second and third eigenvectors of $C_{\eps, \tau}$
with their projection onto the span of the set functions $\{ \tchi_k\}_{k=1}^3$. We displayed the case $\alpha=1$, the behavior for other choices of $\alpha$ is similar.
Here $P_0$ denotes the projection onto the span of $\{ \pmb \chi_k \}_{k=1}^3$ following
the notation of appendix~\ref{appendix:perturbation-theory}. We observed that $\| ( I - P_0)
\pmb \phi_{j, \eps} \|_2$ 
goes to zero linearly in $\epsilon$ for $j=2,3$, which is precisely the rate  predicted
in Proposition~\ref{low-lying-spectrum-of-C}(iii) suggesting that
the bound is sharp. 

\begin{figure}[htp]
  \centering
  \begin{subfigure}[b]{.45 \textwidth}
    \includegraphics[width=1 \textwidth, clip =true, trim =15ex 1ex 15ex 1ex]{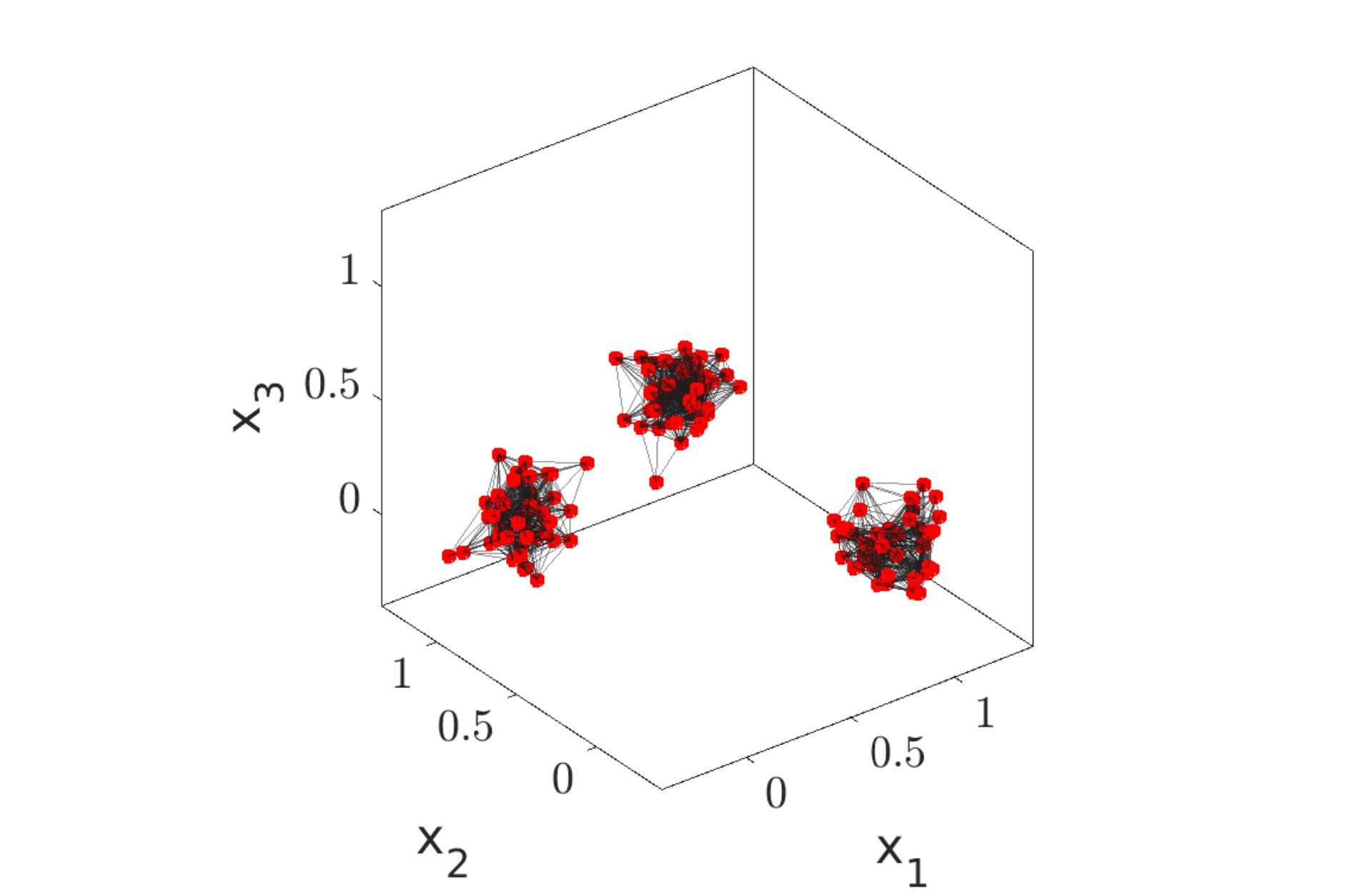}
    \caption{}
\end{subfigure} \\
\begin{subfigure}[b]{.45\textwidth}
  \includegraphics[width=1 \textwidth, clip =true, trim =15ex 1ex 15ex 1ex]{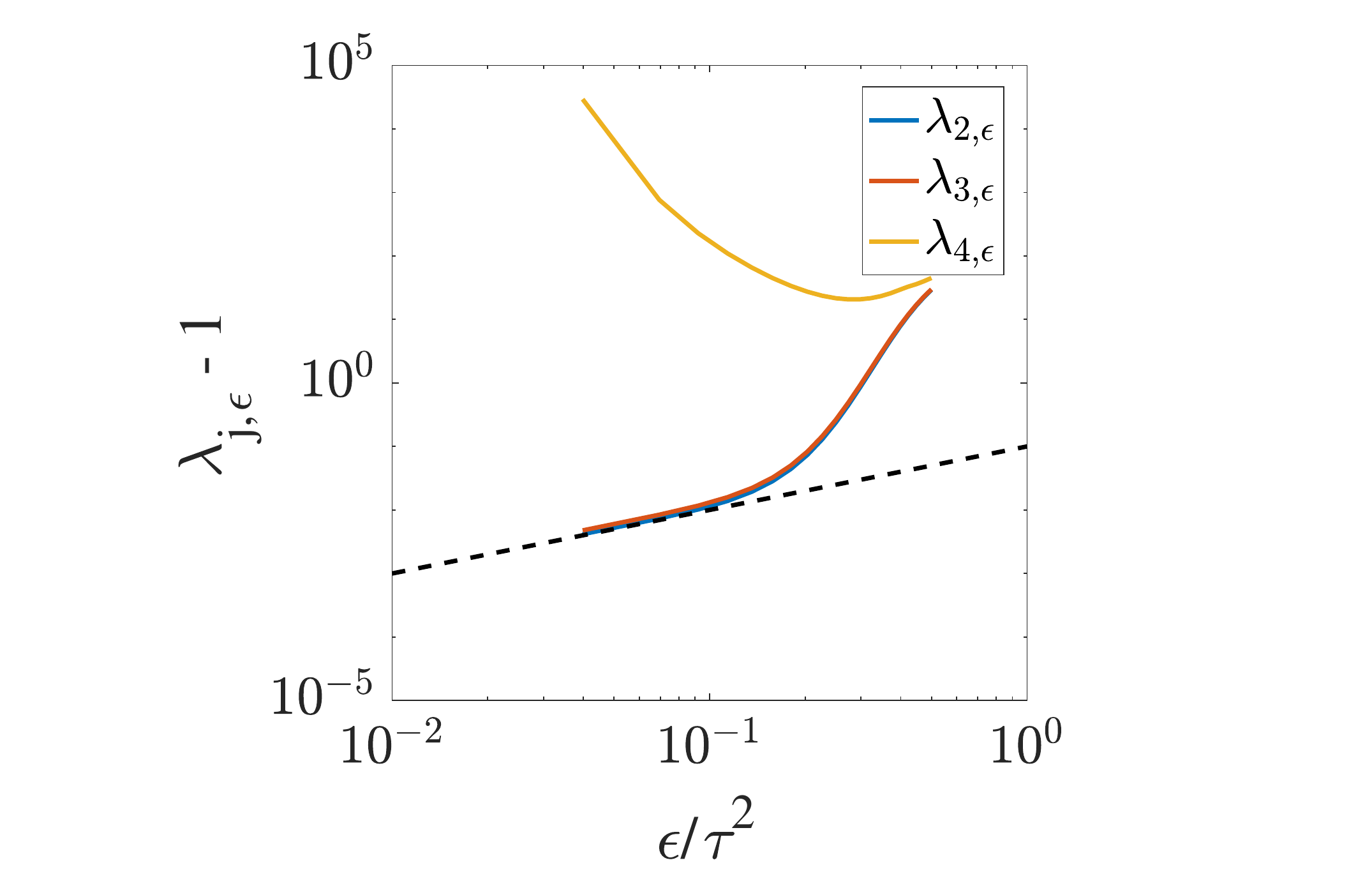}
  \caption{}
  \end{subfigure}
  \begin{subfigure}[b]{.45 \textwidth}
    \includegraphics[width=1 \textwidth, clip =true, trim =15ex 1ex 15ex 1ex]{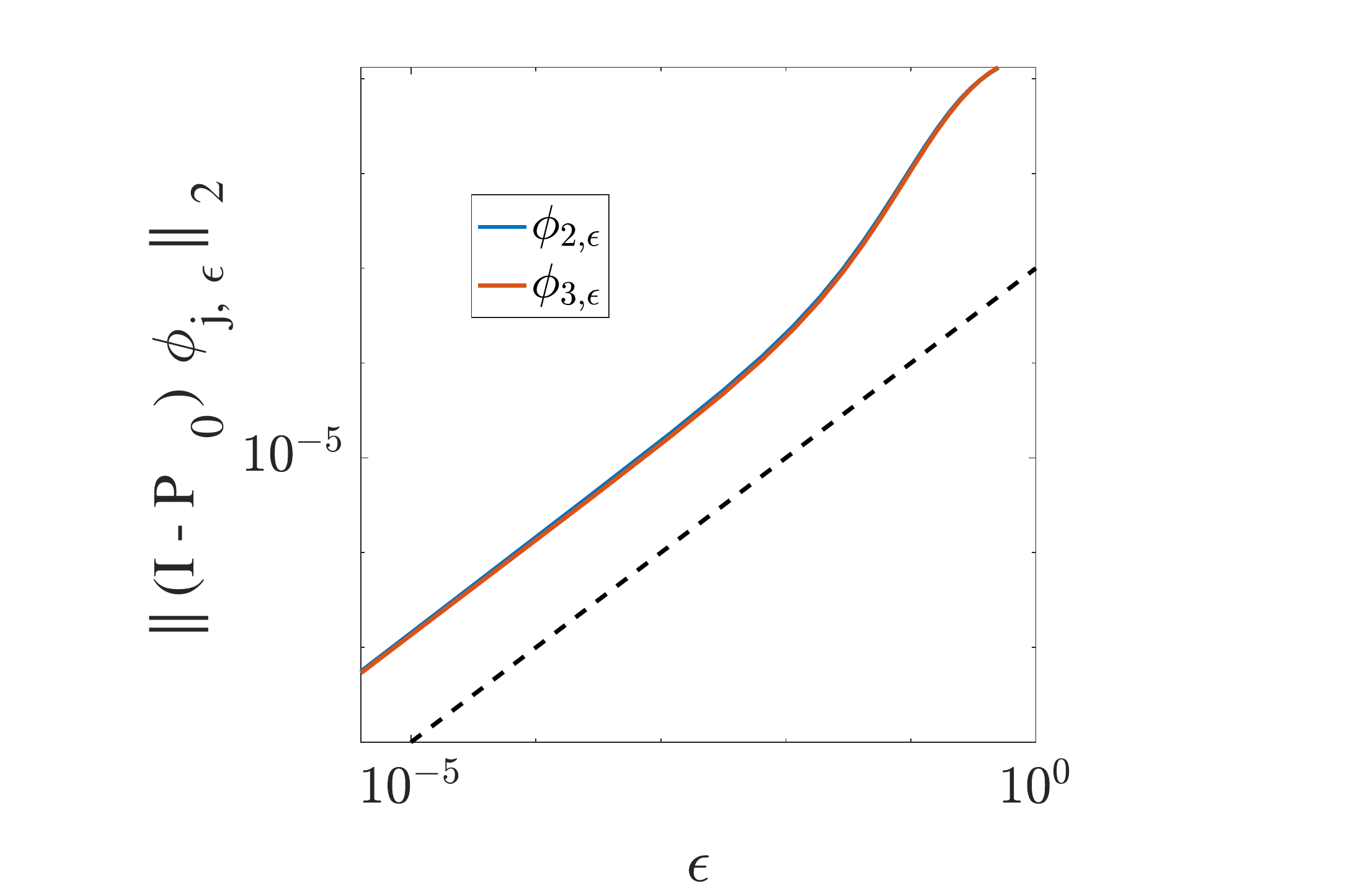}
    \caption{}
  \end{subfigure}
  \caption{(a) The \as{disconnected} graph $G_0$ constructed by random draws from
a mixture of three Gaussians. (b) The first three eigenvalues of the perturbed inverse 
covariance operator $C_{\tau, \eps}^{-1}$ with $\alpha=1$ for different choices of $\eps$ and $\tau^2$. 
The fourth eigenvector  $\lambda_{4,\epsilon}$ blows up with $\tau^{-2}$ as predicted 
while $\lambda_{2, \epsilon}$ and $\lambda_{3, \epsilon}$ converge to $1$ linearly in $\eps /\tau^2$. 
(c) The distance between $\pmb \phi_{2, \eps}$ and $\pmb \phi_{3,\eps}$ and their projections onto the span of 
the set functions $\{ \pmb \chi_j\}_{j=1}^3$. The distance vanishes linearly in $\epsilon$.}
  \label{num-exp-random-clusters}
\end{figure}

\subsection{Binary  Classification With Probit}
\label{sec:probit-num-exp}
We now consider a binary classification problem on the synthetic dataset of
subsection~\ref{sec:num-exp-spectral-perturbation}. Figure~\ref{num-exp-probit-consistency}(a)
shows the true label of the 150 points in the dataset. Blue points have label $+1$
while red points have label $-1$. For the SSL problem
we assume one label is observed within each cluster and that
the observed labels are correct, i.e., $y  = (+1, +1, -1)^T$.
We take the
observation noise $\eta_j$ to be i.i.d.  logistic random variables with mean zero. More precisely,
\begin{equation*}
  \psi_\gamma(t) = \frac{\exp(- t/\gamma) }{\gamma (1 + \exp( - t/ \gamma))}, \qquad
  \Psi_\gamma(t) = (1 + \exp(-t/\gamma))^{-1}.
\end{equation*}
We use the $\kappa_\eps$ kernel of \eqref{num-exp-perturbed-kernel} and construct
the graph Laplacian $L_\eps$ and the covariance operator $C_{\tau, \eps}$ as in
subsection~\ref{sec:num-exp-spectral-perturbation} above.
 Since the matrix $C_{\eps, \tau}^{-1}$ is
 ill-conditioned for small $\tau$ and large values of $\alpha$ we
 found it crucial to 
use a low-rank approximation to $C_{\eps, \tau}$ in order to solve
the EL equations \eqref{binary-EL} in a stable manner.
Following the expansion \eqref{c-j-eigenvector-expansion}
and the fact that there exists a uniform spectral gap between $\lambda_{3,\eps}$ and
$\lambda_{4,\eps}$ (recall Figure~\ref{num-exp-random-clusters}(b)) we consider
the truncated expansion
\begin{equation}\label{hat-C-eps-tau}
  \hat{C}_{\eps, \tau} = \sum_{k=1}^{n} \frac{1}{\lambda_{k,\eps}} \pmb \phi_{k, \eps}  \pmb \phi_{k,\eps}^T\,,
\end{equation}
where a suitable truncation $n<N$ will be chosen later,
and solve the approximate EL equations
\begin{equation*}
  \bu^\ast =  \sum_{j\in Z'} \hat{C}_{\eps, \tau} F_j(u^\ast_j) \be_j.
\end{equation*}
We used MATLAB's \texttt{fsolve} function for this task.
For our first set of experiments we fix the noise parameter
$\gamma = 0.5$ and used $n=10$ terms in the approximation
of $\hat{C}_{\tau, \eps}$. We then vary $\eps$, $\tau$ and $\alpha$.
We considered  $\tau \in (0.01, 1)$,
$\eps/\tau^2 \in (0.01, 0.5)$, and $\alpha \in (0.25, 10)$, i.e., for each value of $\alpha$ we pick two  sequences of 
$\tau$ and $\eps/\tau^2$ values and set $\eps = \tau^2 \times \eps/\tau^2$. 
Figure~\ref{num-exp-probit-consistency}(b) shows the
percentage of mislabelled points when $\sgn(\bu^\ast)$ is used
as the label predictor. The maximum error of $33\%$ corresponds to
all red points being labelled as blue. Our results suggest that 
when $\eps/\tau^2$ is large the Probit classifier $\bu^\ast$ tends to
assign the majority labels to all points in the dataset. Furthermore,
we observe a sharp transition between perfect label recovery and
assignment of majority labels. This effect is amplified for larger values of $\alpha$
in that the transition seems to happen for a smaller value of $\eps /\tau^2$.

We also consider the distance between the span of the set functions $\pmb \chi_j$
and the minimizer $\bu^\ast$. Figure~\ref{num-exp-probit-consistency}(c)
shows $\| ( I - P_0) \bu^\ast\|_{2}$ as a function of $\tau^2$ and for
different values of $\alpha$. We see that for smaller values of $\alpha$
the projection error is $\mcl O( \tau^{2\alpha})$ which is in line with the
predicted error between $\mbf c_{j,\eps}$ and the span of $\pmb \chi_j$
in Proposition~\ref{geometry-of-covaraince-functions}(a) since for smaller
values of $\alpha$ the leading order error term is $\mcl O(\tau^{2\alpha})$. When $\alpha$ is large 
the projection error is controlled by the $\mcl O(\eps/\tau^2 + \eps)$ terms in the
error bound of Proposition~\ref{geometry-of-covaraince-functions}(a)  and no longer depends on $\alpha$.

We noticed that the labelling accuracy is more or less independent of $\gamma$ so long
as the data $y$ is correct as demonstrated in
Figure~\ref{num-exp-consistency-with-different-gamma-and-n}.
We also note that, for the most part, the labelling accuracy is 
independent of $n$ as well. As shown in
Figure~\ref{num-exp-consistency-with-different-gamma-and-n} the labelling accuracy
increases slightly only for small values of $\alpha$ and larger values of $\eps/\tau^2$. 

\begin{figure}[htp]
  \centering
  \begin{subfigure}[htp]{.42 \textwidth}
    \includegraphics[width=1 \textwidth, clip =true, trim =20ex 1ex 20ex 1ex]
{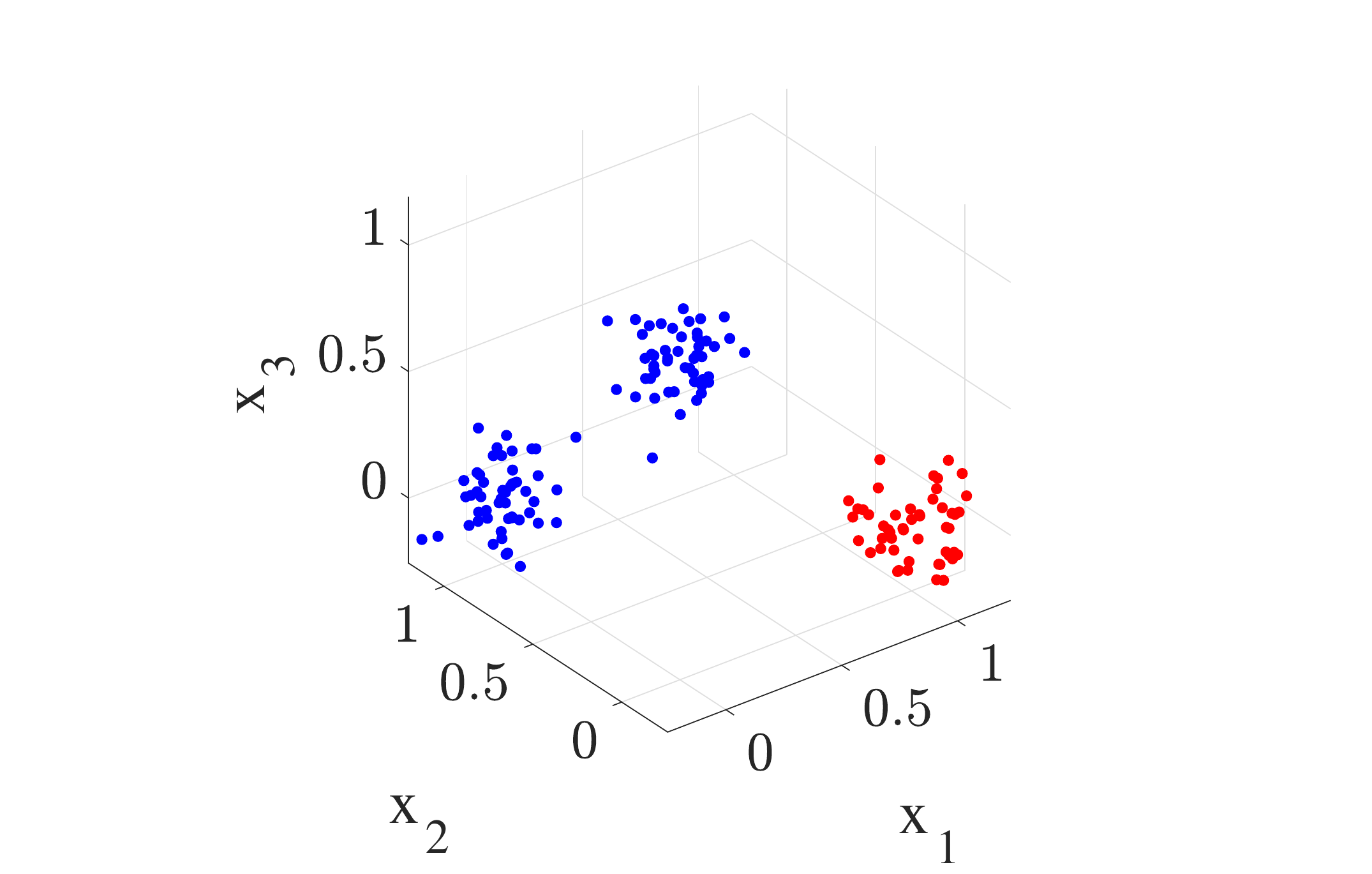}
    \caption{}
  \end{subfigure}\\
    \begin{subfigure}[htp]{.42 \textwidth}
    \includegraphics[width=1 \textwidth, clip =true, trim =20ex 1ex 20ex 1ex]
    {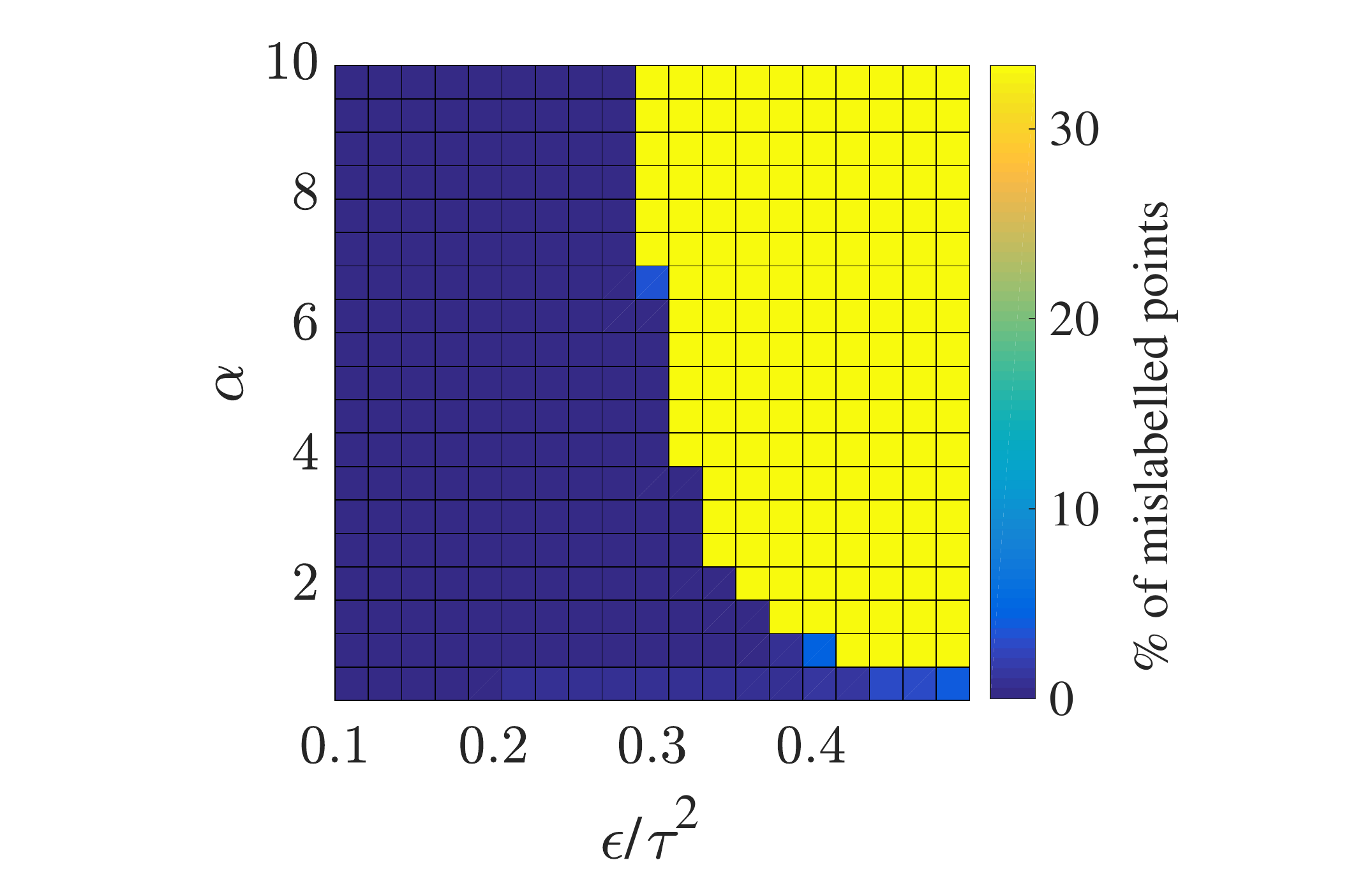}
    \caption{}
  \end{subfigure}
\begin{subfigure}[htp]{.45 \textwidth}
  \includegraphics[width=1 \textwidth, clip =true, trim =0ex 35ex 0ex 40ex]
  {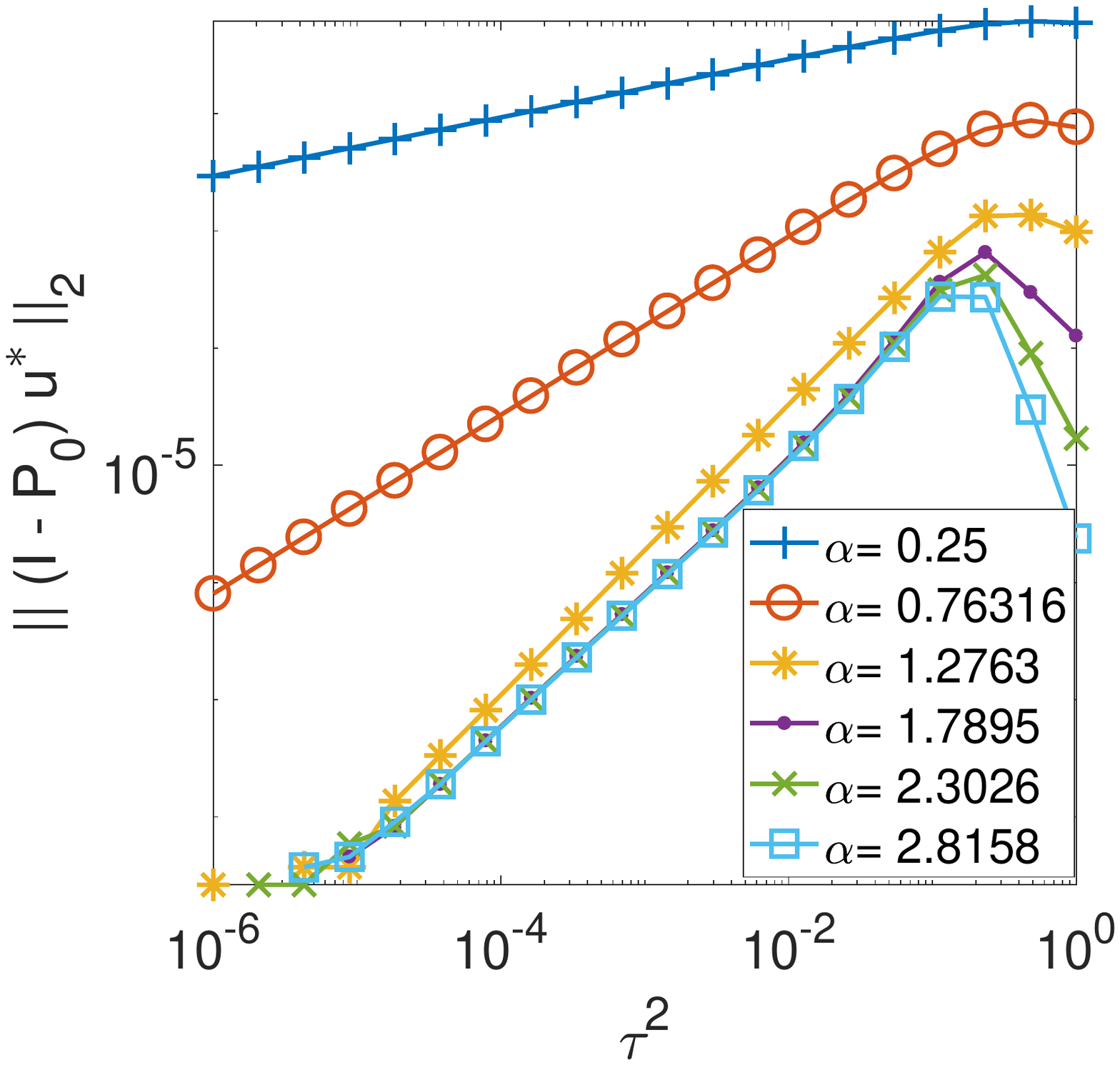}
  \caption{}
  \end{subfigure}
  \caption{(a) The true binary labels in the synthetic dataset of subsection~\ref{sec:probit-num-exp}.
    Blue points have label $+1$ and red points have label $-1$.
    (b) Heat map of the percentage of mislabelled points using the probit classifier.
    The $33\%$ error mark corresponds to assigning label $+1$ to all red points which
    constitute a third of the dataset. For fixed values of $\alpha$
    we observe a sharp transition as $\eps/\tau^2$ increases
    where we go from labelling all of the points correctly to labelling the red points as blue.
    (c) The distance between $\bu^\ast$ and the span of $\{ \pmb \chi_j\}_{j=1}^3$.
    This distance is $\mcl O(\tau^{2\alpha})$ when $\alpha$ is small and does not depend on $\alpha$ when it is large indicating 
that the distance is controlled by $\mcl O (\eps/\tau^2 + \eps)$.}
  \label{num-exp-probit-consistency}
\end{figure}

\begin{figure}[htp]
  \centering
  \begin{subfigure}[b]{.32 \textwidth}
    \includegraphics[width=1 \textwidth, clip =true, trim =20ex 1ex 20ex 1ex]
{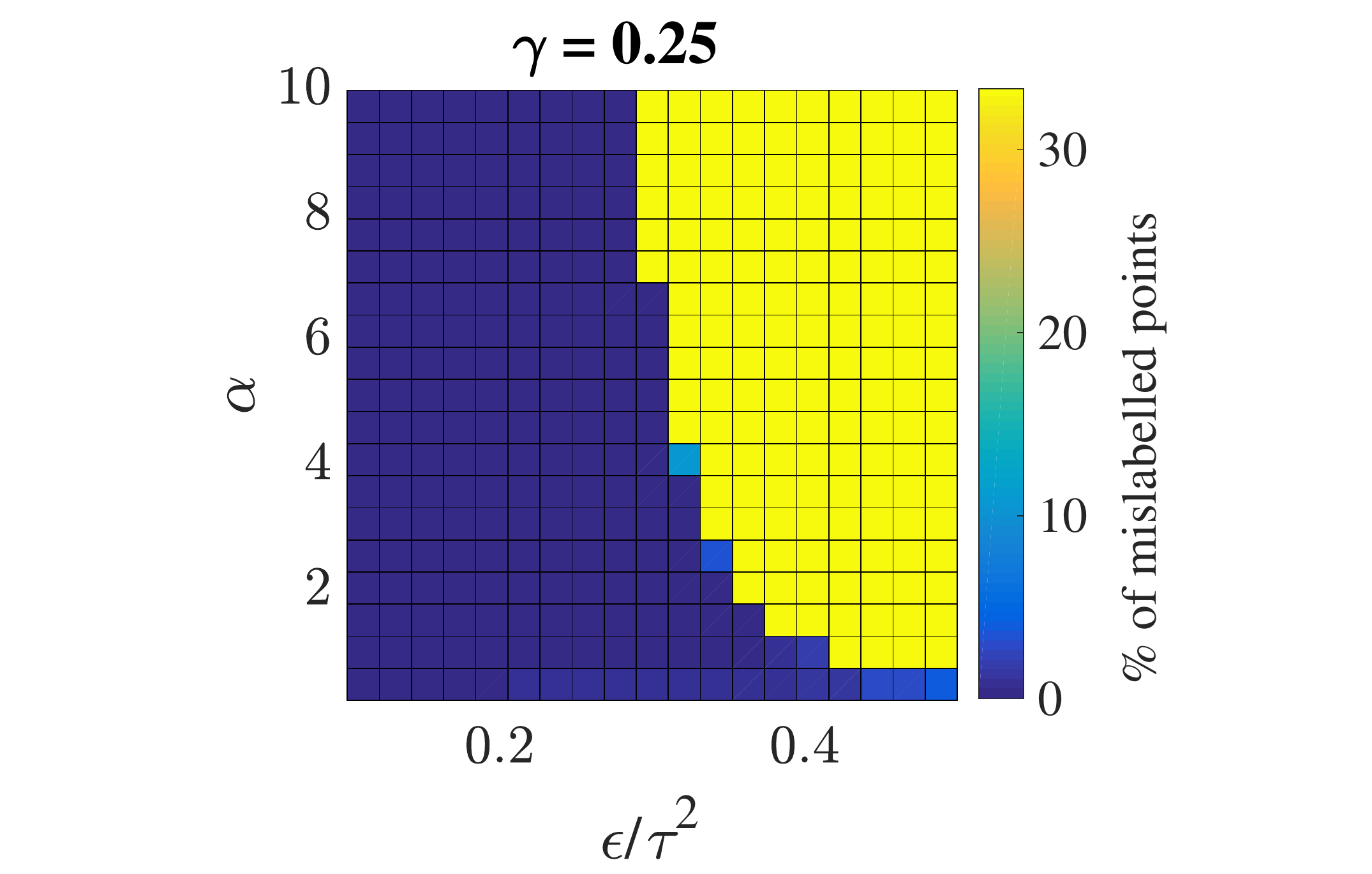}
  \end{subfigure}
    \begin{subfigure}[b]{.32 \textwidth}
    \includegraphics[width=1 \textwidth, clip =true, trim =20ex 1ex 20ex 1ex]
    {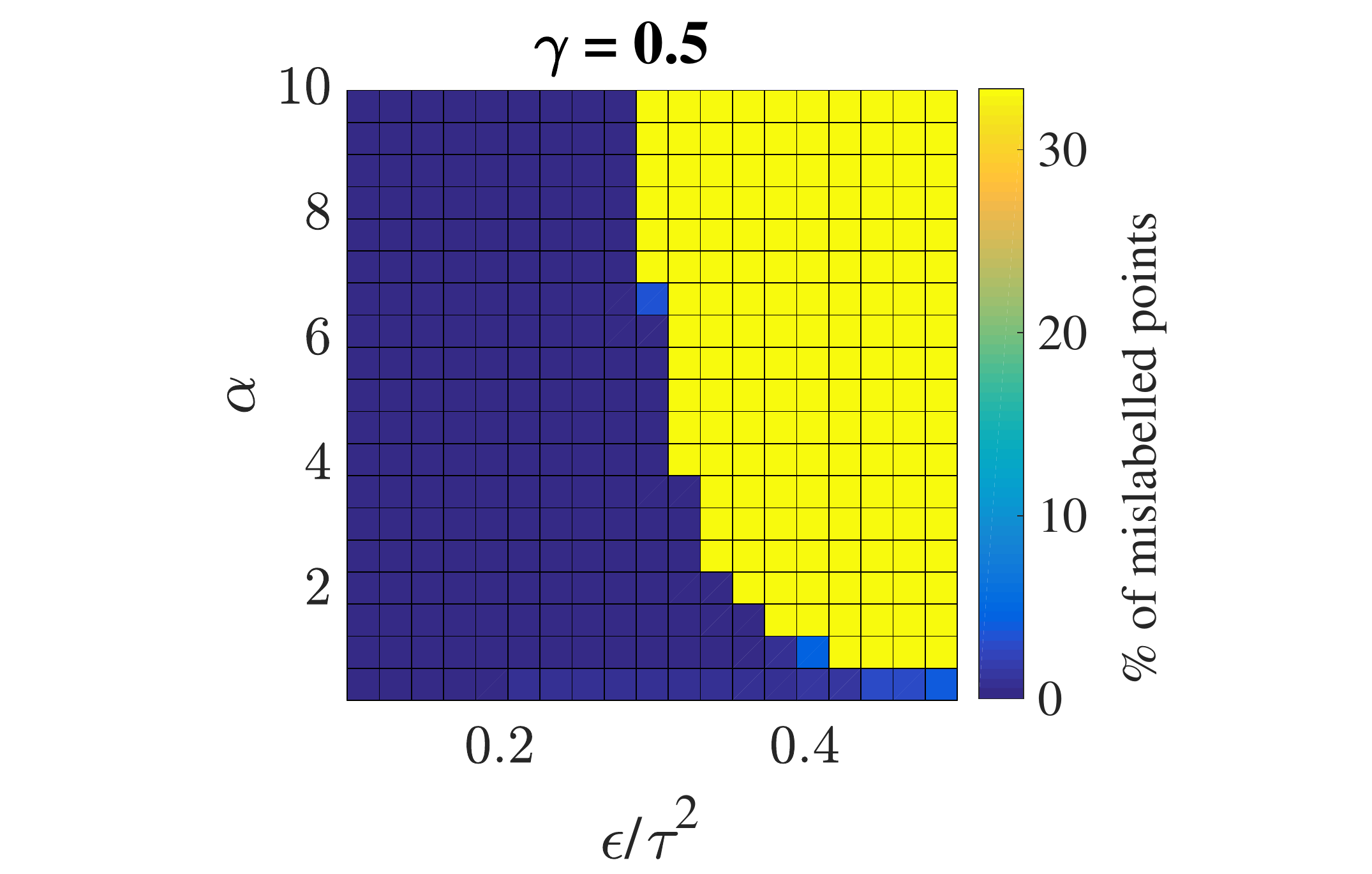}
  \end{subfigure}
\begin{subfigure}[b]{.32 \textwidth}
  \includegraphics[width=1 \textwidth, clip =true, trim =20ex 1ex 20ex 1ex]
  {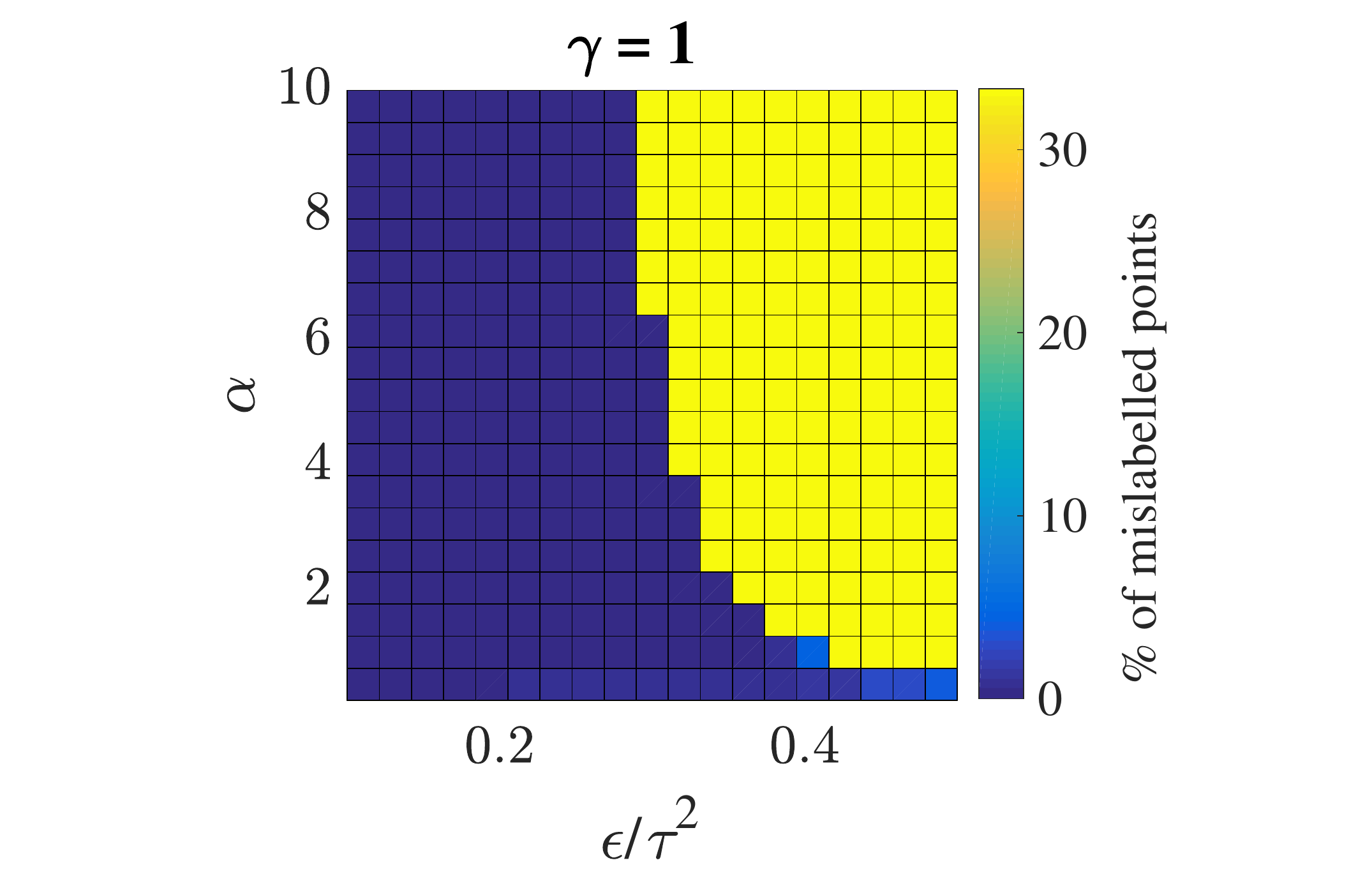}
\end{subfigure}
  \begin{subfigure}[b]{.32 \textwidth}
    \includegraphics[width=1 \textwidth, clip =true, trim =20ex 1ex 20ex 1ex]
{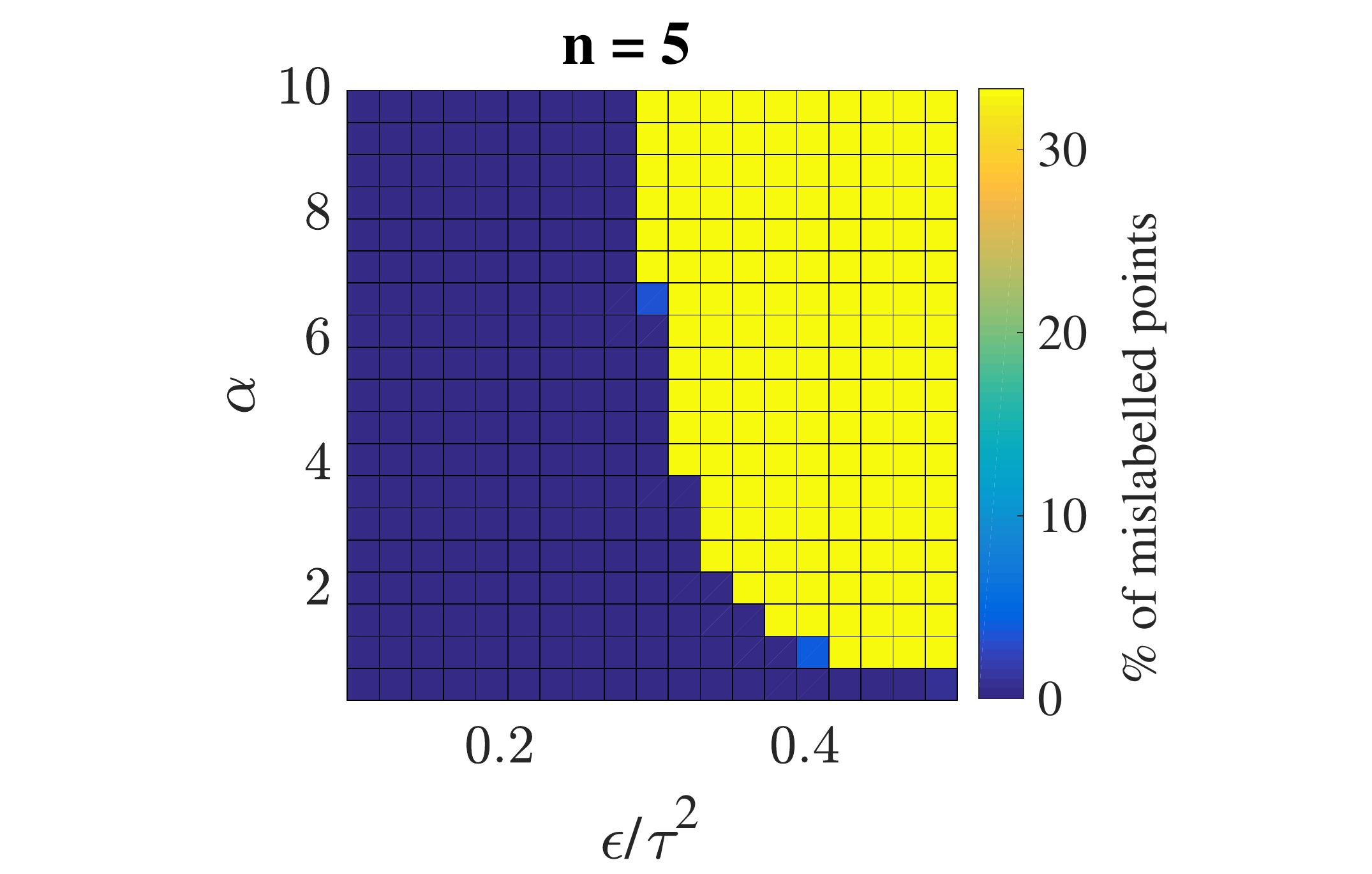}
  \end{subfigure}
    \begin{subfigure}[b]{.32 \textwidth}
    \includegraphics[width=1 \textwidth, clip =true, trim =20ex 1ex 20ex 1ex]
    {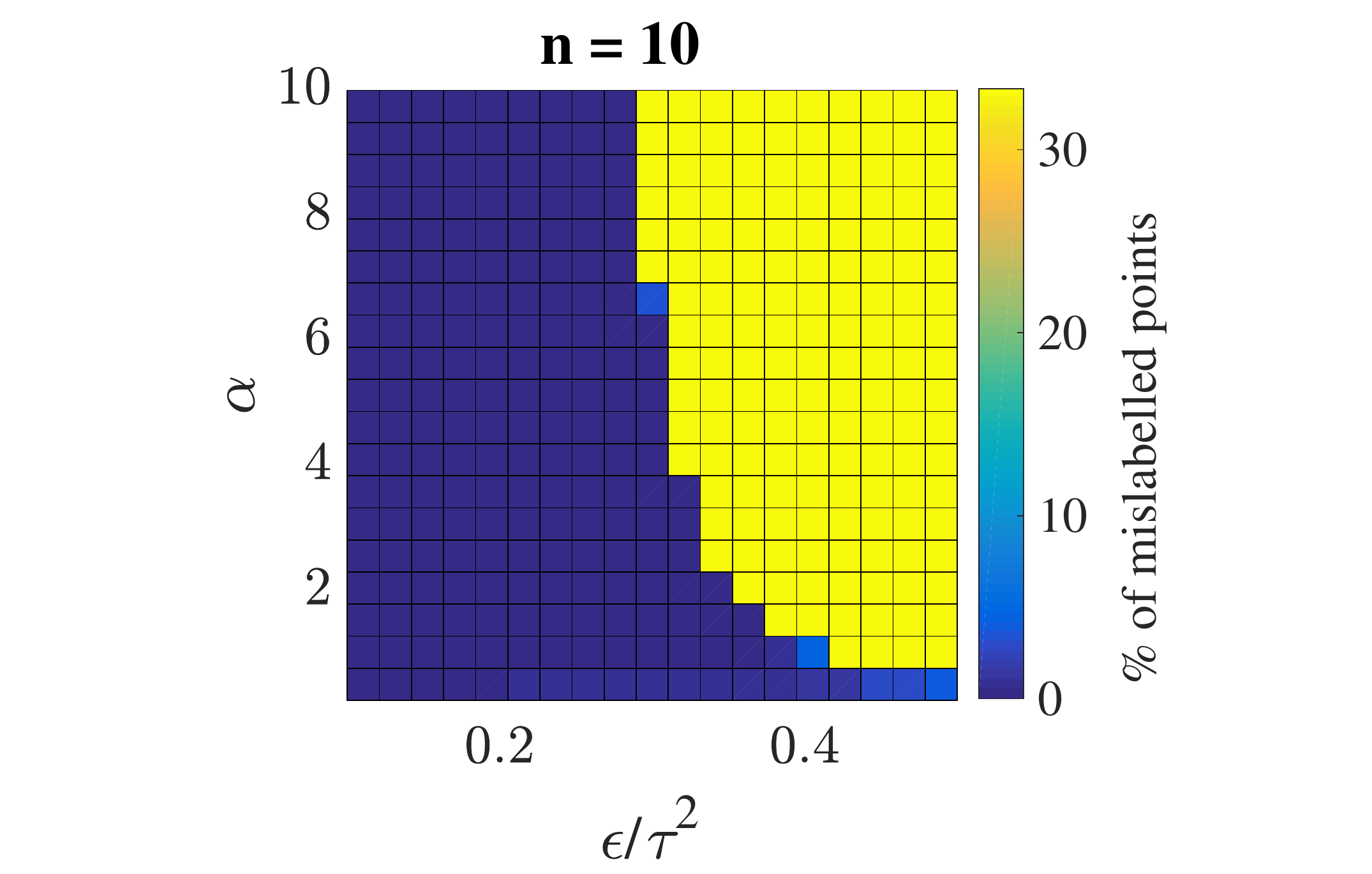}
  \end{subfigure}
\begin{subfigure}[b]{.32 \textwidth}
  \includegraphics[width=1 \textwidth, clip =true, trim =20ex 1ex 20ex 1ex]
  {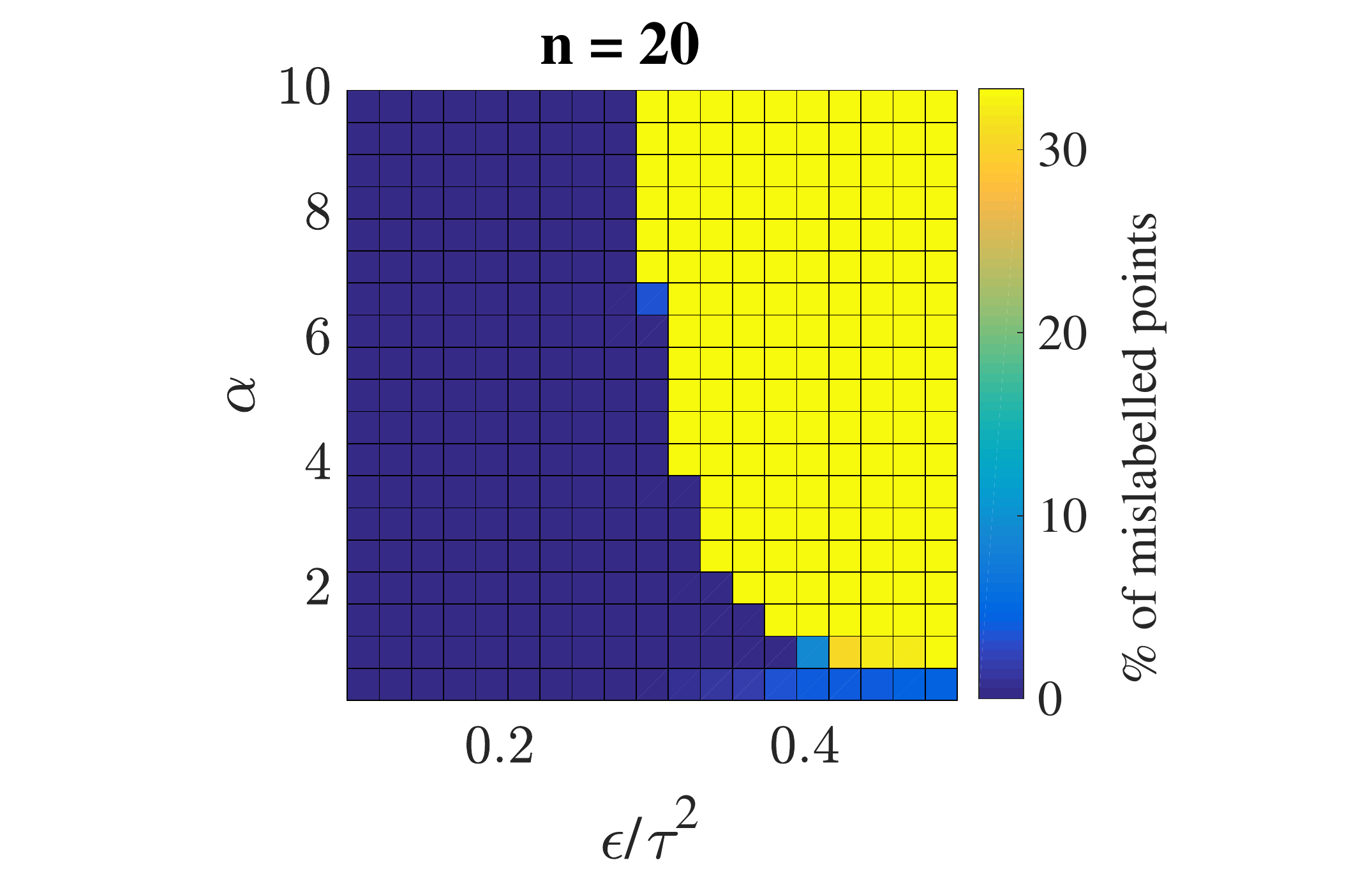}
\end{subfigure}
  \caption{Demonstrating the  labelling accuracy of probit as a function of
    measurement noise $\gamma$
    and the truncation parameter $n$. In the top row we fix $n=10$ and modify $\gamma$,
    while in the bottom row we fix $\gamma = 0.5$ and modify $n$. 
    We do not observe much sensitivity to $n$, in particular
    when $\alpha$ is large. When $\alpha$ is small we can see
    a slight increase in error for large values of $\eps /\tau^2$
    as $n$ grows larger. Modifying $\gamma$ does not have a significant
  impact so long as the data $y$ is correct.}
  \label{num-exp-consistency-with-different-gamma-and-n}
\end{figure}

For our final set of experiments we consider noisy data.
First, we fix $\tau = 0.5$, $\eps/\tau^2 = 0.1$, $n=10$, and take $\alpha \in (0.25, 10)$
and $\gamma \in (0.1, 1)$. For each value of $\alpha$ and $\gamma$ we
perform 100 experiments where we randomly perturb the data $y$ by drawing independent
measurement noise $\eta_j$ using the model
\eqref{probit-observation-model} and consider the labelling accuracy of the probit model.
If all labels are recovered correctly by $\sgn(\bu^\ast)$
we consider the experiment a success and otherwise a failure.

In Figure~\ref{num-exp-success-rate}(a), we plot the probability of success of predicting the correct
label of all points as a function of $\gamma$ and $\eps/\tau^2$ for fixed $\alpha = 2$. We chose
 $\tau = 0.5$, $\eps/\tau^2 \in (0.1, 0.6)$ and $\gamma \in (0.1, 1)$.
Here we see a clear transition in the success probability as a function of $\eps/\tau^2$.
When $\eps/ \tau^2$ is small the success probability is almost independent of
$\eps/\tau^2$ and depends only on $\gamma$ but for larger values of $\eps/\tau^2$
the success probability suddenly drops to zero meaning that
some points are always mislabelled. This behavior is in line with
Figure~\ref{num-exp-consistency-with-different-gamma-and-n}  where we observed a
sharp increase in the prediction error when $\eps/\tau^2$ is too large.
We emphasize that this behavior is also in line with
Proposition~\ref{multiple-observation-consistency} stating that
the probability of success is controlled only by $\gamma$ provided that
$\eps/\tau^{2}$ and $\eps$ are sufficiently small.

Next, we fixed $\eps/\tau^2 = 0.1$ and modified $\alpha$, see Figure~\ref{num-exp-success-rate}(b).
We do not observe any dependence of the success probability on $\alpha$.
This is in line with  Proposition~\ref{multiple-observation-consistency},
which states that if $\eps$ and  $\eps/\tau^{2}$ are sufficiently small then the
success probability is essentially controlled by the probability of the event
where the data is correct which depends only on $\gamma$.

\begin{figure}[htp]
  \centering
    \begin{subfigure}[b]{.45 \textwidth}
    \includegraphics[width=1 \textwidth, clip =true, trim =15ex 1ex 15ex 1ex]
    {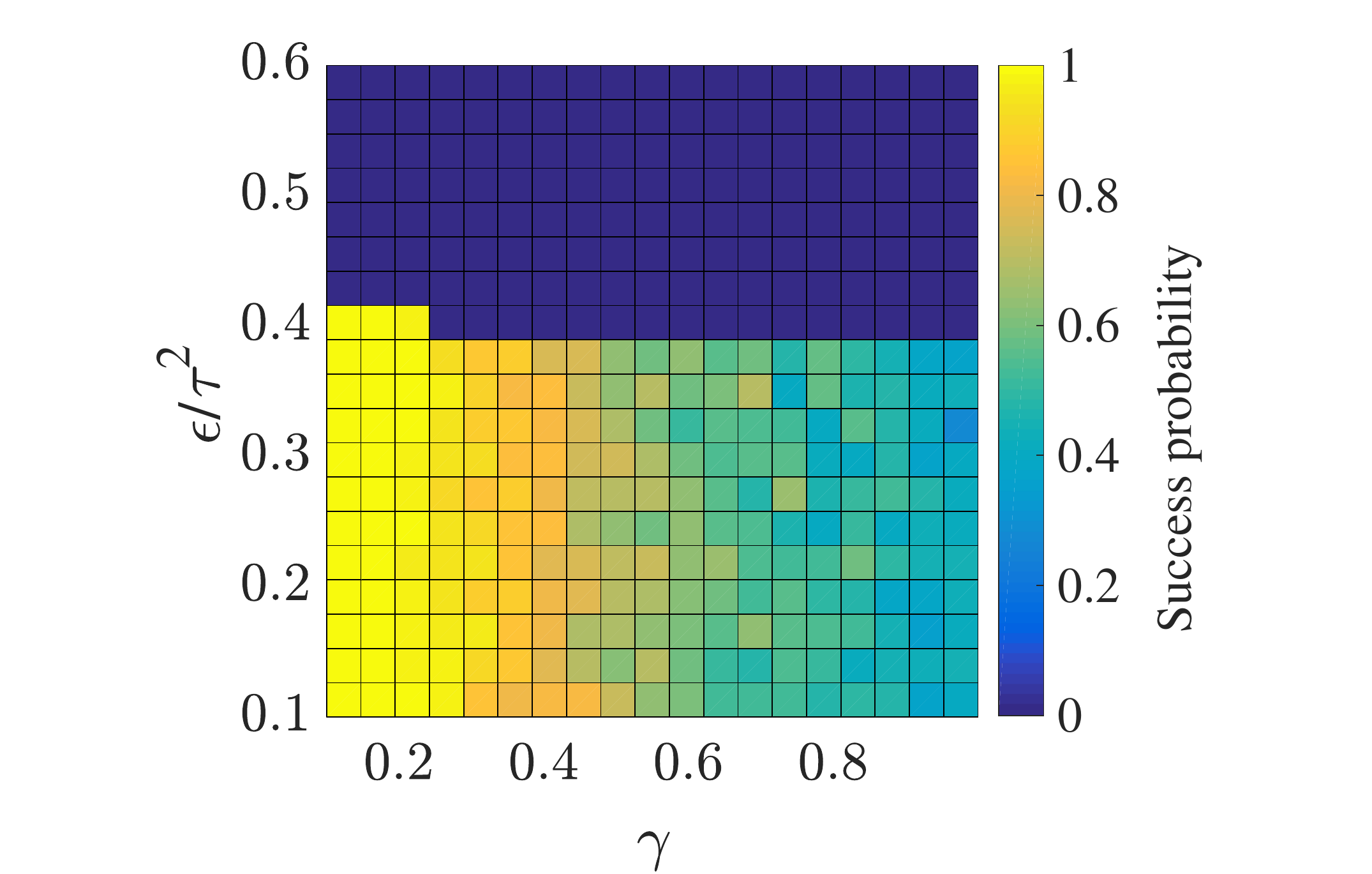}
    \caption{}
  \end{subfigure}
  \begin{subfigure}[b]{.45 \textwidth}
    \includegraphics[width=1 \textwidth, clip =true, trim =15ex 1ex 15ex 1ex]
    {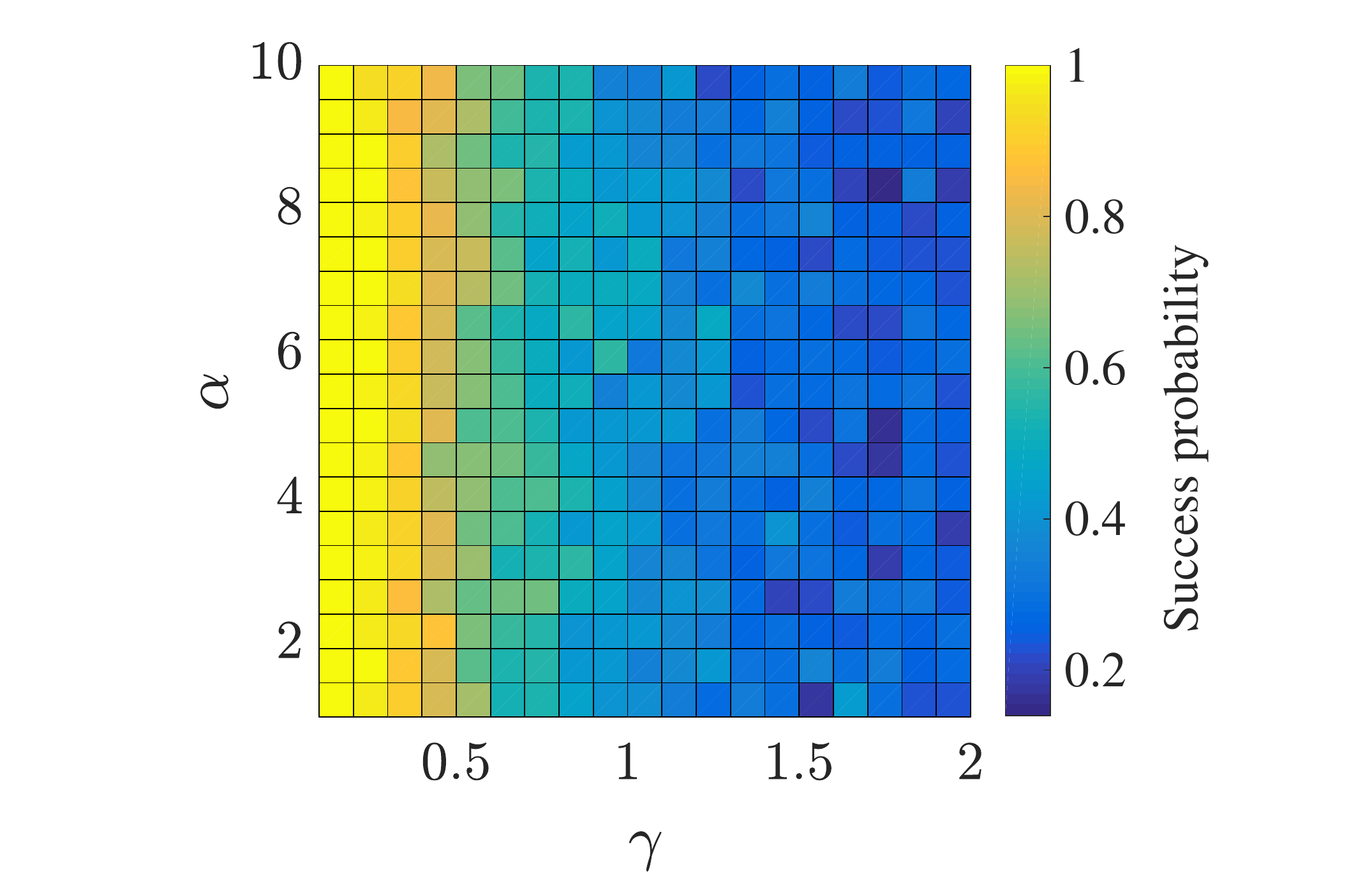}
    \caption{}
  \end{subfigure}
  \caption{Heat map of the probability of success of probit in predicting the correct label
    of all the points in the dataset. 
    (a) We plot the success probability of probit as a function of $\eps/\tau^2$
    and $\gamma$ for fixed value of $\alpha=2$. When $\eps/\tau^2$ is small
    the success probability appears to only depend on $\gamma$ but as $\eps/\tau^2$ increases we
    see a sharp transition where the success probability drops to zero indicating that some points
  are always mislabelled in this regime. 
    (b) The success probability for different
    values of $\alpha$ and $\gamma$ for  fixed value of $\eps/\tau^2 = 0.1$. We do not
    observe any strong dependence on $\alpha$ and the success probability appears to
    depend on $\gamma$ only. }
  \label{num-exp-success-rate}
\end{figure}

\subsection{Multi-Class Classification Using One-Hot}
\label{sec:num-exp-one-hot}
For the next set of numerical experiments we consider
multi-class classification with the One-hot method. Once again we
use the synthetic dataset of subsection~\ref{sec:num-exp-spectral-perturbation}
but now we assume there exist three classes within the graph as
depicted in Figure~\ref{num-exp-one-hot-clusters}.
Similar to previous sections we take $\psi_\gamma$ to be the logistic
distribution but this time use the  model \eqref{one-hot-label-model}
for the observed labels. We are assuming that we are given one label in each cluster, i.e. $J=3$. In the perfect measurement case
$y = (1, 2, 3)^T$. The main modification in this case, as compared to
binary classification is that now we need to minimize the one-hot functional 
which has a more complicated misfit function as in \eqref{one-hot-likelihood}.
Furthermore, the minimizer is now a matrix $U^\ast \in \mbb R^{M \times N}$
where $M=3$ and $N = 150$. In this case \eqref{one-hot-EL}
is a  nonlinear system of equations in $3 \times 150$ dimensions
which is slow to solve. Instead, we solve the
dimension reduced system \eqref{one-hot-dimension-reduced-EL} and
identify $U^\ast$ via $B^\ast$. We noticed that this approach offers significant
speedup in our calculations. The dimension reduced system
\eqref{one-hot-dimension-reduced-EL} is small enough that MATLAB's \texttt{fsolve}
can still be very effective. We highlight that we approximate
the matrix $\tC_{ \eps, \tau}$ by finding $\hat{C}_{\eps, \tau}$
as in \eqref{hat-C-eps-tau} and then keep only the rows and
columns of $\hat{C}_{\eps, \tau}$ that correspond to the observation vertices in $Z'$.

Overall we find that the one-hot method behaves similarly to probit
as expected following our analysis. In Figures~\ref{num-exp-one-hot-consistency-with-different-gamma-n}
 we
show the accuracy of the one-hot method in predicting the
correct label of the points using perfect observed labels $y$. Similarly to the probit case
we see little sensitivity to $n$ and $\gamma$ but a clear phase transition
in the prediction error as $\eps/\tau^2$ increases for each value of $\alpha$.
We see that the prediction error is either very small or close to $66 \%$. The
latter value is a result of all three clusters being labelled as the
same class. Overall, it seems that the prediction error is smaller and less
sensitive to $\eps/\tau^2$ when $\alpha$ is small.

Similarly to the probit case we also study the success probability of one-hot
for different values of $\alpha$, $\eps/\tau^2$ and $\gamma$. Here we
say that the one-hot classification is successful if all labels within the
dataset are
predicted correctly. As before we vary the value of $\gamma$ between $0$ and $1$
 and estimate
 the success probability of one-hot by averaging over 100 trial runs with
 randomly perturbed observed labels $y$. Figure~\ref{num-exp-one-hot-prob-success}(a)
 and (b) show a heat-map of success probability of one-hot for different
 values of $\eps/\tau^2$ and $\gamma$. Here we observe some dependence
 between the success probability and $\eps/\tau^2$. In fact, for fixed value of
 $\gamma$ we see a small increase in success probability as $\eps/\tau^2$
 increases. However, as depicted in Figure~\ref{num-exp-one-hot-prob-success}(b)
 increasing $\eps/\tau^2$ eventually leads to a sudden drop in probability of
 success. This behavior is again in line with the phase transition observed in
 Figure~\ref{num-exp-one-hot-consistency-with-different-gamma-n}.
 On the other hand, we observe in Figure~\ref{num-exp-one-hot-prob-success}(c)
 that for fixed values of $\eps/\tau^2$ the success probability is effectively
 independent of $\alpha$ and only controlled by $\gamma$. This is in line with our numerical results in the binary case, see Figure~\ref{num-exp-success-rate}(a).

\begin{figure}[htp]
  \centering
  \begin{subfigure}[b]{.45 \textwidth}
    \includegraphics[width=1 \textwidth, clip =true, trim =20ex 1ex 20ex 1ex]
    {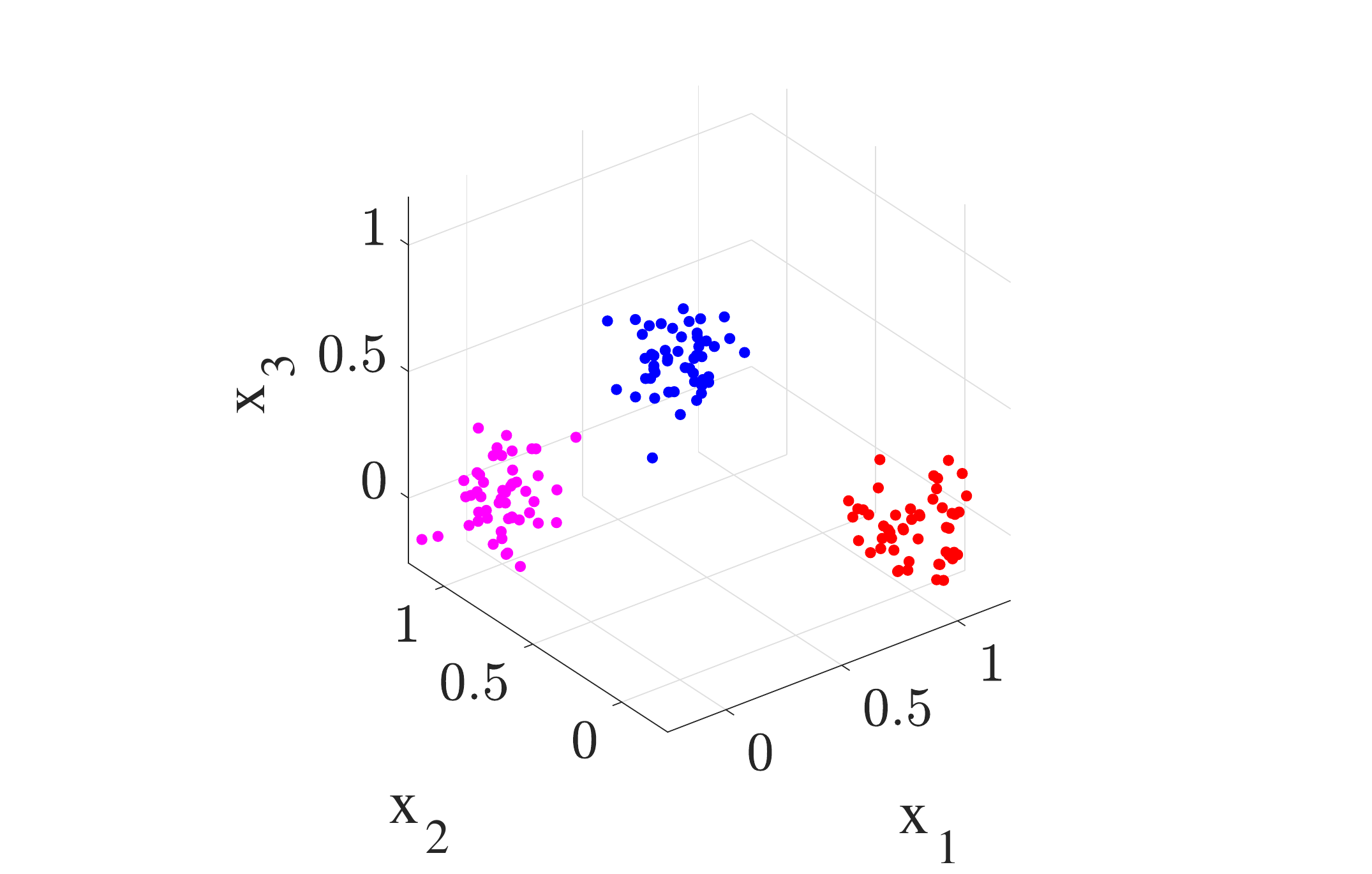}
  \end{subfigure}
  \caption{Visualization of the three classes within the three cluster dataset
  for multi-class classification using one-hot. We observe a single label in each cluster.}
  \label{num-exp-one-hot-clusters}
\end{figure}

\begin{figure}[htp]
  \centering
  \begin{subfigure}[b]{.32 \textwidth}
    \includegraphics[width=1 \textwidth, clip =true, trim =20ex 1ex 20ex 1ex]
{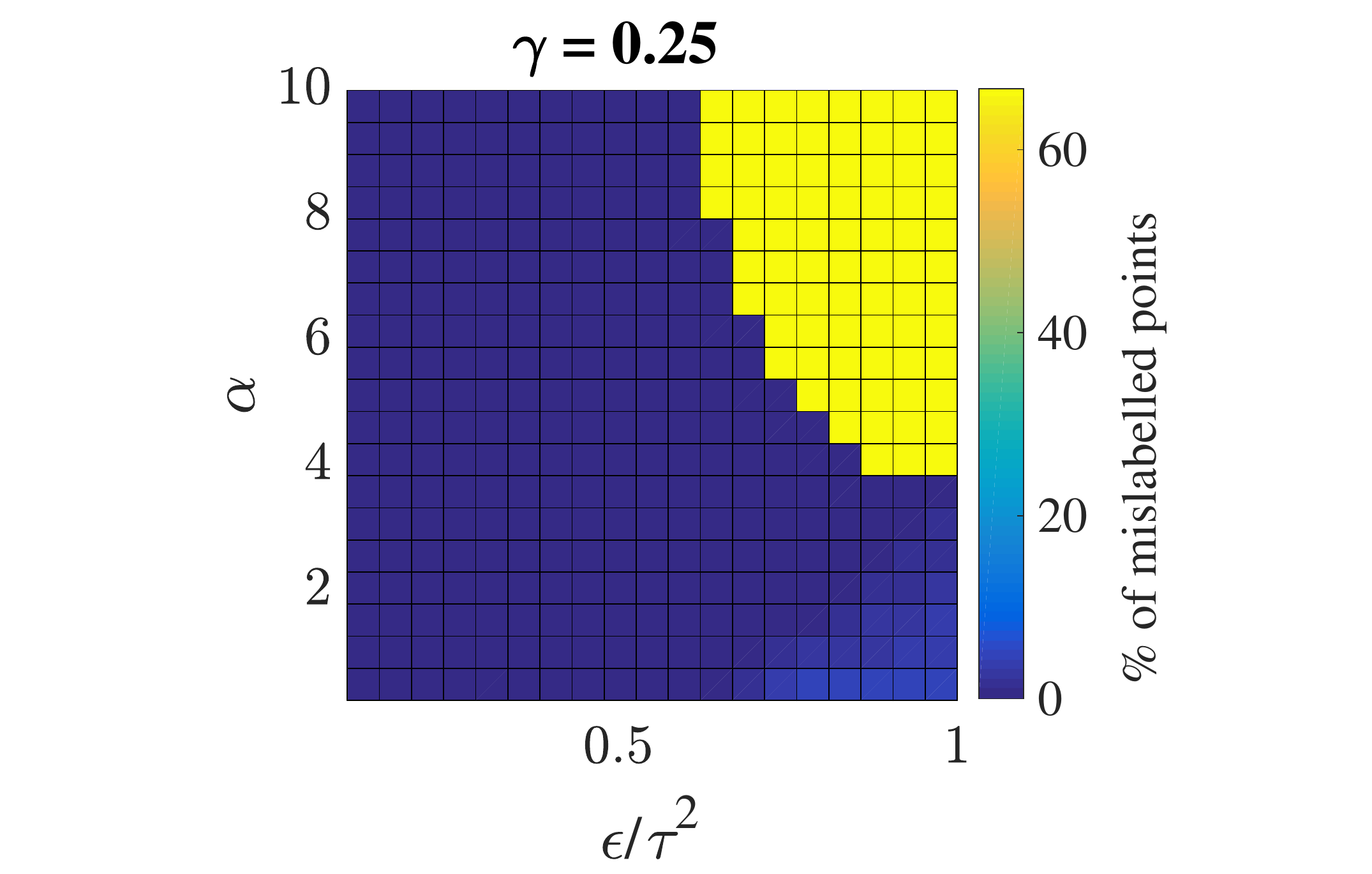}
  \end{subfigure}
    \begin{subfigure}[b]{.31 \textwidth}
    \includegraphics[width=1 \textwidth, clip =true, trim =20ex 1ex 20ex 1ex]
    {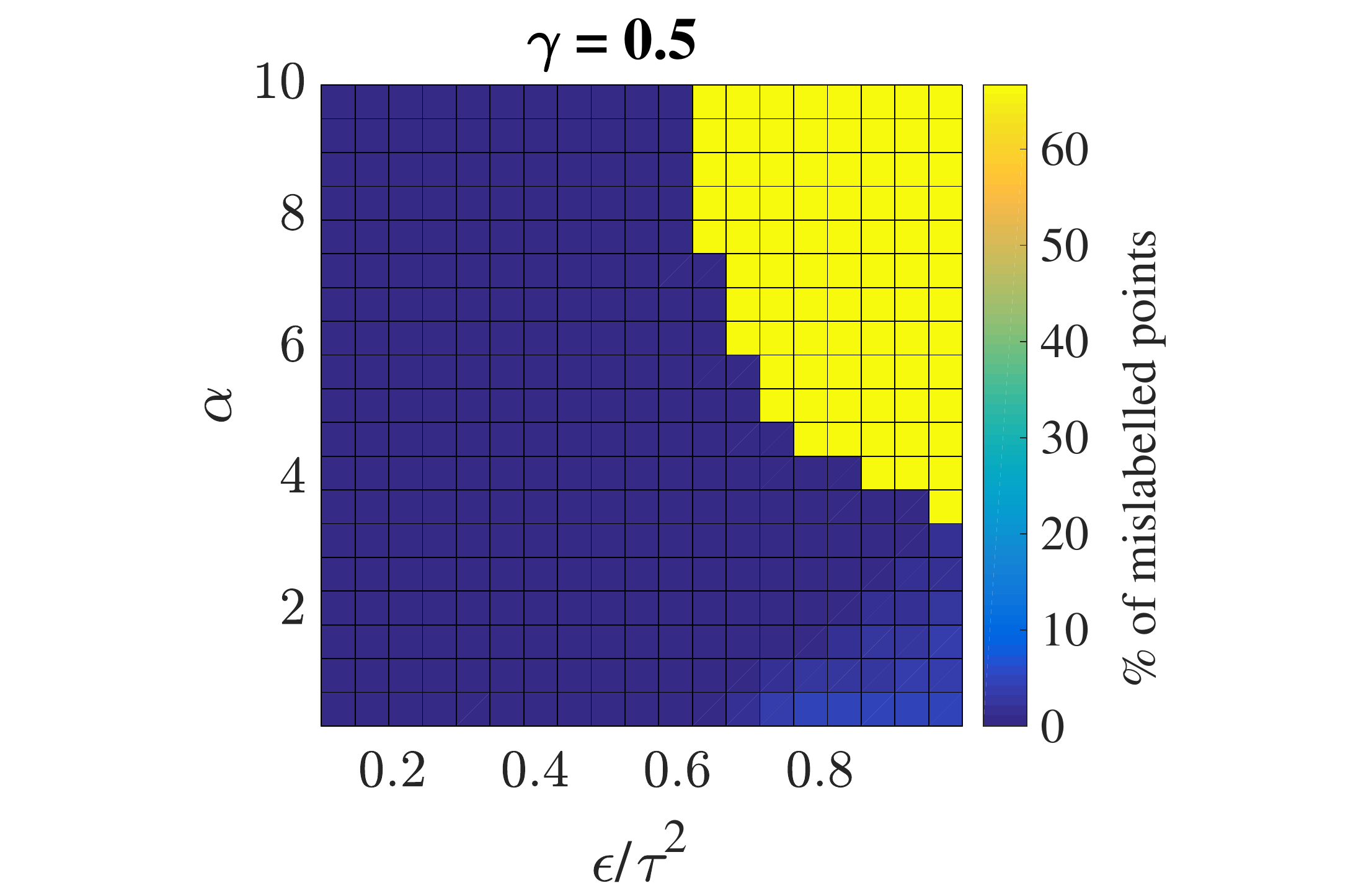}
  \end{subfigure}
\begin{subfigure}[b]{.32 \textwidth}
  \includegraphics[width=1 \textwidth, clip =true, trim =20ex 1ex 20ex 1ex]
  {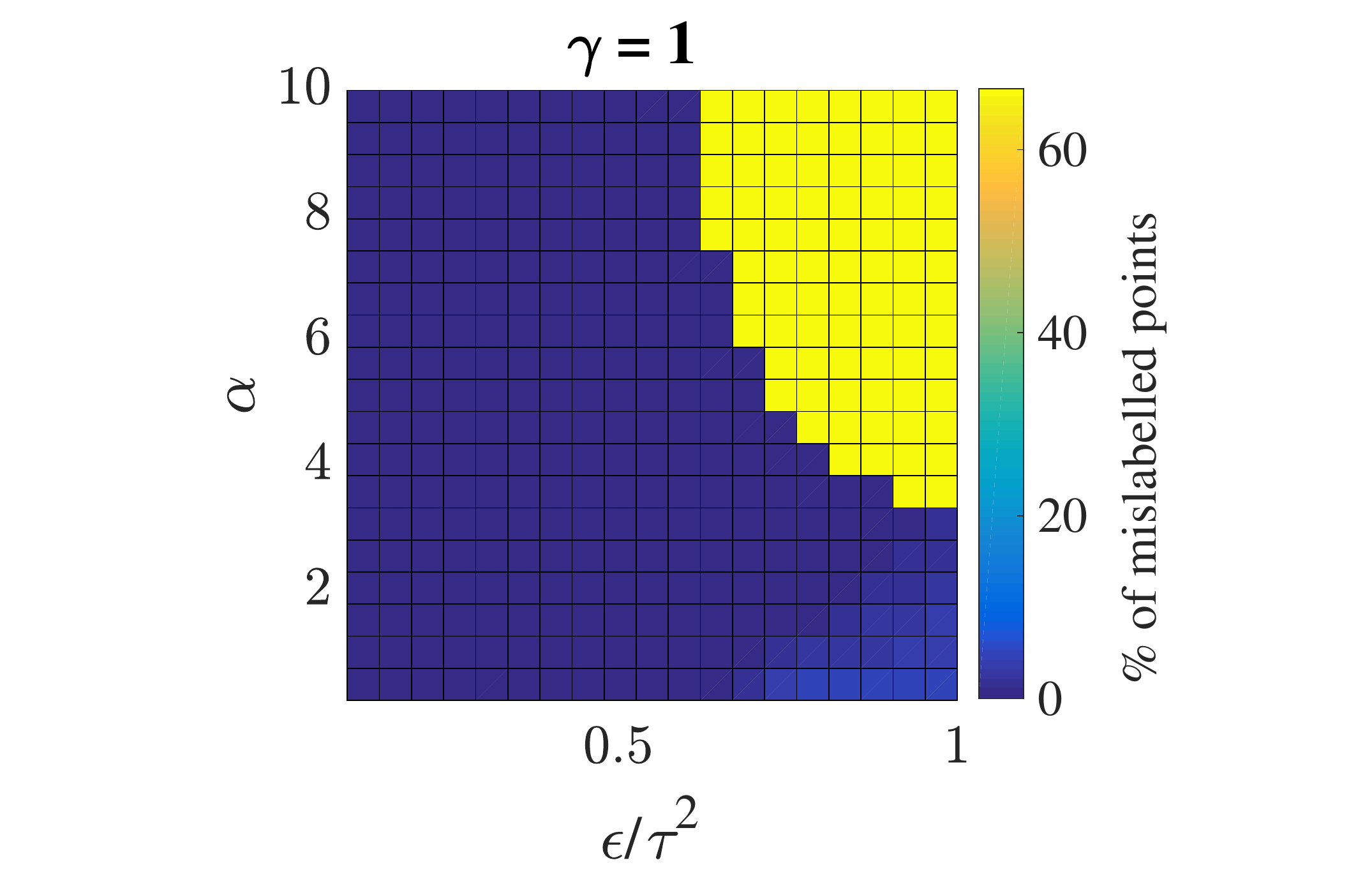}
\end{subfigure}\\
  \begin{subfigure}[b]{.32 \textwidth}
    \includegraphics[width=1 \textwidth, clip =true, trim =20ex 1ex 20ex 1ex]
{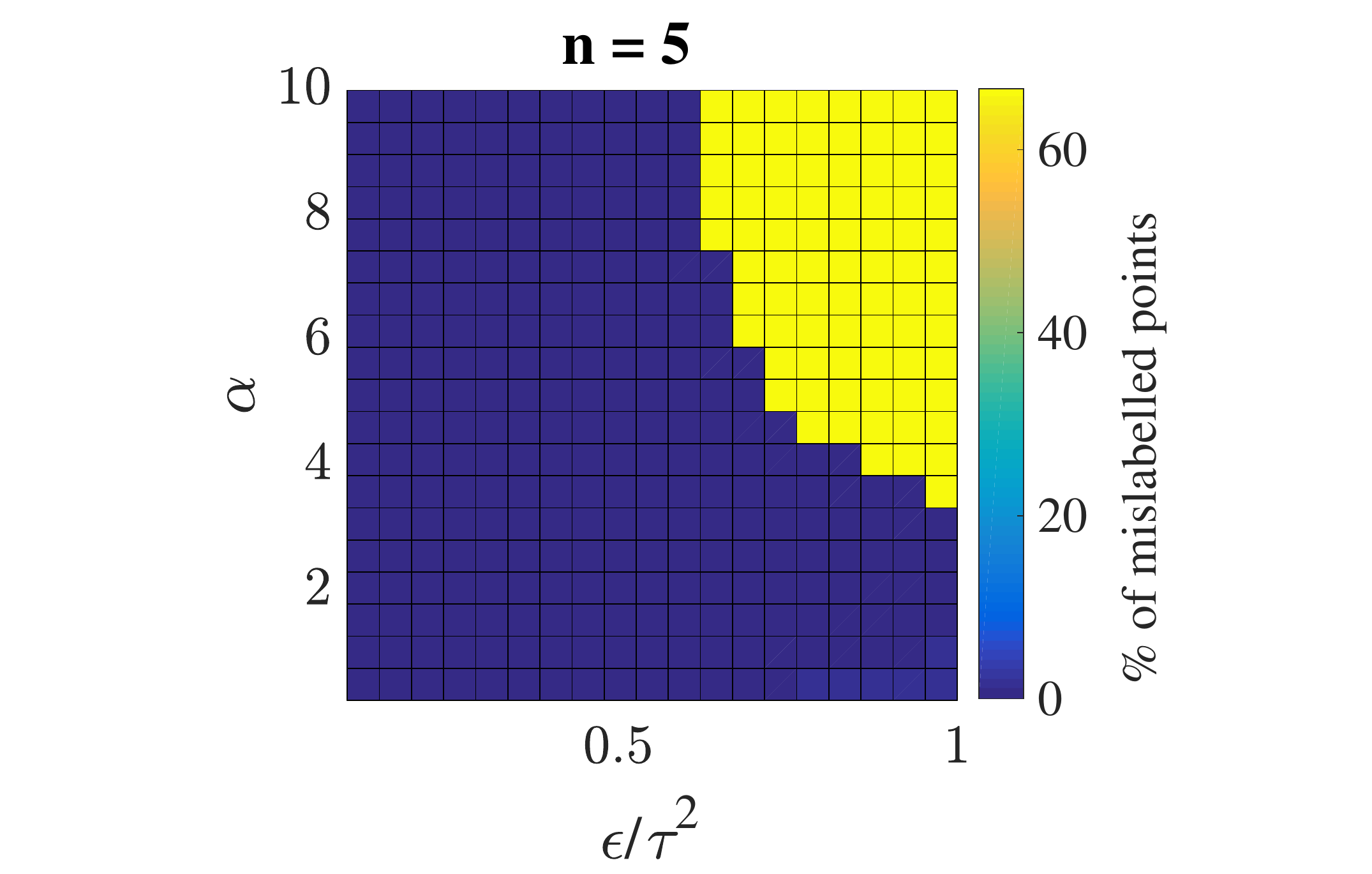}
  \end{subfigure}
    \begin{subfigure}[b]{.29 \textwidth}
    \includegraphics[width=1 \textwidth, clip =true, trim =20ex 1ex 20ex 1ex]
    {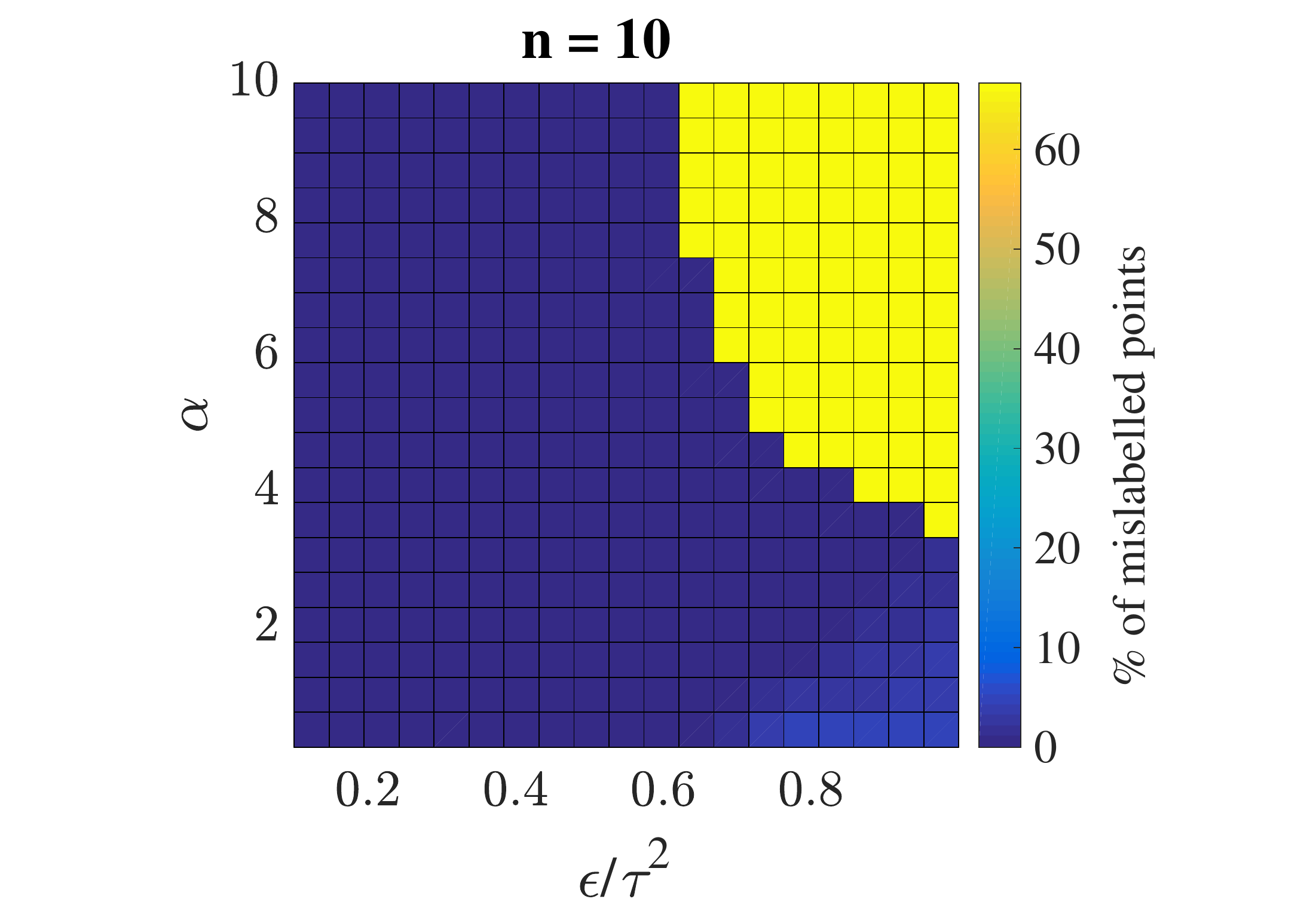}
  \end{subfigure}
\begin{subfigure}[b]{.32 \textwidth}
  \includegraphics[width=1 \textwidth, clip =true, trim =20ex 1ex 20ex 1ex]
  {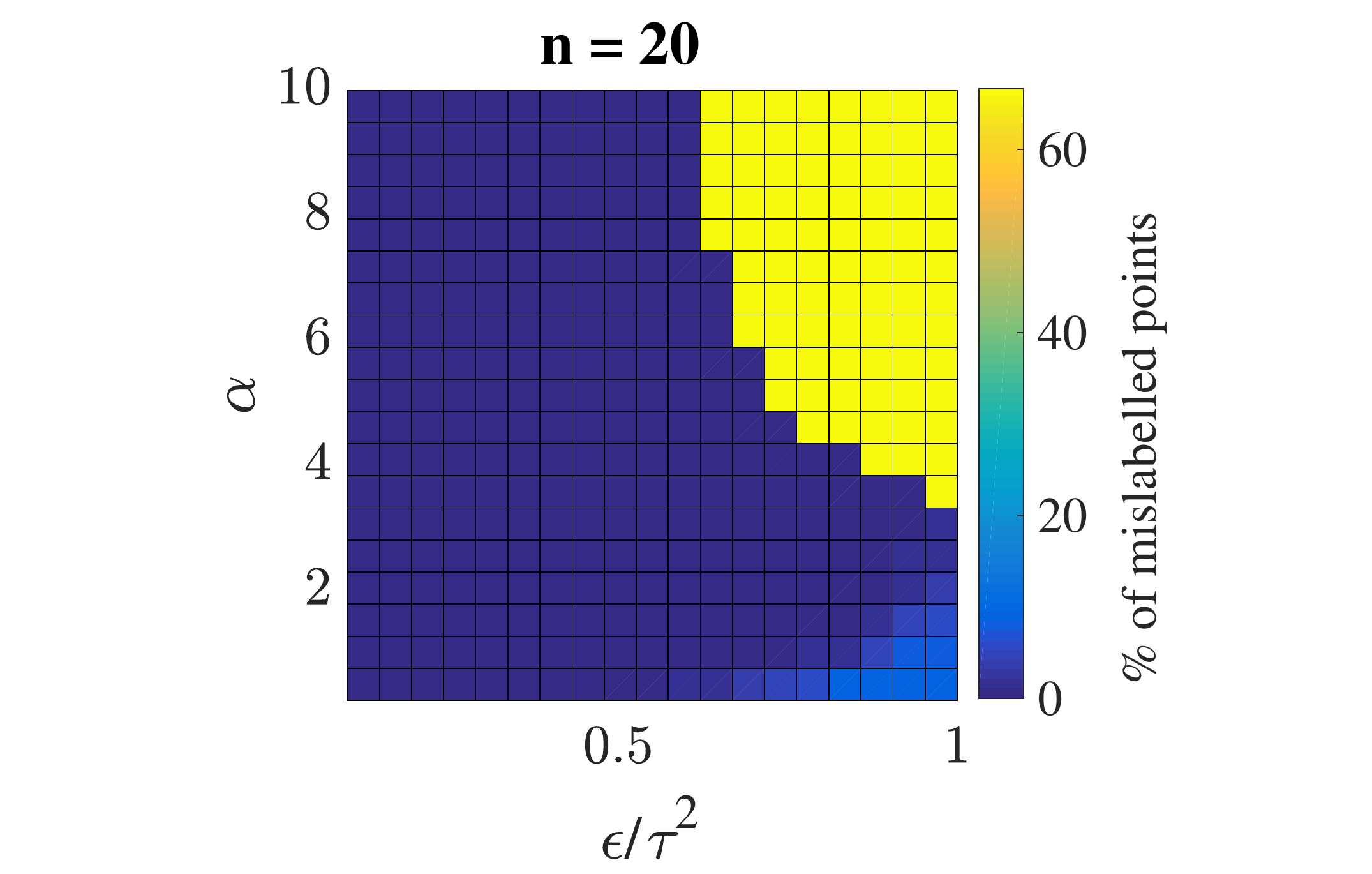}
  \end{subfigure}
  \caption{Prediction error of the one-hot classifier using perfect measurements
    for different values of $\gamma$ while fixing $n=10$ (top row) and values of $n$ while fixing $\gamma=0.5$ (bottom row).
    A clear phase transition is observed as $\eps/\tau^2$ increases for each value of $\alpha$. For large values of $\eps/\tau^2$ and small $\alpha$ we observe a
    slight increase in prediction error as $n$ increases (compare the figures in
 the bottom row when $\alpha \approx 1$)..}
  \label{num-exp-one-hot-consistency-with-different-gamma-n}
\end{figure}

\begin{figure}[htp]
  \centering
   \begin{subfigure}[b]{.32 \textwidth}
    \includegraphics[width=1 \textwidth, clip =true, trim =15ex 1ex 15ex 1ex]
    {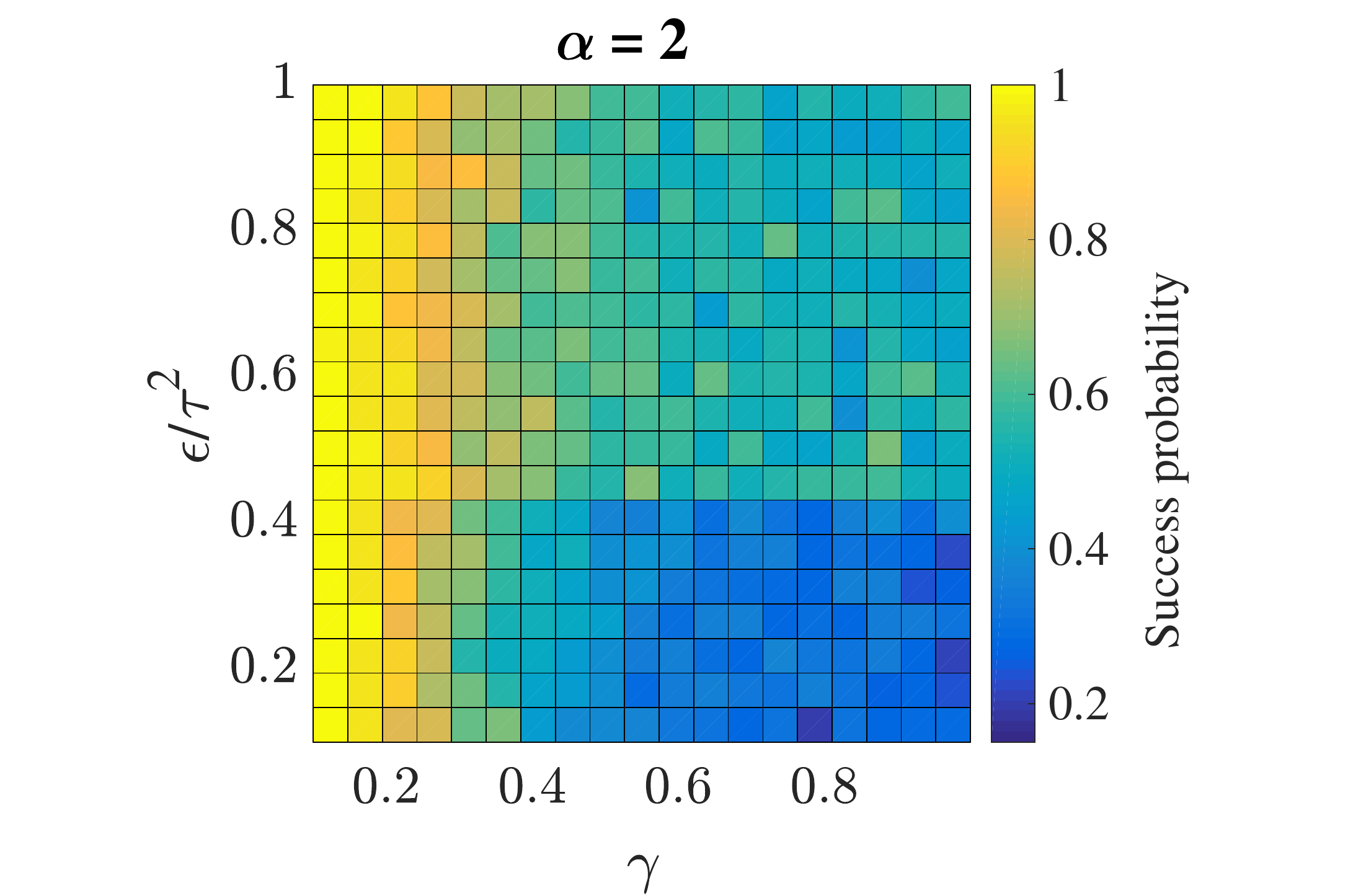}
    \caption{}
  \end{subfigure}
    \begin{subfigure}[b]{.32 \textwidth}
    \includegraphics[width=1 \textwidth, clip =true, trim =15ex 1ex 15ex 1ex]
    {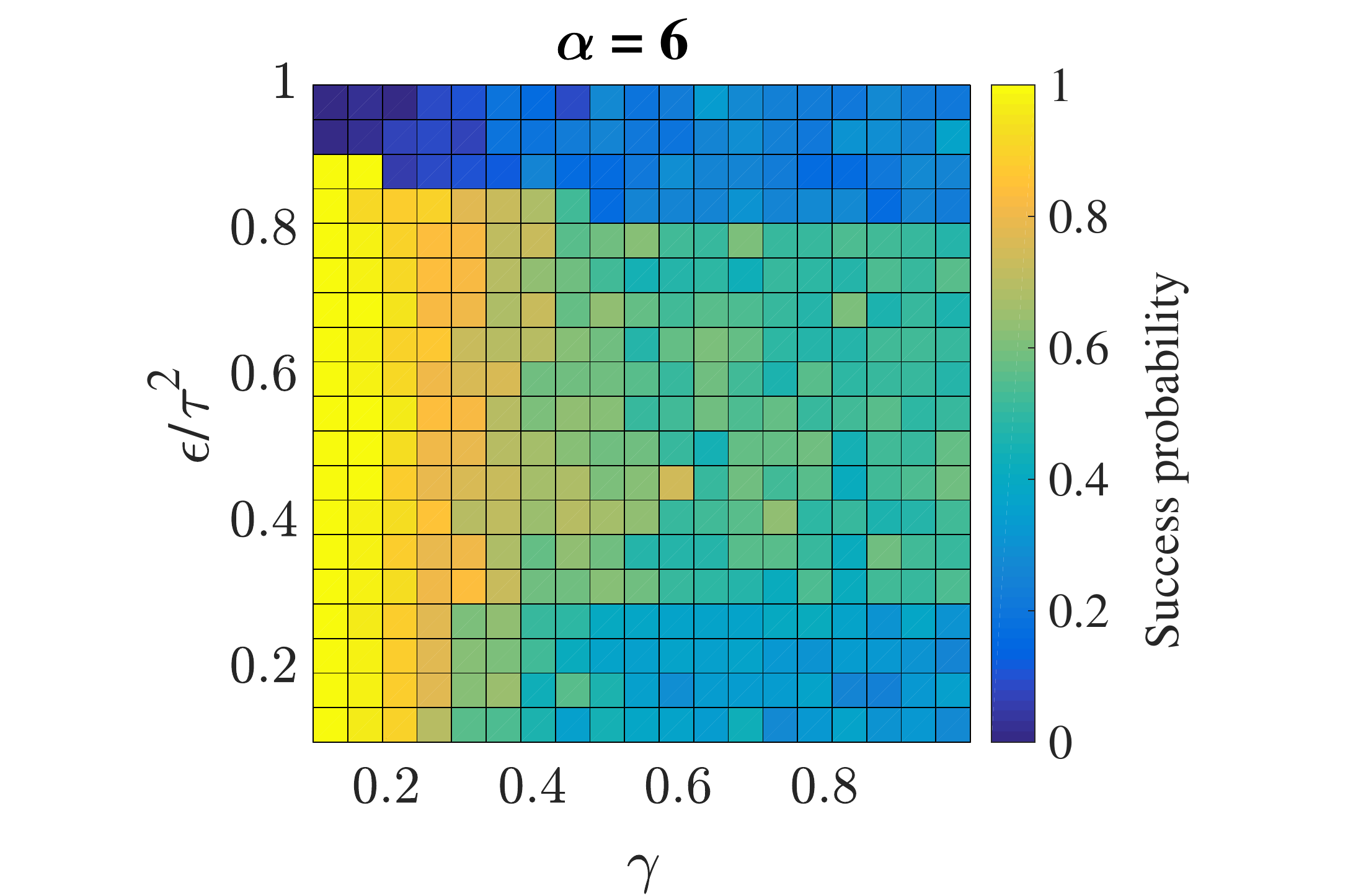}
    \caption{}
  \end{subfigure}
\begin{subfigure}[b]{.32 \textwidth}
  \includegraphics[width=1 \textwidth, clip =true, trim =20ex 1ex 20ex 1ex]
  {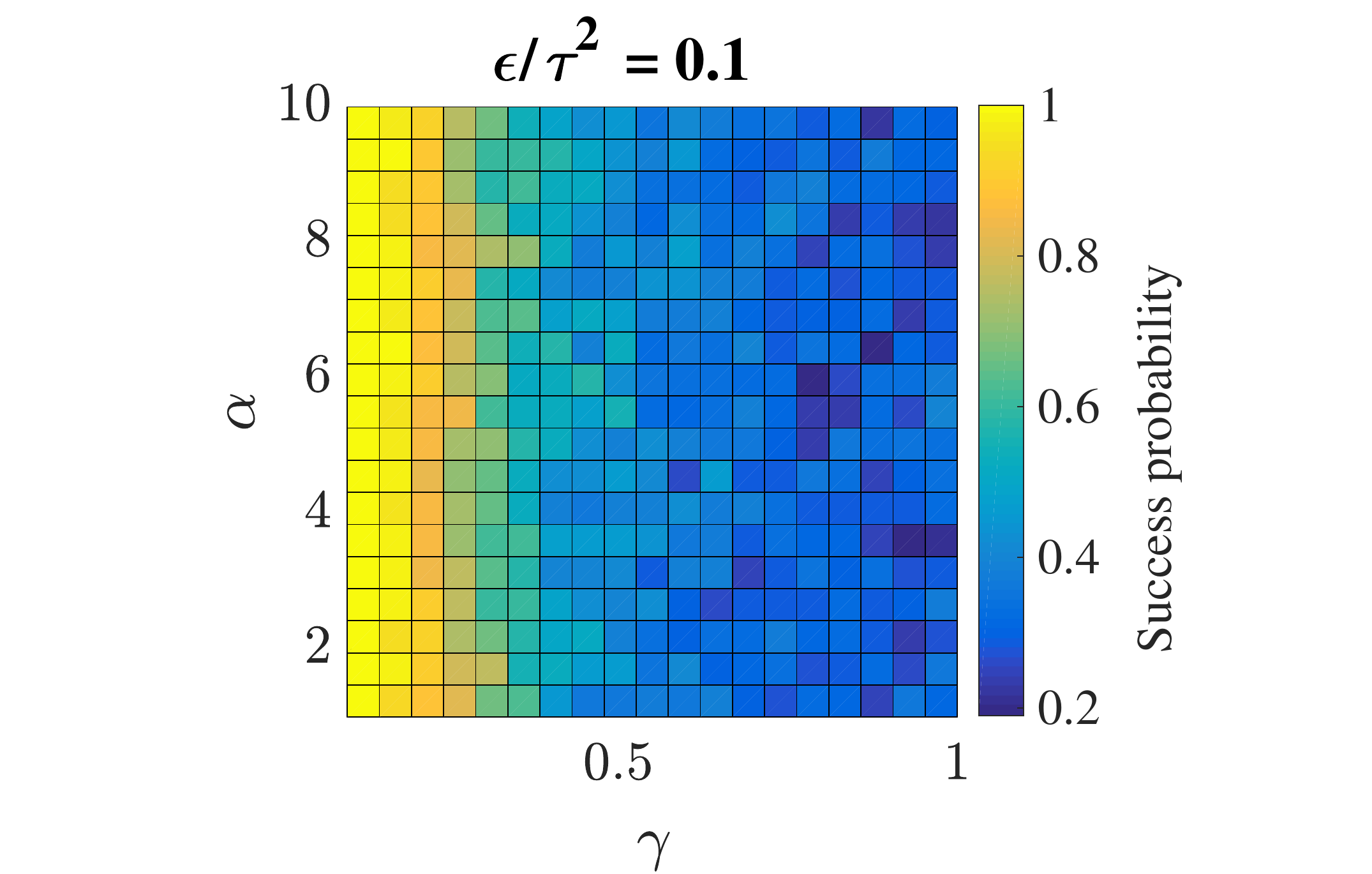}
  \caption{}
  \end{subfigure}
  \caption{Heat map of success probability of one-hot as a function of
    $\gamma$, $\alpha$ and $\eps/\tau^2$. (a) and (b) show the success
    probability as a function of $\eps/\tau^2$ for $\alpha = 2$ and $\alpha = 6$.
    In the latter case we see a sudden drop in the success probability
    as $\eps/\tau^2$ increases, even for very small values of $\gamma$.
    (c) shows the success probability as a function of $\alpha$ and $\gamma$
    for fixed $\eps/\tau^2 = 0.1$. Here the success probability is mostly
  controlled by $\gamma$ and does not appear to depend on $\alpha$. }
  \label{num-exp-one-hot-prob-success}
\end{figure}

\subsubsection{Effect Of Number Of Observations In Each Cluster}

For our final set of observations we consider varying the number of observations
in each cluster. We use the same dataset as above with $N = 150$ in three
distinct clusters but this time we vary the number of observed labels
in each cluster. In Figure~\ref{num-exp-one-hot-accuracy-multiple-observations-in-clusters}
(a) and (b) we compare
the accuracy of one-hot with a single observation in all clusters and
with three observations in all clusters respectively. We only consider
noiseless observations in this case. We see a small
perturbation in the error phase transition for larger number of
observations but overall the 
accuracy appears to depend weakly on the number of observations so long as
we have the same number of observations in all clusters. Figure~\ref{num-exp-one-hot-accuracy-multiple-observations-in-clusters}(c)  shows the same
experiment as above except that here we took two observations in
one of the clusters and a single observation in the other two. Figure~\ref{num-exp-one-hot-success-multiple-observations-in-clusters}(d)
shows a similar calculation with three observations in one cluster and a
single observation in the rest. In comparison to Figure~\ref{num-exp-one-hot-accuracy-multiple-observations-in-clusters}(a) and (b) we now see a major change in the
accuracy of one-hot, namely that the jump in error now occurs for smaller
values of $\eps/\tau^2$ indicating that the majority label of one of the
clusters is propagated to the rest of the dataset when we have an unbalanced number
of observations within the clusters. 

In Figure~\ref{num-exp-one-hot-success-multiple-observations-in-clusters} we
show the success probability of one-hot when an unbalanced number of labels
are observed in the clusters: three labels are observed in one cluster
and single labels are seen in the rest. As in previous success probability
calculations we randomly perturbed the observed labels in 100 trials and
approximated the success probability of one-hot as a function of $\eps/\tau^2$
and $\gamma$. As before, we see a
sharp transition in the success probability as $\eps/\tau^2$ grows but below
a certain critical value of $\eps/\tau^2$ the success probability appears to only
depend on $\gamma$. We also note that this critical value of $\eps/\tau^2$
appears to shift towards smaller values for larger $\alpha$.

\as{These experiments reveal an interesting and complicated feature of the probit
  and one-hot minimizers in connection to the balancing of labelled points in clusters
  that warrant future analysis.
  It appears that having a balanced number of labels in different
  clusters allows for a larger range of acceptable $\eps, \tau^2$ parameters; note that
  the blue regions are larger in Figure~\ref{num-exp-one-hot-accuracy-multiple-observations-in-clusters}(a,b)
  when all clusters have the same number of labelled points 
  compared to Figure~\ref{num-exp-one-hot-accuracy-multiple-observations-in-clusters}(c, d)
  where one cluster has more labelled points. This suggests that balancing of labelled
  points in practical applications might lead to better accuracy albeit at a high
  computational cost.
  The sensitivity to the balancing of labelled points further highlights the importance
  hierarchical Bayesian methods can tune the $\tau, \alpha$ parameters
  automatically.
}

\begin{figure}[htp]
  \centering
  \begin{subfigure}[b]{.42 \textwidth}
  \includegraphics[width=1 \textwidth, clip =true, trim =20ex 1ex 20ex 1ex]
  {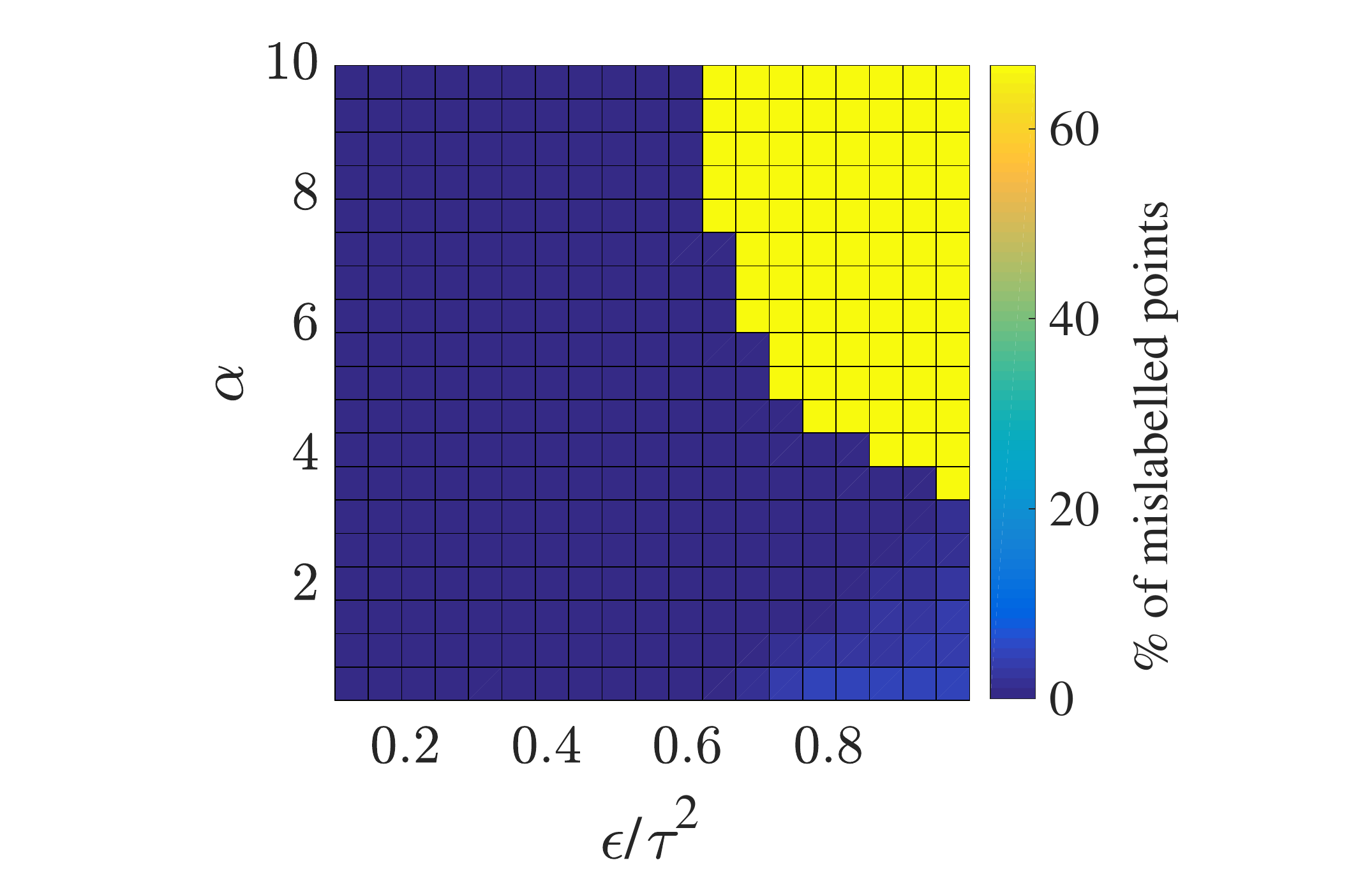}
  \caption{}
\end{subfigure}
\begin{subfigure}[b]{.42 \textwidth}
  \includegraphics[width=1 \textwidth, clip =true, trim =20ex 1ex 20ex 1ex]
  {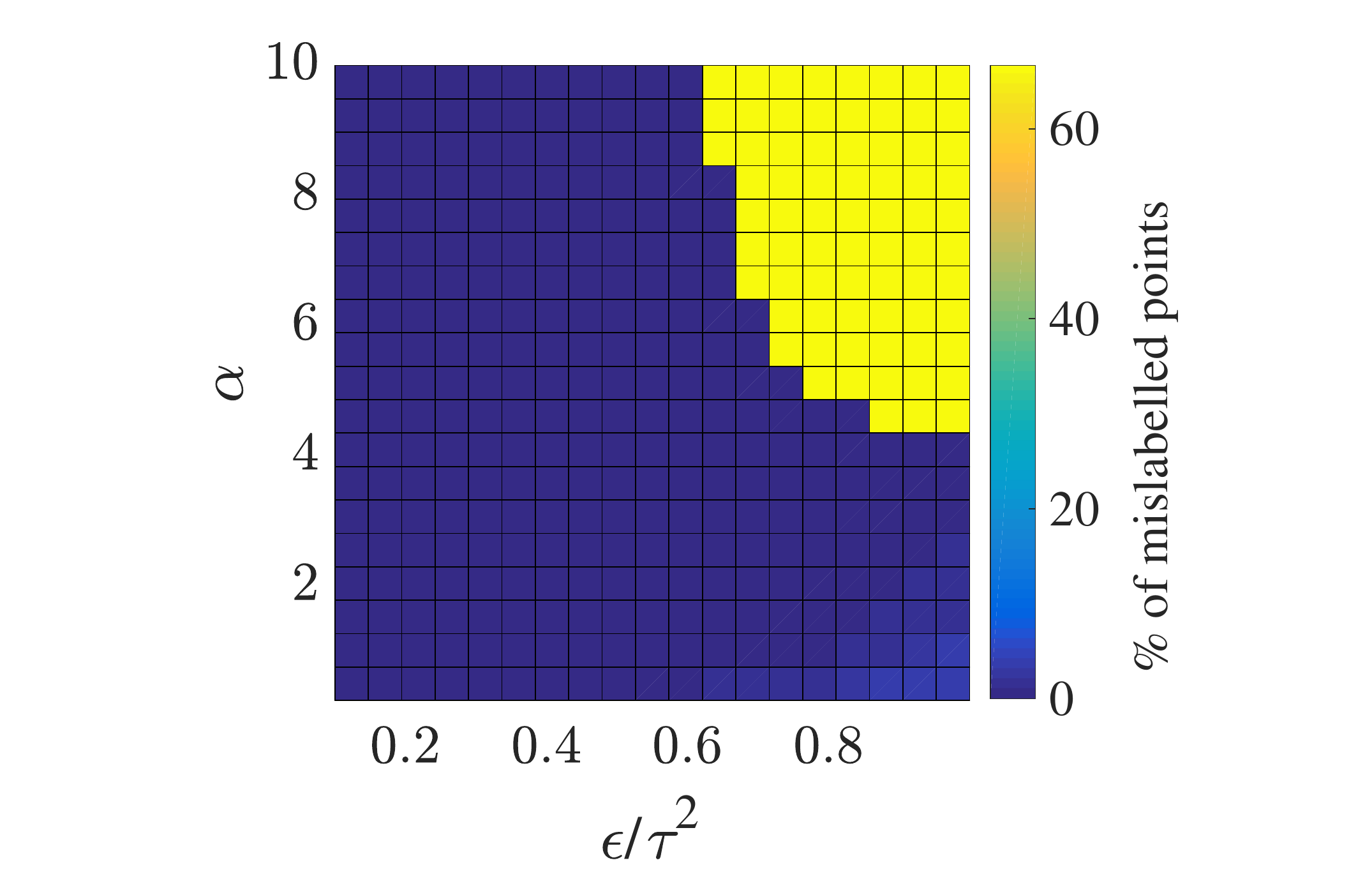}
  \caption{}
  \end{subfigure}\\
   \begin{subfigure}[b]{.42 \textwidth}
    \includegraphics[width=1 \textwidth, clip =true, trim =20ex 1ex 20ex 1ex]
    {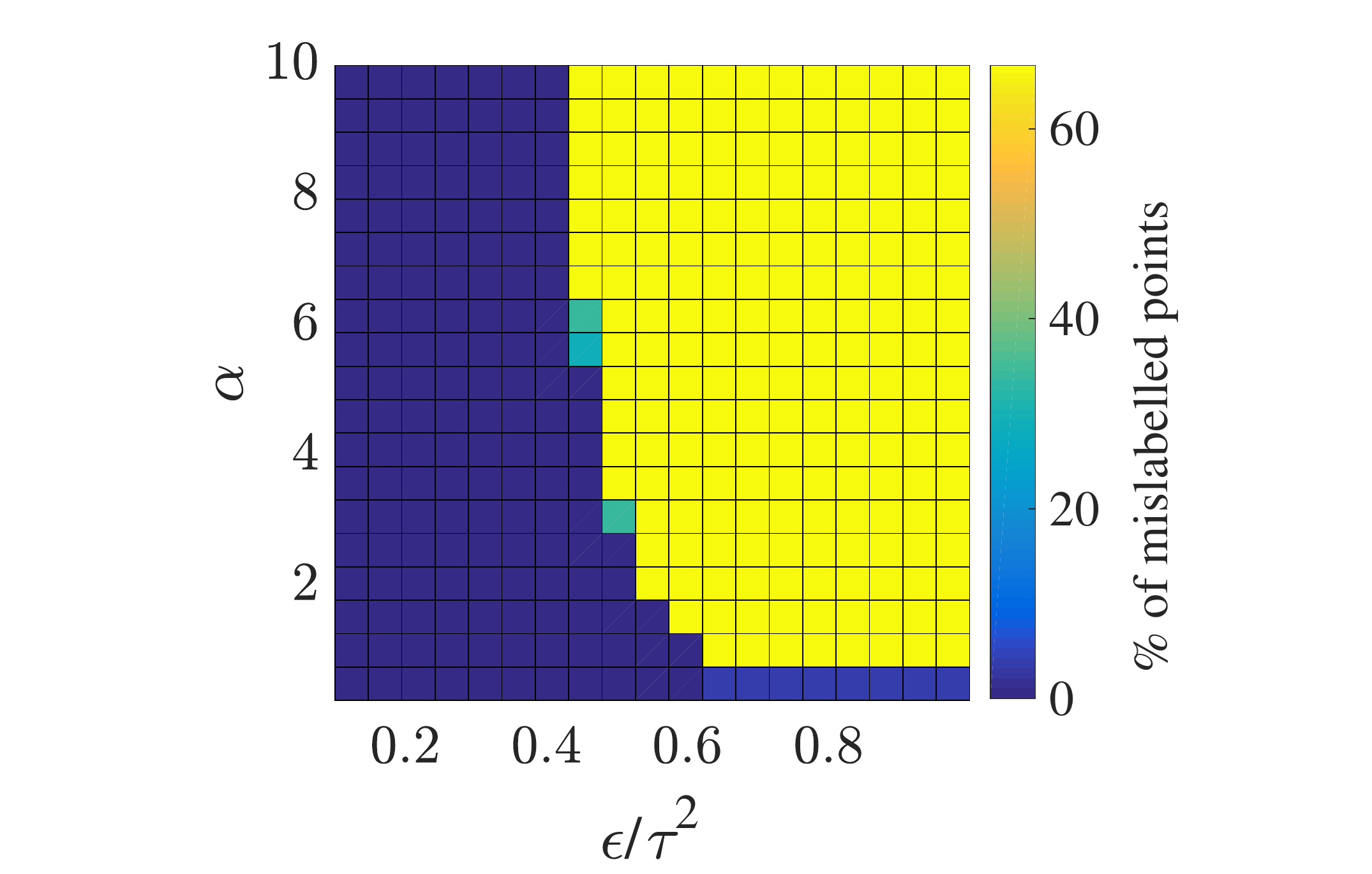}
    \caption{}
  \end{subfigure}
    \begin{subfigure}[b]{.42 \textwidth}
    \includegraphics[width=1 \textwidth, clip =true, trim =20ex 1ex 20ex 1ex]
    {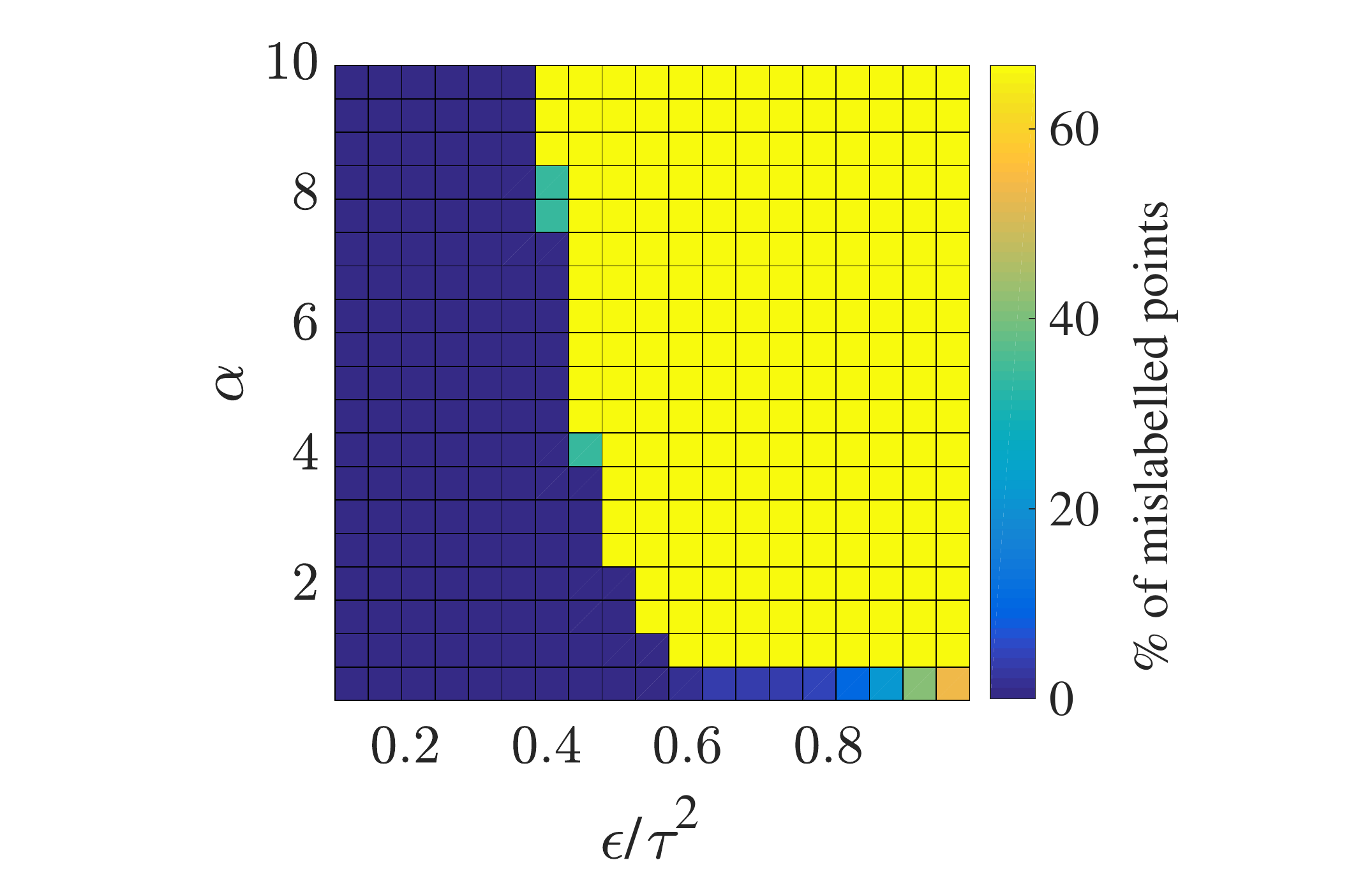}
    \caption{}
  \end{subfigure}
  \caption{Effect of the number of observed labels
    on accuracy of   one-hot for $n=10$ and $\gamma =0.5$
    as a function of $\alpha$ and $\eps/\tau^2$. 
    (a) A single label is observed in each of the three clusters. (b)
    Three labels are observed in each cluster. (c) Two observations in
    one cluster and a single observation in the other two. (d) Three observations in
  one cluster and single observations in the other two.}
  \label{num-exp-one-hot-accuracy-multiple-observations-in-clusters}
\end{figure}

\begin{figure}[htp]
  \centering
  \begin{subfigure}[b]{.42 \textwidth}
  \includegraphics[width=1 \textwidth, clip =true, trim =15ex 0ex 15ex 0ex]
  {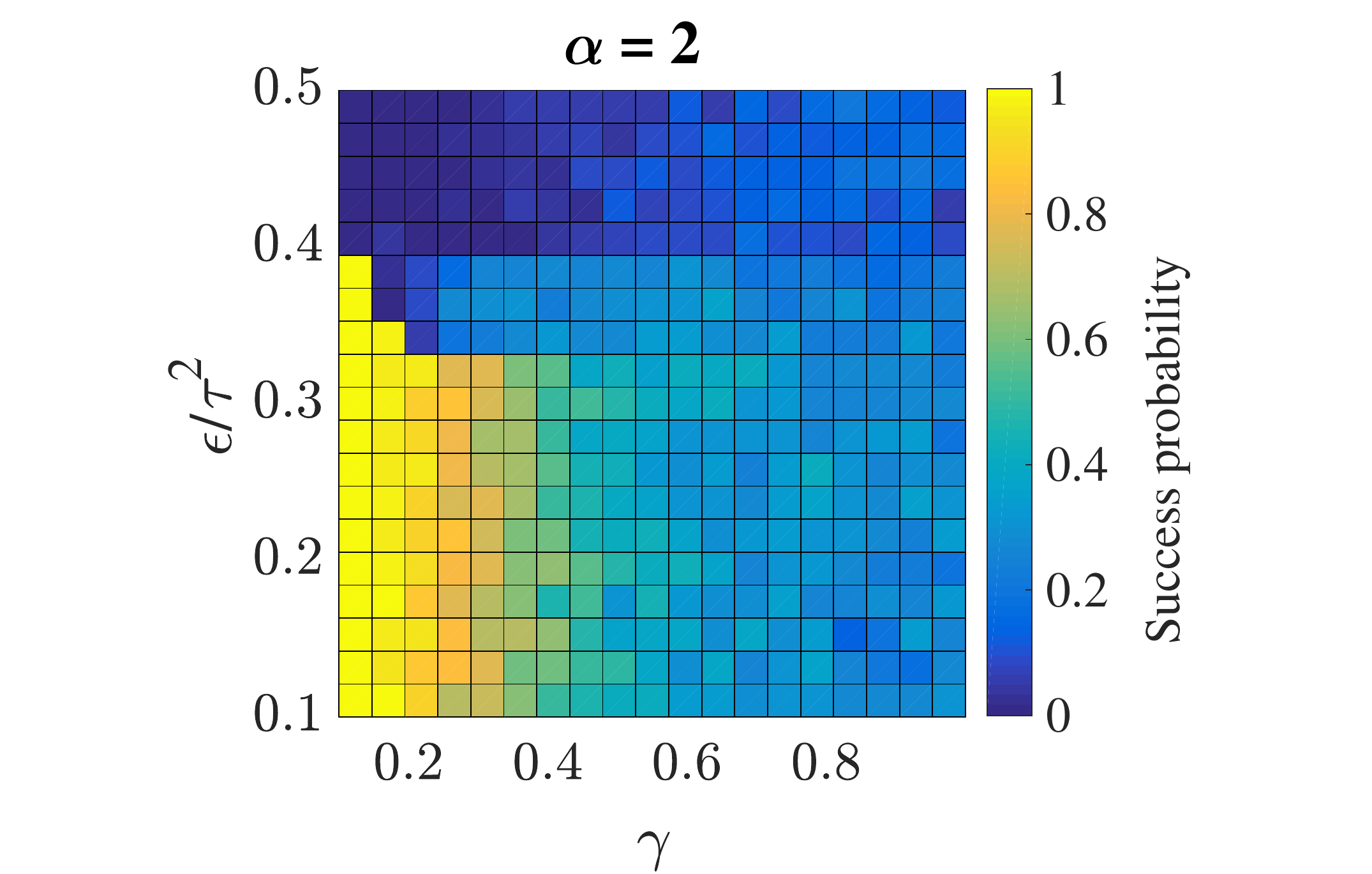}
  \caption{}
\end{subfigure} \hspace{2ex}
\begin{subfigure}[b]{.41 \textwidth}
  \includegraphics[width=1 \textwidth, clip =true, trim =15ex 0ex 15ex 0ex]
  {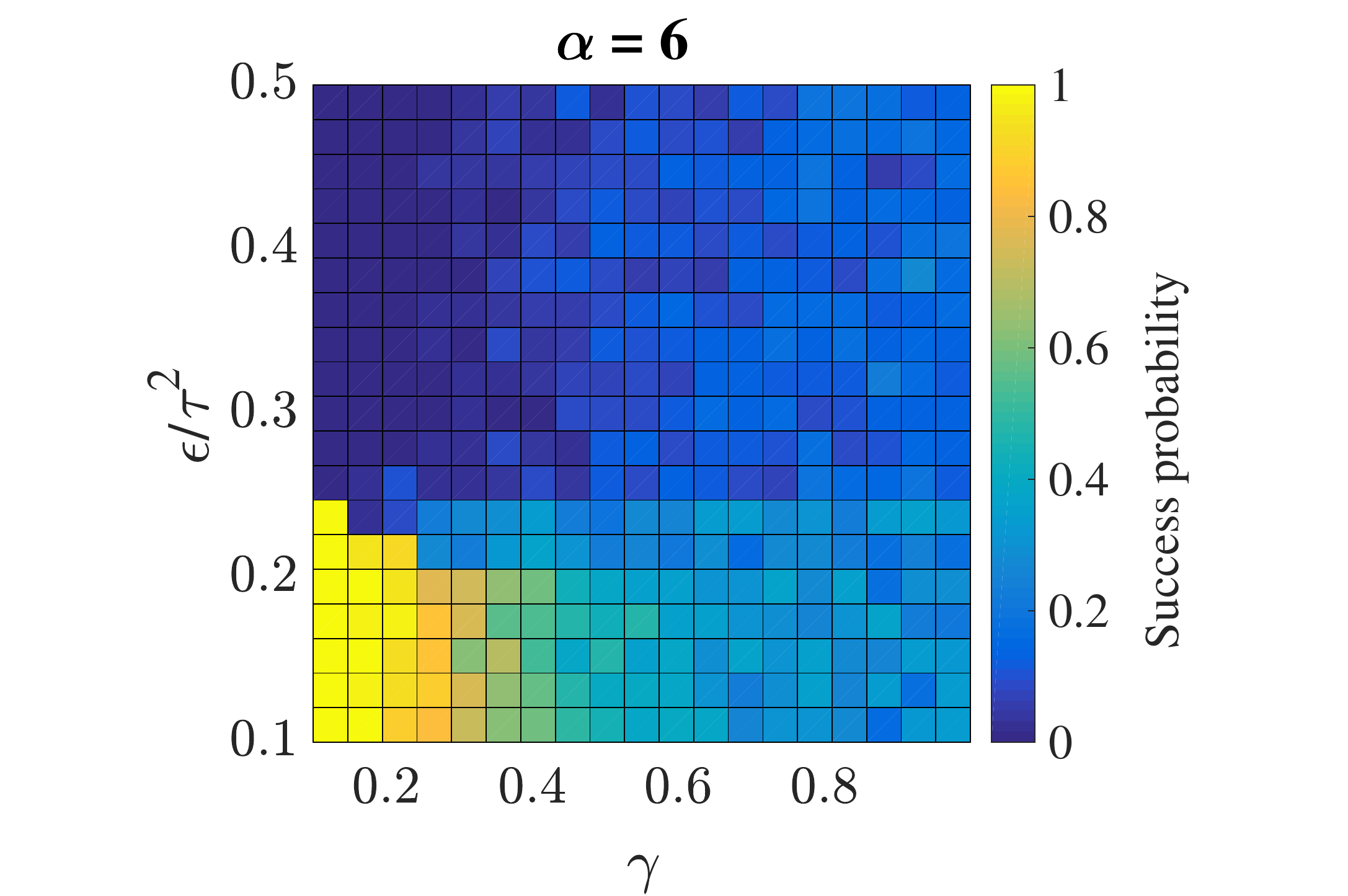}
  \caption{}
  \end{subfigure}
  \caption{Success probability of  one-hot with three observations
    in one cluster and single observations in the other two clusters for $n=10$,
    $\alpha = 2$ and $ 6$, and 
    as a function of  $\eps/\tau^2$ and $\gamma$.
    As $\eps/\tau^2$ grows there appears to be a phase transition where
    the success probability suddenly drops.}
  \label{num-exp-one-hot-success-multiple-observations-in-clusters}
\end{figure}

\section{Conclusions}
\label{sec:conc}

We have studied the consistency of the probit and one-hot methods for SSL, demonstrating
that the combination of ideas from unsupervised learning and supervised learning can
lead to consistent labelling of large data sets, given only a few labels. 
Our theory and numerical results demonstrate that with careful choice of the function
of the graph Laplacian appearing in the quadratic penalty, namely the 
 parameters $\alpha$, $\tau$ and $\eps$, correct labelling of the dataset can be achieved asymptotically. 

However, our theory and numerics also indicate that the choice of these parameters is crucial 
and can lead to failure of the methods. For example, we observed that when $\eps/\tau^2$ is too large then 
the methods have a tendency to propagate the majority label rather than matching the observed labels.
 This sensitivity strongly suggests the importance of
hierarchical Bayesian techniques which can determine such choices in a data-driven
fashion. \as{Proving that hierarchical methods can learn the scaling
of $\tau$ in terms of $\epsilon$ that emerges from our analysis would be of interest.} 

There are a number of directions in which this work can be taken, including 
the study of Bayesian posterior consistency, which we
undertake in \cite{BHLMS19} for the harmonic-function based approach to
graph-based SSL, and the  
study of the limit of large unlabelled data sets in
\cite{HHOS2,HHOS1}.
Another interesting question is a detailed analysis of the majority label propagation phenomenon that 
is observed when $\eps/\tau^2$ is too large. Furthermore, throughout this paper we mainly focused on 
the setting where the ground truth function $u^\dagger$ is consistent with the clustering and assigns the 
same label to all points within the same cluster. Then the question arises, how do probit and one-hot behave 
when mislabelled points are present in the data or $u^\dagger$ assigns more than one class to a cluster.
\as{It would also be interesting to study different clustering assumptions,
such as those arising from stochastic block models \cite{lei2015consistency}.}

%
\vspace{0.1in}

\noindent{\bf Acknowledgements} The authors are grateful to
Nicol{\'a}s Garc{\'i}a-Trillos, Mark Girolami and Omiros Papaspiliopoulos for helpful discussions about the probit
methodology and spectral clustering.
FH is partially supported by Caltech's von K{\'a}rm{\'a}n postdoctoral instructorship.
  BH is supported in part  by an NSERC
  PDF fellowship.
  AMS is grateful to AFOSR (grant FA9550-17-1-0185) 
  and NSF (grant DMS 18189770) for financial support.

\bibliographystyle{abbrv}
\bibliography{discrete_probit_analysis_references}

\appendix

\section{Spectral Analysis Of Covariance Operators}
\label{appendix:perturbation-theory}
Here we study the spectrum of the covariance operators $C_{\tau, 0}$ and
$C_{\tau, \eps}$ and prove  Propositions~\ref{perfectly-separated-c-j}
and \ref{geometry-of-covaraince-functions}. Throughout this
section we use the notation of subsection~\ref{sec:covariance-perturbation}.
We start with some preliminary results regarding the spectrum of $C_{\tau, 0}$.
Recall that Assumption~\ref{assumptions-on-G-0} ensures that $G_0$
consists of $K < N$ disconnected components $\tG_k$ and that the graph
Laplacian restricted to each $\tG_k$ has a one-dimensional null-space
consisting of constant functions on $\tG_k.$

\begin{lemma}\label{C-is-strict-positive-definite}
Suppose  $G_0 =\{ X, W\}$ is a proximity graph satisfying
Assumption~\ref{assumptions-on-G-0}, and let $L_0$ be a graph Laplacian operator on $G_0$ 
as in \eqref{graph-Laplacian} with $p=q$. Then  $L_0$ is positive semi-definite
  for any choice of $p \in \mbb R$ and the matrix 
  $C_{\tau, 0}^{-1}$ as in \eqref{C-tau-zero-definition} is symmetric and
  strictly positive definite for any values of $\tau^2 >0$ and $\alpha > 0$.
\end{lemma}

\begin{proof}
  Assume $Z = \{ \tZ_1, \tZ_2, \cdots, \tZ_K\}$
  where $\tZ_k$ are the collection of nodes in the $k^{th}$ component $\tG_k$  of $G_0$ and
  let $\tL_k$ denote the graph Laplacian operator on $\tG_k$ 
  constructed as in \eqref{graph-Laplacian} by replacing $D_0$ and $W_0$
  with $\tD_k$ and $\tW_k$
  the submatrices corresponding to  $\tG_k$.

  Then by \cite[Lem.~1.7]{chung1997spectral} the matrices $\tL_k$ are positive
  semi-definite with eigenvalues $0 = \tilde \sigma_{1,k} < \tilde \sigma_{2,k} \le \tilde  \sigma_{3,k} \le \cdots
  \le \tilde \sigma_{N_k, k}$ where $N_k = |\tZ_k|$. It follows that
  $L_0 = \text{diag}(\tL_1, \tL_2, \cdots, \tL_K)$ and so $L_0$ is also
  positive semi-definite with eigenvalues $0 = \sigma_{1,0} = \sigma_{2,0}= \cdots =
  \sigma_{K,0} < \sigma_{K+1,0} \le \sigma_{K+2,0} \le \cdots \le \sigma_{N,0}$.

  Now suppose $\alpha =1$. Then it is straightforward to check that $C_{\tau,0}^{-1}$
  has eigenvalues $\lambda_{k,0}=(\tau^{-2}\sigma_{k,0}+1)$, so that
 $1 = \lambda_{1,0} = \lambda_{2,0} = \lambda_{K,0} < \lambda_{K+1,0}
  \le \lambda_{K+2,0} \le \cdots \le \lambda_{N,0}$ and so it is strictly positive definite.
  The case with $\alpha >0$ then follows since the eigenvalues are
simply $\lambda_{k,0} = (\tau^{-2}\sigma_{k,0} +1)^{\alpha}$.
\end{proof}

\color{black}

\begin{lemma}\label{c-j-eigen-expansion}
  Suppose $\tau^2, \alpha >0$ and let $\{ \lambda_{k,0}\}_{k=1}^N$ and $\{ \pmb \phi_{k,0}\}_{k=1}^N$
  be the eigenvalues and eigenvectors \as{of $C_{\tau,0}^{-1}$ respectively.} Then
  \begin{equation}\label{c-j-eigen-expansion-equation}
    \mbf c_{j,0} = \sum_{k=1}^N \frac{1}{\lambda_{k,0}} (\pmb \phi_{k,0})_j \pmb \phi_{k,0}.
  \end{equation}
\end{lemma}

\begin{proof}
  Since $C_{\tau,0}^{-1}$ is self-adjoint and positive it has an eigendecomposition
  \begin{equation*}
    C_{\tau,0}^{-1} = Q \Lambda Q^{-1},
  \end{equation*}
  where $\Lambda$ is a diagonal matrix with diagonal elements
  $\lambda_{k,0}$
  and $Q = [ \pmb \phi_{1,0}, \cdots, \pmb \phi_{N,0}]$ is a unitary matrix
  with columns $\pmb \phi_{k,0}$.
  Substituting this expression in the identity $\mbf c_{j,0} = C_{\tau,0} \mbf e_j$
  and noting that $Q^{-1} = Q^T$ then
  gives
  \begin{equation*}
    \mbf c_{j,0} = Q \Lambda^{-1} Q^{T} \mbf e_j =
    \begin{bmatrix}
      \frac{1}{\lambda_{1,0}} \pmb \phi_{1,0}, \frac{1}{\lambda_{2,0}} \pmb \phi_{2,0},
      \cdots, \frac{1}{\lambda_{N,0}}
      \pmb \phi_{N,0}
    \end{bmatrix}
    \begin{bmatrix}
      (\pmb \phi_{1,0})_j \\
      (\pmb \phi_{2,0})_j \\
      \vdots \\
      (\pmb \phi_{N,0})_j
    \end{bmatrix},
  \end{equation*}
  which concludes the proof.
\end{proof}

  \subsection{Proof of Proposition~\ref{perfectly-separated-c-j}: Disconnected clusters}\label{sec:proof-of-perfectly-sep-spectral-theory}
\begin{proposition}[Disconnected clusters]\label{spectrum-perfectly-sep-clusters}
Let $G_0$ satisfy Assumption~\ref{assumptions-on-G-0} with
$K < N$ components $\tG_k$. Let $\bar{\pmb \chi}_k$ be  as in \eqref{chi-k-definition}.
Then  for $k = 1, \cdots, K$
\begin{equation*}
  \lambda_{k,0} = 1 \quad \text{and} \quad \pmb \phi_{k,0} \in 
  \text{span} \{ \bar{\pmb \chi}_j\}_{j=1}^K,
\end{equation*}
and $\lambda_{K+1,0} > 1$.
\end{proposition}

\begin{proof}
 The statement involving the eigenvalues $\{\lambda_{k,0}\}_{k=1}^{K+1}$ follows from the
  proof of Lemma~\ref{C-is-strict-positive-definite} and so
  we only prove the result regarding the eigenfunctions.  Observe that since
  $G_0$ consists of $K$ disconnected components then
  $L_0 = \text{diag}( \tilde{L}_1, \cdots, \tilde{L}_K)$
  and each block matrix $\tilde{L}_k$ is itself a graph Laplacian. Writing the inner product $\langle\mbf x,L\mbf x\rangle$ in symmetric form as in \eqref{xLx}, it is easy to see that the first
  eigenvector of any graph Laplacian operator $L$ of the form \eqref{graph-Laplacian}
  is 
  $D_0^p \mbf 1$ 
  for $p=q$ with corresponding zero eigenvalue. Using this fact for the submatrices $\tilde{L}_k$, we infer that
  $L_0 \bar{\pmb \chi}_k = 0$
  for $k=1, \cdots, K$. We now conclude the proof
  by noting that $C_{\tau,0}^{-1}$ and $L_0$ have the same eigenvectors.
\end{proof}

\begin{proposition}\label{separate-clusters-c-dominance-by-low-spectrum}
 Suppose $G_0$ satisfies Assumption~\ref{assumptions-on-G-0}. Then
  \begin{equation*}
    \mbf c_{j,0} =  
    \left( \bar{\pmb \chi}_k \right)_j 
    \bar{\pmb \chi}_k 
     +
    \mcl{O}(\tau^{ 2\alpha} ) \qquad \forall j\in \tZ_k.
  \end{equation*}
\end{proposition}

\begin{proof}
  Follows directly from the expansion \eqref{c-j-eigen-expansion-equation}
  and the observation that $\lambda_{1,0}= \cdots = \lambda_{K,0} =1$ while \as{ $\lambda_{K+1,0}^{-1}
  = \mcl O( \tau^{-2\alpha})$.} Finally, note that
\begin{equation*}
  \sum_{\ell=1}^K (\bar{\pmb \chi}_\ell)_j \bar{\pmb \chi}_\ell =\left( \bar{\pmb \chi}_k\right)_j
  \bar{\pmb \chi}_k \qquad \forall j\in \tZ_k.
   \end{equation*}
  {}
\end{proof}

\subsection{Proof of Proposition~\ref{geometry-of-covaraince-functions}: Weakly connected clusters}
\label{sec:proof-of-weakly-sep-spectral-theory}

We now analyze the spectrum of the $C_{\tau,\eps}$ operators beginning with
auxiliary results regarding the matrices $W_\eps$ and $L_\eps$. 

\begin{lemma}\label{L-eps-has-power-expansion}
  Let $W_\eps$  be as in \eqref{W-eps-perturbation-expansion} and
 suppose 
 Assumptions~\ref{assumptions-on-G-0} and \ref{assumptions-on-G-eps} are satisfied.
 Let $L_\eps $  be as in \eqref{L-eps-perturbation-expansion}.
  Then  there exists $\eps_0 >0$ so that for all $\eps \in (0, \eps_0)$
 the matrix $W_\eps$ has non-negative weights and the
    graph Laplacian operator $L_\eps$
  satisfies an expansion of the form
  \begin{equation}
    \label{L-eps-power-expansion}
    L_\eps = L_0 + \sum_{h=1}^\infty\eps^h L^{(h)}
  \end{equation}
  with 
    $\| L^{(h)}\|_2 \in \ell^\infty$.
\end{lemma}

Note that in the above, the perturbations $L^{(h)}$ are not necessarily of graph Laplacian form.

\begin{proof} 
First, we prove that for sufficiently small $\eps$ the
  $W_\eps$ matrix has non-negative weights and is therefore a well-defined weight matrix for the graph $G_\eps$.
  By \eqref{W-k-condition}, we only need to consider indices $i,j$ for which $w^{(0)}_{ij}>0$ since  $w^{(h)}_{ij}$ may be negative for such indices, \as{i.e., by Assumption~\ref{assumptions-on-G-eps}
    these are only the edge weights within clusters since perturbed edge weights between
  different clusters are constrained to be positive}. Since $\| W^{(h)} \|_2$
  is uniformly bounded we can use the equivalence of the $\ell_2$ and $\ell_\infty$ norms
 to infer that all entries of $W^{(h)}$ are also uniformly bounded. 
  Then for $i,j \in \tZ_k$ we have
  \begin{equation*}
    w^{(\eps)}_{ij} = w^{(0)}_{ij} + \sum_{h=1}^\infty \eps^h w^{(h)}_{ij}
    \ge w^{(0)}_{ij} - \sup_{h} |w^{(h)}_{ij}| \left(\frac{\eps}{1 - \eps} \right), 
\end{equation*}
which is positive for all $i,j \in \tZ_k$ for sufficiently small $\eps \le \eps_k$.
Taking the infimum of $\eps_k$ we get the uniform constant $\eps_0$. 

It is straightforward to check  that
$D_\eps = D_0 + \sum_{h=1}^\infty \eps^h D^{(h)}$ where the $D^{(h)}: = \text{diag}( W^{(h)} \mbf 1)$
are the degree matrices
of the $W^{(h)}$. Let $d^{(0)}_{i}, d^{(\eps)}_{i}$ and $d^{(h)}_{i}$ denote
the diagonal entries of $D_0, D_\eps$ and $D^{(h)}$ respectively.
Then for $i =1, \cdots, N$,
\begin{equation*}
  d^{(\eps)}_{i} = d^{(0)}_{i} \left( 1 + \sum_{h=1}^\infty \eps^h \frac{d^{(h)}_{i}}{d^{(0)}_{i}} \right).
\end{equation*}
It follows from Assumption~\ref{assumptions-on-G-0}(b) that the
entries $d^{(0)}_{ii}$ of the degree matrix $D_0$ are strictly positive (i.e., the  components $\tG_k$ are pathwise connected) and so the ratio inside the inner sum
is bounded. Thus, we can further write
\begin{equation*}
  d^{(\eps)}_{i} = d^{(0)}_{i} \left( 1 + \eps \sum_{h=0}^\infty \eps^h \frac{d^{(h+1)}_{i}}{d^{(0)}_{i}} \right)
   = d^{(0)}_{i} (1 + \eps \hat d_{i}).
 \end{equation*}
 where the entries $\hat d_{i}$ are well-defined since the sum inside the bracket converges for
 $\eps <1$.
 Assuming $\eps$ is sufficiently small so that $\eps \hat d_{i} <1$ we
 can use the generalized binomial expansion to write for any $\beta\in \mbb{R}$
 \begin{equation*}
   \begin{aligned}
     \left(d^{(\eps)}_{i}\right)^\beta
     & = \left(d^{(0)}_{i}\right)^\beta \left( 1 + \eps \tilde d_{i} \right)^\beta\\
     & = \left(d^{(0)}_{i}\right)^\beta
     \left( 1+  \sum_{h=1}^\infty {\beta\choose h} \eps^h \left(\hat d_{i}\right)^h \right).
 \end{aligned}
 \end{equation*}
Thus, for any $p\in\mbb{R}$
there exists a diagonal matrix $\hat D^{(p)}$ so that
\begin{equation*}
  D_\eps^{-p} = D_0^{-p} +  \eps \hat D^{(p)},
\end{equation*}
and the entries of $\hat D^{(p)}$  are uniformly bounded for sufficiently small $\eps$. 
Therefore, we can write
\begin{equation*}
  \begin{aligned}
    L_\eps &= D_\eps^{-p} ( D_\eps - W_\eps) D_\eps^{-p} \\
    & = D_\eps^{-p} \left( D_0 + \sum_{h=1}^\infty \eps^h D^{(h)} - W_0 - \sum_{h=1}^\infty \eps^h W^{(h)} \right)
    D_\eps^{-p} \\
    & =  D_\eps^{-p} \left( D_0 - W_0 + \sum_{h=1}^\infty \eps^h (D^{(h)} - W^{(h)}) \right) D_\eps^{-p} \\
    & = D_\eps^{-p} (D_0 - W_0) D_\eps^{-p} + \sum_{h=1}^\infty \eps^h D_\eps^{-p} (D^{(h)} - W^{(h)}) D_\eps^{-p}\\
    & = L_0 + \sum_{h=1}^\infty \eps^h L^{(h)},
  \end{aligned}
\end{equation*}
where the $L^{(h)}$ matrices are obtained by gathering the $\mcl O(\eps^h)$ terms. 
Note that  since the entries of  $W^{(h)}$ and $\hat D^{(p)}$ are uniformly bounded
then all entries of the $L^{(h)}$ are bounded
uniformly from which it follows that $\| L^{(h)} \|_2 \in \ell^\infty$.
{}
\end{proof}

We now characterize the low-lying eigenvalues and eigenvectors of
$C_{\eps, \tau}^{-1}$. 
\begin{proposition}[The spectrum of $C^{-1}_{\tau,\eps}$]\label{low-lying-spectrum-of-C}
  Suppose Assumptions~\ref{assumptions-on-G-0} and \ref{assumptions-on-G-eps} are satisfied
  and let
  $\{\lambda_{j, \eps}, \pmb \phi_{j, \eps}\}$ denote the orthonormal eigenpairs of $C^{-1}_{\tau,\eps}$. Then
  there exists $\eps_0 >0$ so that
  \begin{enumerate}[(i)]
  \item  $\lambda_{1, \eps} =1$ and $\pmb \phi_{1, \eps} =  \bar{\pmb \chi}$,
    as in \eqref{chi-definition}. 
  \item For any $\eps \in (0, \eps_0)$, there exists constants $\Xi_1=\Xi_1(K,\|L^{(1)}\|_2)>0 $ and $\Xi_2=\Xi_2(K,\sup_{h\ge 2}\|L^{(h)}\|_2,\eps_0)>0$ independent of $\eps $ so that
    \begin{equation}
      \label{low-lying-eigenvalue-rate}
      \lambda_{k,\eps} 
      \le \left(1+\Xi_1 \eps \tau^{-2} +\Xi_2\eps^2 \tau^{-2} \right)^{\alpha} , \qquad
      \forall k \in \{2, \cdots, K\}.
      \end{equation}

    \item If there exists a uniform
      constant $\vartheta >0$ so that $\lambda_{K+1, \eps} - 1 \ge \vartheta$
     then 
    there exists a constant $\Xi_3=\Xi_3(K, \|L^{(1)}1\|_2, \vartheta) >0$ 
independent of $\eps \in (0, \eps_0)$ so that
      \begin{equation}
        \label{low-lying-span}
        \left| 1 - \sum_{j=1}^K
          \left\langle \pmb \phi_{j,\eps} , \bar{ \pmb \chi}_k \right\rangle^2 \right|
       \le \Xi_3 \eps^2 + \mcl O (\eps^3),
        \qquad \forall j \in \{ 1,\cdots, K\}, 
      \end{equation}
      with
      $\bar{ \pmb \chi}_k$ as in \eqref{chi-k-definition}.
  \end{enumerate}
\end{proposition}

\begin{proof}
  (i) Follows from the fact  that  $L_\eps$ is a graph Laplacian operator with first
  eigenvalue $\sigma_{1, \eps} = 0$ and first eigenvector 
  $\pmb \phi_{1, \eps} =  D_\eps^p \mbf 1/ \| D_\eps^p \mbf 1\|$.

  (ii) Let $\sigma_{j,\eps}$ for $j=1,\cdots, N$ denote the eigenvalues of $L_\eps$. 
  By the min-max principle  (see for example \cite[Thm.~8.1.2]{golub2012matrix}, also known as Courant--Fisher theorem)
  \begin{equation}
    \label{min-max-principle}
    \sigma_{k,\eps} = \min_{ \mbb U \in \mbb V_k} \max_{\stackrel{\mbf x \in \mbb U}{\| \mbf x\| = 1}}
    \langle \mbf x, L_{\eps} \mbf x \rangle ,
  \end{equation}
  where $\mbb V_k$ is the set of all $k$-dimensional subsets in $\mbb R^N$.
  Now for $k \le K$ take 
  $\mbb U = \text{span} \{ \bar{ \pmb \chi}_j\}_{j=1}^k$.
  Since
  the vectors  
  $\bar{ \pmb \chi}_j$ 
  are
  by definition orthonormal then $\mbb U$ is a $k$-dimensional subspace of $\mbb R^N$.
  These vectors are also in the null space of  $L_0$ and
  so we have for $i,j\in \{ 1, \dots, K\}$
  \begin{equation}\label{chi-rayleigh-quotient-calculation}
    \begin{aligned}
     \left\langle \bar{ \pmb \chi}_i ,   L_\eps  \bar{ \pmb \chi}_j
    \right \rangle 
    &= \sum_{h = 1}^\infty \eps^h \langle \bar{ \pmb \chi}_i, L^{(h)}  \bar{ \pmb \chi}_j \rangle
     \le \sum_{h =1}^\infty \eps^h \| L^{(h)} \|_2  \\
     &= \eps\|L^{(1)}\|_2  + \eps^2\left(\|L^{(2)}\|_2  + \left( \sup_{h= 2,3, \cdots}  \| L^{(h)} \|_2\right) \frac{\eps}{1-\eps}\right)\,.
  \end{aligned}
\end{equation}
We can now generalize this bound to all unit vectors $\mbf x \in \mbb U$
 to get 
\begin{equation*}
  \left\langle \mbf x , L_\eps \mbf x \right\rangle \le \Xi_1\eps +\Xi_2\eps^2
   \end{equation*}
   where
$$
\Xi_1:=K^2\|L^{(1)}\|_2\,,\qquad
\Xi_2:=  \frac{K^2}{(1- \eps_0)}\left( \sup_{h= 2,3, \cdots}  \| L^{(h)} \|_2\right) \,. 
$$
   From \eqref{min-max-principle} we now infer that
   \begin{equation}\label{bound-on-low-lying-sigma-eps}
     \sigma_{k,\eps} \le   \Xi_1\eps +\Xi_2\eps^2\,, \qquad \forall  k \in\{ 2, \dots, K\}.
   \end{equation}
   Then  \eqref{low-lying-eigenvalue-rate}  follows by noting that
   $\lambda_{k, \eps} =  \tau^{-2\alpha}( \sigma_{k,\eps} + \tau^2)^\alpha$.

   (iii) Using the fact that $L_0\bar{\pmb \chi}_k = 0$ for $k=1,\dots, K$  we can write
   \begin{align*}
     \left \| L_\eps\bar{ \pmb \chi}_k \right\|^2
     & = \left\langle L_\eps \bar{\pmb \chi}_k ,   L_\eps \bar{\pmb \chi}_k
       \right \rangle 
      =  
       \sum_{h=1}^\infty \sum_{\ell =1}^\infty
       \eps^h \eps^\ell  \langle L^{(h)} \bar{\pmb \chi}_k, L^{(\ell)}\bar{\pmb \chi}_k\rangle \\
        & \le \eps^2 \| L^{(1)}\|_2^2 + \left(\sup_{\ell=2, 3, \cdots} \|L^{(\ell)}\|_2^2\right)\left(\frac{\eps^2}{1-\eps}\right)^2 \,.  
   \end{align*}
   Now let $\bar{ \pmb \chi}_k = \sum_{j=1}^N q_{kj} \pmb \phi_{j,\eps}$ and
   assume $\sigma_{K+1, \eps} \ge \vartheta$. Note that $q_{kj}\le 1$ for all $k\in \{1,...,K\}$ and $j\in Z$ since $\bar{ \pmb \chi}_k$ is normalized. Then the above calculation yields  
    \begin{equation*}
       \left\langle L_\eps \bar{ \pmb \chi}_k ,   L_\eps  \bar{ \pmb \chi}_k
       \right \rangle
        = \sum_{j=1}^N q_{kj}^2 \sigma_{j,\eps}^2  \le \eps^2 \| L^{(1)}\|_2^2  + \mcl{O}(\eps^4).
      \end{equation*}
      From this it follows that
      \begin{align*}
        \vartheta^2 \left( 1 - \sum_{j=1}^K q_{kj}^2 \right) & = \vartheta^2 \sum_{j=K+1}^Nq_{kj}^2
        \le \sum_{j=K+1}^N q_{kj}^2 \sigma_{j,\eps}^2 \\
        &\le \eps^2  \| L^{(1)} \|_2^2
        - \sum_{j=1}^Kq_{kj}^2 \sigma_{j,\eps}^2
        + \mcl O(\eps^4).
      \end{align*}
      Now using \eqref{bound-on-low-lying-sigma-eps} we obtain 
      \begin{align*}
        \left|1 - \sum_{j=1}^K\left\langle \pmb \phi_{j,\eps}, \bar{ \pmb \chi}_k  \right\rangle^2 \right|
        & = \left| 1 - \sum_{j=1}^Kq_{kj}^2 \right|\\
        &\le \frac{1}{\vartheta^2} \left( \eps^2 \| L^{(1)}\|_2^2
          + \sum_{j=1}^K \sigma_{j,\eps}^2 + \mcl O (\eps^4)) \right)\\
        &\le \Xi_3 \eps^2 + \mcl O(\eps^3).
      \end{align*}
      The desired result follows since $C_{\eps,\tau}^{-1}$ has the same eigenfunctions
      as $L_\eps$. 
    \end{proof}

    The result in part (iii) of Proposition~\ref{low-lying-spectrum-of-C} is central to the
    rest of our arguments
    as it states that the eigenvectors $\{\pmb \phi_{j,\eps}\}_{j=1}^K$ and
    the  functions $\{ \bar{ \pmb \chi}_j  \}_{j=1}^K$ have nearly
    the same span for small $\eps$ provided that $L_{\eps}$ has a uniform spectral gap
    between $\sigma_{K, \eps}$ and $\sigma_{K+1, \eps}$. We now show that this condition
    is satisfied under very general conditions.


\begin{proposition}[Existence of spectral gaps]\label{spectral-gap-existence}
  Suppose Assumptions~\ref{assumptions-on-G-0} and \ref{assumptions-on-G-eps} are satisfied.
  Then 
   \begin{equation*}
     \sigma_{K+1, \eps} \ge \theta - \sum_{h=1}^\infty \eps^h \| L^{(h)}\|_2 \qquad
     \text{and} \qquad
     \lambda_{K+1,\eps} \ge
     \tau^{-2\alpha} \left(\tau^2 + \theta  - \sum_{h=1}^\infty \eps^h  \| L^{(h)} \|_2  \right)^\alpha,
   \end{equation*}
   where  $\theta>0$ is the constant appearing
  in Assumption~\ref{assumptions-on-G-0}(b).
\end{proposition}

\begin{proof}
  By the max-min principle
  \begin{equation}
    \label{max-min-principle}
    \sigma_{k+1, \eps} = \max_{\mbb U \in \mbb V_{k}} \min_{\stackrel{\mbf x \bot \mbb U}{\| \mbf x \| =1}}
    \langle \mbf x , L_\eps \mbf x \rangle.
  \end{equation}
  where $\mbb V_{k}$ denotes the set of $k$-dimensional subspaces of $\mbb R^N$ and
  $\mbf x \bot \mbb U$ means the vector $\mbf x$ belongs to the orthogonal complement of $\mbb U$.

  Now take $\mbb U = \text{span} \{\bar{ \pmb \chi}_j \}_{j=1}^K$. 
  Let $\mbf x_k\in \mbb R^{N_k}$ denote the restriction of $\mbf x$ to the subset
  of indices $\tZ_k$. 
  Then 
  $\mbf x_k^T (\tD_k^p \mbf 1) = 0$, where 
  we recall $\tD_k$ is the degree matrix of the 
  $k$-th cluster $\tG_k$.
  Now 
  we can write for all $k\in\{1,...,K\}$,
  \begin{equation*}
    \begin{aligned}
      \langle \mbf x, L_\eps \mbf x \rangle & =  \langle \mbf x, L_0 \mbf x \rangle +
       \sum_{h=1}^\infty \eps^h \langle \mbf x, L^{(h)} \mbf x \rangle\\
       & = \sum_{k=1}^K \langle \mbf x_k, \tL_k \mbf x_k \rangle +
       \sum_{h=1}^\infty \eps^h \langle \mbf x, L^{(h)} \mbf x \rangle\\
       & \ge \theta \sum_{k=1}^K \langle \mbf x_k, \mbf x_k \rangle  +
       \sum_{h=1}^\infty \eps^h \langle \mbf x, L^{(h)} \mbf x \rangle\\
       & \ge \theta \| \mbf x\|^2 - \sum_{h=1}^\infty \eps^h \| L^{(h)} \| \| \mbf x\|^2\\
       & = \left(\theta - \sum_{h=1}^\infty \eps^h \| L^{(h)} \|_2 \right) \| \mbf x\|^2\,.
    \end{aligned}
  \end{equation*}
  The lower bound on $\sigma_{K+1,\eps}$ now follows from \eqref{max-min-principle}
 while the lower bound on $\lambda_{K+1,\eps}$ follows from the
  observation that $\lambda_{j,\eps} = \tau^{-2\alpha} ( \sigma_{j,\eps} + \tau^2)^\alpha$
  from the definition of $C^{-1}_{\tau,\eps}$.
  {}
\end{proof}

\begin{example}[Perturbed kernels]
  Consider a proximity graph $G$ where the weight matrix $W_0$ is given by
  \begin{equation*}
    w^{(0)}_{ij} = \kappa(x_i - x_j), \qquad i,j \in Z, 
  \end{equation*}
where $\kappa: \mbb R^N \mapsto \mbb R$ is a positive, uniformly bounded,  radially symmetric and non-increasing kernel with full support and
   $\int_{\mbb R^N} \kappa(x)^2 dx < + \infty$.
  Now let $\tilde \kappa: \mbb R^N \mapsto \mbb R$ be another radially symmetric, positive and
  uniformly  bounded function and define the perturbed   weight matrix $W_\eps$ by
  \begin{equation*}
    w^{(\eps)}_{ij} = w^{(0)}_{ij} + \eps \tilde \kappa(x_i - x_j).
  \end{equation*}
  Then the resulting perturbed graph Laplacian $L_\eps$
  satisfies the conditions of Propositions~\ref{low-lying-spectrum-of-C} and
  \ref{spectral-gap-existence}.
  As a concrete example take $\kappa(x) = \exp( -\| x\|)$ and take
  $\tilde \kappa(x) =1$. For more details on applications using this type of weight matrix, see \cite{zhu2005semi}.
\end{example}

\begin{example}[Adding a few edges]
  As another example consider a weight matrix $W_0$ of block diagonal form
  consisting of two matrices $W_+$ and $W_-$ corresponding to two disjoint
  and connected subgraphs, i.e,
  \begin{equation*}
    W_0 = {\rm diag}( W_+, W_-)
  \end{equation*}
and any pair of nodes, both in $W_+$ (resp. $W_-$) are  connected by a sequence
of edges with strictly positive weights.
  Since each subgraph is assumed to be connected then the graph $G_0$
  satisfies Assumption~\ref{assumptions-on-G-0}.
  Let $Z$ denote the collection of all nodes in the graph and select a
  subset $\tilde Z \subset Z$ such that $2 \le |\tilde Z| \le |Z|$ and
  $\tilde Z$ contains at least one point in each of the two disjoint components.
  Define the matrix
  \begin{equation*}
    w^{(1)}_{ij} =\left\{
    \begin{aligned}
      &1, \qquad i,j \in \tilde Z, i \neq j,\\
      & 0 \qquad \text{otherwise,}
    \end{aligned}\right.
\end{equation*}
and consider the perturbation
\begin{equation*}
  W_\eps = W_0 + \eps W^{(1)}.
\end{equation*}
This perturbation corresponds to weak coupling of two disjoint subgraphs by adding 
edges between subgraphs. Once again, we can directly verify that the resulting
perturbed graph Laplacian $L_\eps$ satisfies the conditions of
Propositions~\ref{low-lying-spectrum-of-C} and \ref{spectral-gap-existence}.
\end{example}

At the end of this section we use our results on the closeness of
the eigenvectors $\{ \pmb \phi_{j,\eps} \}_{j=1}^K$ and 
$\{ \bar{ \pmb \chi}_j\}_{j=1}^K$
to identify the geometry of the $\mbf c_{j,\eps}$ the columns of $C_{\eps, \tau}$.

\begin{proposition}\label{geometry-of-greens-functions}
  Suppose Assumptions~\ref{assumptions-on-G-0} and \ref{assumptions-on-G-eps} are satisfied. Then
  \begin{enumerate}[(a)]
  \item If $\epsilon = o(\tau^2)$  there exists a constant $\Xi_4 >0$ independent of
    $\epsilon$ and $\tau$ so that 
  \begin{equation*}
    \left\| \mbf c_{\ell,\eps} 
    -\left( \bar{\pmb \chi}_k \right)_\ell
    \bar{ \pmb \chi}_k\right\|^2 \le \Xi_4 \left(
      \frac{\eps^{2}}{\tau^{4}}+    \tau^{4\alpha} +  \eps^2  \right)\,, \qquad \forall\,\ell \in \tZ_k.
  \end{equation*}
  That is, the columns of $C_{\tau, \eps}$ have the same geometry as the weighted set functions $\bar{ \pmb \chi}_k$
  when  $\eps, \tau$ and $\eps/\tau^2$ are small.

\item If $\eps/\tau^2= \beta$ is constant,    
 there exists a constant $\Xi_5 >0$
  independent of  $\epsilon$ and $\tau$ so that
 \begin{equation*}
   \left\| \mbf c_{\ell,\eps} - 
   \left[  \left( 1 - \tbeta\right) (\bar{\pmb \chi})_\ell \bar{\pmb \chi} 
   + \tbeta( \bar{\pmb \chi}_k)_\ell\bar{\pmb \chi}_k
   \right]  \right\|^2
   \le \Xi_5 \left( \eps^2+\tau^{4\alpha} \right)\,,
   \qquad \forall\,\ell \in \tZ_k\,,
 \end{equation*}
 where $\tbeta= (1+\Xi_1 \beta)^{-\alpha}$ and $\Xi_1$ is the constant in \eqref{low-lying-eigenvalue-rate}.
 \end{enumerate}
\end{proposition}

\begin{proof} Let $P_\eps \in \mbb R^{N \times N}$ denote the projection matrix  onto
    ${\rm span}\{ \pmb \phi_{j,\eps}\}_{j=1}^K$ and
    $P_0 \in \mbb R^{N \times N}$ denote the projection onto
    ${\rm span} \{ \bar{ \pmb \chi}_k\}_{k=1}^K$.
     Define the residuals
    \begin{align*}
        \mbf r_k &:= (I-P_\eps)\tchi_k
        =\bar{\pmb \chi}_k - \sum_{j=1}^K \langle
        \pmb \phi_{j,\eps}, \bar{\pmb \chi}_k \rangle  
        \pmb \phi_{j,\eps}\,,\\
         \mbf s_j &:= (I-P_0)\pmb\phi_{j,\eps}
         =\pmb\phi_{j,\eps} - \sum_{k=1}^K \langle
        \pmb \phi_{j,\eps}, \bar{\pmb \chi}_k \rangle  
         \tchi_k\,.
      \end{align*}
    By Proposition~\ref{low-lying-spectrum-of-C}(iii) there exists a uniform constant
    $\Xi_6 >0$ so that 
    \begin{equation}\label{close-subspaces-prop-display-1}
    \sum_{k=1}^K\|\mbf r_k\|^2 = \sum_{j=1}^K\|\mbf s_j\|^2
    =
    K - \sum_{k=1}^K\sum_{j=1}^K
          \left\langle \pmb \phi_{j,\eps} , \tchi_k \right\rangle^2  \le K \Xi_6 \eps^2\,.
    \end{equation}
Writing $\pmb\phi_{j,\eps}=P_0\pmb\phi_{j,\eps} + \mbf s_j$ for any $j\in\{1,...,N\}$ and using the fact that the $\{ \pmb \phi_{j,\eps}\}_{j=1}^K$ and  $\{ \tchi_k\}_{k=1}^K$ are unit vectors, we can estimate
\begin{align}\label{term3}
    \left\|\sum_{j=1}^K  \left( \pmb \phi_{j, \eps} \right)_\ell
  \pmb \phi_{j, \eps}
  - \sum_{k=1}^K\left(\tchi_k\right)_\ell \tchi_k \right\|^2
  &= \left\|\sum_{j=1}^K  \left( \pmb \phi_{j, \eps} \right)_\ell
  P_0\pmb \phi_{j, \eps}
  - \sum_{k=1}^K\left(\tchi_k\right)_\ell \tchi_k 
  + \sum_{j=1}^K  \left( \pmb \phi_{j, \eps} \right)_\ell
   \mbf s_j\right\|^2\notag\\
  &\le 2\left\|\sum_{j=1}^K  \left( \pmb \phi_{j, \eps} \right)_\ell
  P_0\pmb \phi_{j, \eps}
  - \sum_{k=1}^K\left(\tchi_k\right)_\ell \tchi_k \right\|^2
  + 2K\sum_{j=1}^K \left\|\mbf s_j\right\|^2\notag\\
 &= 2\left\|\sum_{k=1}^K \left[\left(\sum_{j=1}^K  \left( \pmb \phi_{j, \eps} \right)_\ell \langle \tchi_k,\pmb\phi_{j,\eps}\rangle\right)
  - \left(\tchi_k\right)_\ell\right] \tchi_k \right\|^2
  + 2K\sum_{j=1}^K \left\|\mbf s_j\right\|^2\notag\\
  &= 2\left\|\sum_{k=1}^K \left((P_\eps-I)\tchi_k\right)_\ell \tchi_k \right\|^2
  + 2K\sum_{j=1}^K \left\|\mbf s_j\right\|^2\notag\\ 
&\le 2K\sum_{k=1}^K \left\|(P_\eps-I)\tchi_k \right\|^2
  + 2K\sum_{j=1}^K \left\|\mbf s_j\right\|^2\notag\\ 
&= 2K\sum_{k=1}^K \left\|\mbf r_k \right\|^2
  + 2K\sum_{j=1}^K \left\|\mbf s_j\right\|^2
  \le 4K^2\Xi_6\eps^2\,,
\end{align}
where the last inequality follows from \eqref{close-subspaces-prop-display-1}.

Now recall the expansion \eqref{c-j-eigenvector-expansion} for columns of $C_{\eps,\tau}$,
\begin{equation*}
  \begin{aligned}
  \mbf c_{\ell,\eps} & = \sum_{j=1}^N \frac{1}{\lambda_{j,\eps}} \left( \pmb \phi_{j, \eps} \right)_\ell
  \pmb \phi_{j, \eps}.
\end{aligned}
\end{equation*}
For (a), we want to estimate
\begin{align}\label{eq:3terms(a)}
   \left\| \mbf c_{\ell,\eps} - \sum_{k=1}^K\left(\tchi_k\right)_\ell \tchi_k\right\|^2\notag
   & \le 3\left\|\sum_{j=K+1}^N \frac{1}{\lambda_{j,\eps}} \left( \pmb \phi_{j, \eps} \right)_\ell
  \pmb \phi_{j, \eps} \right\|^2
  + 3\left\|\sum_{j=1}^K \left(\frac{1}{\lambda_{j,\eps}} -1\right)\left( \pmb \phi_{j, \eps} \right)_\ell
  \pmb \phi_{j, \eps} \right\|^2 \\
  &+ 3\left\|\sum_{j=1}^K  \left( \pmb \phi_{j, \eps} \right)_\ell
  \pmb \phi_{j, \eps}
  - \sum_{k=1}^K\left(\tchi_k\right)_\ell \tchi_k \right\|^2
\end{align}
The last term is of order $\eps^2$ by \eqref{term3}.
The first term can be controlled using Proposition~\ref{spectral-gap-existence} together with the same bound as in \eqref{chi-rayleigh-quotient-calculation},
\begin{align*}
   & \left\|\sum_{j=K+1}^N \frac{1}{\lambda_{j,\eps}} \left( \pmb \phi_{j, \eps} \right)_\ell
  \pmb \phi_{j, \eps} \right\|^2
  \le \left(\frac{(N-K)}{\lambda_{K+1,\eps}}\right)^2\\
  &\quad\le (N-K)^2 \tau^{4\alpha} \left(\tau^2 + \theta  - \eps\|L^{(1)}\|_2 -\eps^2\left(\sup_{h\ge 2}  \| L^{(h)} \|_2 \right) \right)^{-2\alpha}\\
  &\quad\le (N-K)^2\theta^{-2\alpha} \tau^{4\alpha}
  +  \mcl O\left(\tau^{4\alpha}\left(\tau^2+\eps\right)\right)\,,
\end{align*}
   where  $\theta>0$ is the constant appearing
  in Assumption~\ref{assumptions-on-G-0}(b).
Next, when $\epsilon/ \tau^2$ is small, we can estimate the second term in the right-hand side of \eqref{eq:3terms(a)} using Proposition~\ref{low-lying-spectrum-of-C}(ii): recall that $\lambda_{1,\eps}=1$, and so
\begin{align}\label{lambdaest(a)}
    \left|\frac{1}{\lambda_{j,\eps}}-1\right|=\frac{|1-\lambda_{j,\eps}|}{|\lambda_{j,\eps}|} 
    \le |1-\lambda_{j,\eps}|
    \le \alpha \Xi_1\frac{\eps}{\tau^2} +\mcl O\left(\eps^2\tau^{-4}\right)\qquad \text{ for } j\in\{2,...,K\}\,,
\end{align}
and therefore
\begin{align*}
    \left\|\sum_{j=1}^K \left(\frac{1}{\lambda_{j,\eps}} -1\right)\left( \pmb \phi_{j, \eps} \right)_\ell
  \pmb \phi_{j, \eps} \right\|^2
  &\le K \sum_{j=1}^K \left|\frac{1}{\lambda_{j,\eps}} -1\right|^2\left|\left( \pmb \phi_{j, \eps} \right)_\ell\right|^2
  \|\pmb \phi_{j, \eps}\|^2\\
  &\le  \alpha ^2 K^2 \Xi_1^2 \left(\frac{\eps}{\tau^2}\right)^2 +\mcl O\left(\eps^3\tau^{-6}\right)\,.
\end{align*}
Putting the above estimates together, we can find a constant $\Xi_7=\Xi_7(N,K,\theta,\alpha,\eps_0)>0$ independent of $\eps$ and $\tau$ such that for all $\ell\in Z$,
\begin{align*}\label{est(a)}
   &\left\| \mbf c_{\ell,\eps} - \sum_{k=1}^K\left(\tchi_k\right)_\ell \tchi_k\right\|^2
   \le \Xi_7\left( \tau^{4\alpha}  +  \left(\frac{\eps}{\tau^2}\right)^2 +\eps^2\right)
   + \mcl O\left(\tau^{4\alpha}\left(\tau^2+\eps)\right)\right)
   +\mcl O\left(\left(\frac{\eps}{\tau^{2}}\right)^3\right)\,.
\end{align*}
Finally, note that for each $\ell\in Z$, there exists a unique $k_0\in\{1,...,K\}$ such that $\ell\in \tZ_{k_0}$. Then 
$$
\sum_{k=1}^K\left(\tchi_k\right)_\ell \tchi_k
= \left(\tchi_{k_0}\right)_\ell \tchi_{k_0},
$$
which concludes the proof of statement (a).

To prove (b), the argument is similar. For $\beta:= \eps/\tau^2$, we have
$$
\tbeta=\left(1+\beta \Xi_1\right)^{-\alpha}>0\,.
$$
Then, instead of \eqref{lambdaest(a)}, we use the bound in Proposition~\ref{low-lying-spectrum-of-C}(ii) to estimate
\begin{align*}\label{lambdaest(b)}
    \left|\frac{1}{\lambda_{j
\eps}}- \tbeta\right|
    &=\frac{|1- \tbeta\lambda_{j,\eps}|}{|\lambda_{j,\eps}|} 
    \le |1- \tbeta\lambda_{j,\eps}|
    = \tbeta\left|\lambda_{j,\eps}-\frac{1}{\tbeta}\right|\\
    &\le  \tbeta\left|\left(\tbeta^{-1/\alpha}+\eps\beta\Xi_2\right)^\alpha - \tbeta^{-1}\right|
    = \frac{\alpha\beta\Xi_2}{\left(1+\beta \Xi_1\right)}   \eps +\mcl O(\eps^2)\qquad \text{ for } j\in\{2,...,K\}.
\end{align*}
Therefore, our goal is to control
\begin{equation}\label{eq:3terms(b)}
\begin{aligned}
   \bigg\| \mbf c_{\ell,\eps} - (1-\tbeta) & \left( \pmb \phi_{1, \eps} \right)_\ell  \pmb \phi_{1, \eps}
   - \tbeta\sum_{k=1}^K\left(\tchi_k\right)_\ell \tchi_k\bigg\|^2 \\
    & \le  3\left\|\sum_{j=K+1}^N \frac{1}{\lambda_{j,\eps}} \left( \pmb \phi_{j, \eps} \right)_\ell
  \pmb \phi_{j, \eps} \right\|^2\\
  &\quad + 3\left\|\sum_{j=1}^K \frac{1}{\lambda_{j,\eps}} \left( \pmb \phi_{j, \eps} \right)_\ell
  \pmb \phi_{j, \eps} 
  - (1-\tbeta) \left( \pmb \phi_{1, \eps} \right)_\ell  \pmb \phi_{1, \eps}
  - \tbeta\left(\sum_{j=1}^K  \left( \pmb \phi_{j, \eps} \right)_\ell
  \pmb \phi_{j, \eps}\right)
  \right\|^2 \\
  & \quad + 3\left\| \tbeta\left(\sum_{j=1}^K  \left( \pmb \phi_{j, \eps} \right)_\ell
  \pmb \phi_{j, \eps}
  - \sum_{k=1}^K\left(\tchi_k\right)_\ell \tchi_k\right) \right\|^2\\
  &
   = 3\left\|\sum_{j=K+1}^N \frac{1}{\lambda_{j,\eps}} \left( \pmb \phi_{j, \eps} \right)_\ell
  \pmb \phi_{j, \eps} \right\|^2 \\
  & \quad + 3\left\|\sum_{j=2}^K \left(\frac{1}{\lambda_{j,\eps}} - \tbeta\right)\left( \pmb \phi_{j, \eps} \right)_\ell
  \pmb \phi_{j, \eps} 
  \right\|^2 \\
  & \quad + 3\tbeta^2\left\|\sum_{j=1}^K  \left( \pmb \phi_{j, \eps} \right)_\ell
  \pmb \phi_{j, \eps}
  - \sum_{k=1}^K\left(\tchi_k\right)_\ell \tchi_k \right\|^2\,.
\end{aligned}
\end{equation}
Thanks to the previous estimate, the second term can be bounded by
\begin{align*}
    \left\|\sum_{j=2}^K \left(\frac{1}{\lambda_{j,\eps}} - \tbeta\right)\left( \pmb \phi_{j, \eps} \right)_\ell
  \pmb \phi_{j, \eps} 
  \right\|^2
  \le  (K-1)^2\left(\frac{\alpha\beta\Xi_2}{\left(1+\beta \Xi_1\right)}  \right)^2 \eps^2 +\mcl O(\eps^3)\,.
\end{align*}
The first and last terms in \eqref{eq:3terms(b)} can be estimated as for part (a), and so we can find a constant $\Xi_8=\Xi_8(N,K,\theta,\alpha,\eps_0)>0$ independent of $\eps$ and $\tau$ such that for all $\ell\in Z$,
\begin{align*}
     \bigg\| \mbf c_{\ell,\eps} - (1- \tbeta) & \left( \pmb \phi_{1, \eps} \right)_\ell  \pmb \phi_{1, \eps}
   - \tbeta\sum_{k=1}^K\left(\tchi_k\right)_\ell \tchi_k\bigg\|^2\\
   &\le \Xi_8\left( \tau^{4\alpha} +\eps^2\right)
   + \mcl O\left(\tau^{4\alpha}\left(\tau^2+\eps)\right)\right)
   +\mcl O\left(\eps^3\right)\,.
\end{align*}
From Proposition~\ref{low-lying-spectrum-of-C}(i), we know that
$\pmb \phi_{1, \eps}  =  D_\eps^p \mbf 1/ \| D_\eps^p \mbf 1\|$.
Further, using the expansion $D_\eps^p=D_0^p+\eps \hat{D}^{(p)}$ derived in the proof of Lemma~\ref{L-eps-has-power-expansion}, we can simplify the middle term in the expression on the left-hand side above,
\begin{align*}
    (1-\tbeta) \left( \pmb \phi_{1, \eps} \right)_\ell  \pmb \phi_{1, \eps}
    = (1 - \tbeta) (\bar{\pmb \chi})_\ell \bar{\pmb \chi} +\mcl O(\eps)\,.
\end{align*}
This concludes the proof of Proposition~\ref{geometry-of-greens-functions}.
\end{proof}


\end{document}